\documentclass[letterpaper]{article} 
\usepackage[submission]{aaai23}  
\usepackage{times}  
\usepackage{helvet}  
\usepackage{courier}  
\usepackage[hyphens]{url}  
\usepackage{graphicx} 
\usepackage{natbib}  
\usepackage{caption} 
\usepackage[ruled,linesnumbered]{algorithm2e}
\usepackage{algorithmic}

\usepackage{newfloat}
\usepackage{listings}
\DeclareCaptionStyle{ruled}{labelfont=normalfont,labelsep=colon,strut=off} 
\usepackage[cmyk]{xcolor}
\usepackage{enumitem} 
\usepackage{amsthm,amsmath, amsbsy}
\usepackage{amsfonts}
\usepackage{multirow}
\usepackage{makecell}
\usepackage{subfigure} 
\usepackage{booktabs}
\usepackage{listings}
\usepackage{pythonhighlight}
\usepackage[normalem]{ulem}
\useunder{\uline}{\ul}{}
\newtheorem{definition}{Definition}

\newtheorem{lemma}{Lemma}

\title{EDEN: A Plug-in Equivariant Distance Encoding to Beyond the 1-WL Test}
\author {
    Chang Liu, Yuwen Yang, Yue Ding, Hongtao Lu$^*$
}
\affiliations {
    Department of computer science and engineering, Shanghai Jiao Tong University\\ 
    \{isonomialiu, youngfish, dingyue, htlu\}@sjtu.edu.cn
}

\begin{document}

\maketitle

\begin{abstract}
The message-passing scheme is the core of graph representation learning. While most existing message-passing graph neural networks (MPNNs) are permutation-invariant in graph-level representation learning and permutation-equivariant in node- and edge-level representation learning, their expressive power is commonly limited by the 1-Weisfeiler-Lehman (1-WL) graph isomorphism test. Recently proposed expressive graph neural networks (GNNs) with specially designed complex message-passing mechanisms are not practical. To bridge the gap, we propose a plug-in Equivariant Distance ENcoding (EDEN) for MPNNs. EDEN is derived from a series of interpretable transformations on the graph's distance matrix. We theoretically prove that EDEN is permutation-equivariant for all level graph representation learning, and we empirically illustrate that EDEN's expressive power can reach up to the 3-WL test. Extensive experiments on real-world datasets show that combining EDEN with conventional GNNs surpasses recent advanced GNNs. 
\end{abstract}

\section{Introduction}
The pervasiveness of graph-structured data in modern society, such as social networks \cite{tang2009social}, recommendation systems \cite{fan2019graph}, and bio-medicine \cite{zitnik2018biosnap}, has led to the rapid development of Graph Neural Networks (GNNs) \cite{GNN-review-JZ}. 
When representing a graph, perturbing the IDs of nodes in the graph should ensure that the representation of the graph does not change. 
GNNs are preferably either permutation-invariant for graph-level representation learning or permutation-equivariant for node- and edge-level representation learning\cite{azizian2020expressive}.
Message-passing Neural Networks (MPNNs) \cite{MPNN, Sage, GCN, GAT, GIN} with the property of permutation-invariance/equivariance have become the \textit{de facto} standard for graph representation learning tasks. 
However, the expressive power of MPNNs is commonly limited to the 1-Weisfeiler-Lehman (1-WL) graph isomorphism test \cite{WL}, which means MPNNs fail to distinguish nodes or graphs in many special cases (\textit{e.g.,} $d-$regular graphs). 
This main issue hinders the performance of MPNNs and poses a big challenge for the deep graph learning community.

Recent works have attempted to tackle the problem from three different lines of thought. 
\begin{itemize}[leftmargin=*]
    \item Learning differentiated global information by randomly labeling nodes \cite{ DEGNN, PGNN}.
    But such randomness \textbf{loses the property of equivariance and the ability to handle graph-level tasks}. 
    \item Building GNNs with higher-order graph isomorphism test capability (\textit{e.g.,} k-WL test and Folklore test \cite{WL}), but higher-order GNNs lead to \textbf{high computational complexity} \cite{GNN21ICLR, breaking, FGNN, KGNN}. 
    \item Designing the feature aggregator from new perspectives, such as ego graphs \cite{IDGNN} and overlapping subgraphs \cite{GraphSNN}. The potential problem is that it deviates from the message-passing mechanism, making it \textbf{difficult to apply} to downstream tasks.
\end{itemize}
Despite the significant progress in the study of the expressiveness of GNNs, there is still a lack of simple but effective approaches that make GNNs powerfully expressive and, at the same time, applicable to downstream tasks. This paper aims to show that it is easy to break through the limitation of the 1-WL test in MPNNs by applying a simple plug-in encoding. To be specific, we propose \textbf{E}quivariant \textbf{D}istance \textbf{EN}coding (abbreviate for \textbf{EDEN}), which is inspired by recent beyond-1-WL GNNs and the position encoding in Transformer \cite{transformer}. 
As shown in Fig.\ref{Fig.eden}, EDEN is derived from a series of interpretable transformations on the distance matrix consisting of the length of the shortest path between nodes.
EDEN is actually a deterministic position encoding with global information because the distance matrix is innate for a certain graph. The unique advantage of EDEN is the property of permutation-equivariance for all-level graph representation learning, providing powerful expressiveness in graph isomorphism tests and downstream tasks.
\begin{figure*}[t]
    \centering  
    \includegraphics[width=0.9\textwidth]{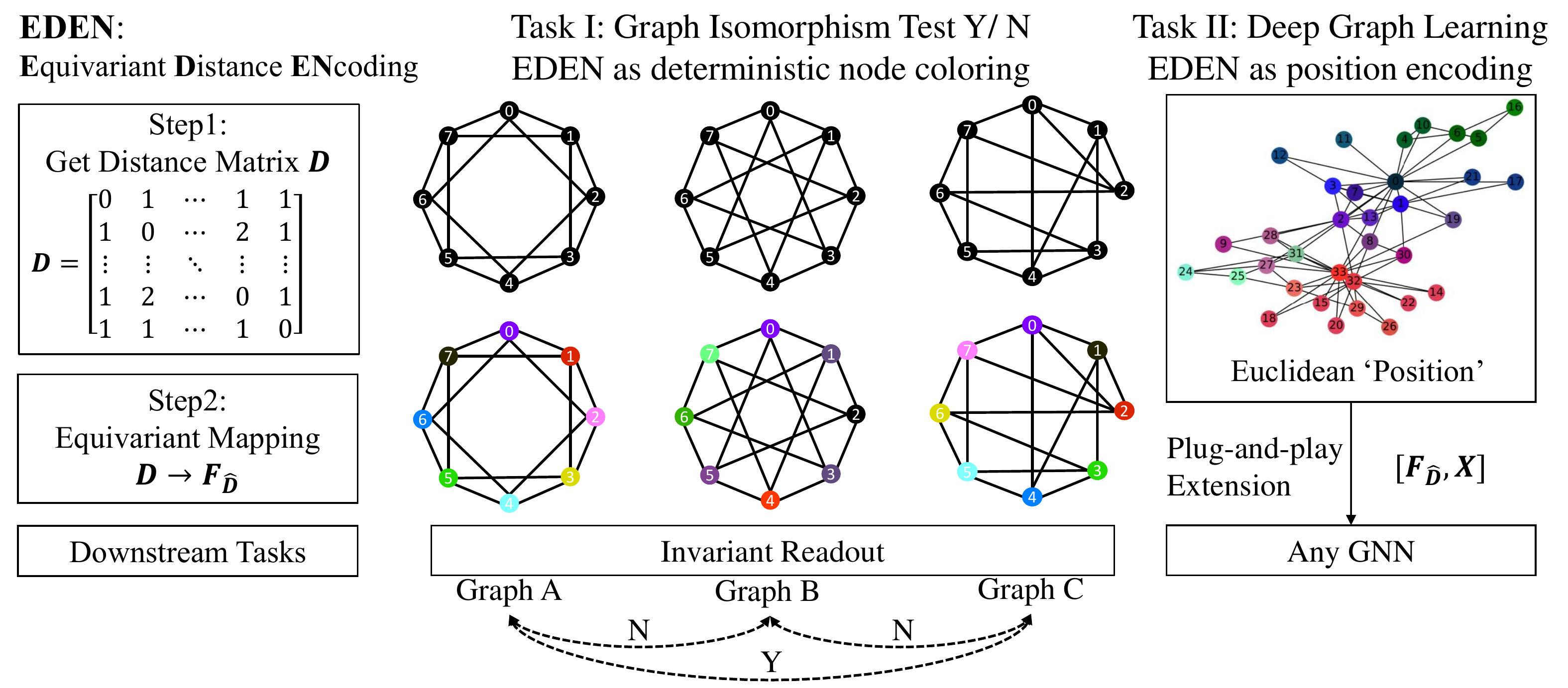}
    \caption{\textbf{E}quivariant \textbf{D}istance \textbf{EN}coding, namely \textbf{EDEN}, is obtained from the distance matrix of the graph through equivariant mapping, where the distance matrix records the shortest path length between two nodes.
    In the graph isomorphism test,  
    EDEN is used as a deterministic node coloring and enables classical GNNs to distinguish regular graphs (even 4-regular graphs). EDEN is also a plug-in embedding representing nodes' global position to improve MPNNs' performance on downstream tasks.}
    \label{Fig.eden}
\end{figure*}
Our major contributions include:
\begin{itemize}[leftmargin=*]
    \item We propose a simple but effective distance encoding called EDEN, and we theoretically prove that EDEN satisfies the property of permutation-equivariance in all-level graph representation learning, which is unavailable for existing plug-in encoding methods for GNN. 
    \item We generalize the cosine positional encoding to non-Euclidean spaces via phase propagation. 
    We empirically show that the expressive power of EDEN reaches up to the 3-WL test on graph isomorphism test datasets. 
    \item As a plug-in, EDEN can greatly improve the performance on downstream tasks when combined with conventional MPNNs. Moreover, experiments show that it is even superior to state-of-the-art methods.
\end{itemize}


\section{Preliminaries}
 Let $G = (V, E)$ denote a graph, where $V$ is the set of nodes and $E$ is the set of edges. $|V| = n$ and $|E|$ represent the number of nodes and edges, respectively. 
 A graph $G$ can also be represented as a composition of the adjacency matrix $\boldsymbol{A} \in \mathbb{R}^{n\times n}$ and $\boldsymbol{X} \in \mathbb{R}^{n\times l}$ formed by the feature embeddings of nodes, where $l$ is the dimension of the node embedding. 
 The row order of $\boldsymbol{X}$ is the same as $\boldsymbol{A}$.

\paragraph{Permutation-Equivariance, Permutation-Invariance and Graph Isomorphism}
We denote $\sigma \star$ as a permutation operation on graph $G$, \textit{i.e.,} re-arranging nodes in $V$ with given order $\sigma$. Permuting node order is equivalent to swapping the corresponding rows and columns in $\boldsymbol{A}$, and rows in $\boldsymbol{X}$.  $f,g$ are two functions defined on the graph $G=(\boldsymbol{A}, \boldsymbol{X})$.

\begin{definition}
Function $g$ is \textbf{permutation-equivariant} if $g(\sigma \star G)=\sigma \star g(G)$, function $f$ is \textbf{permutation-invariant} if $f(\sigma \star G)=f(G)$, and Graph $G$ and $G'$ are \textbf{isomorphisc} if $\exists \sigma \text{, \textit{s.t.}}\quad G =\sigma\star G'$.
\end{definition}


\begin{lemma}
\label{lemma:fg}
The composition of $f$ and $g$ is permutation-invariant if $f$ and $g$ satisfy the property of permutation-invariant and permutation-equivariant, respectively.
$$Proof.\quad f( g(\sigma \star G))=f(\sigma \star g(G))=f(g(G)).$$
\end{lemma}
   
In other words, equivariance makes the output order correspond to the input order, while invariance makes the output independent of the input order. 
Since nodes in the graph are naturally order-independent, changing the node's ID should keep its embedding unchanged. 
For node-level graph representation learning, nodes' feature embeddings will be permuted if we change their order in the graph, which satisfies permutation-equivariance. 
Regarding graph-level representation learning, permuting nodes' ID does not affect the output of the graph embedding (e.g., graph average pooling), which satisfies permutation-invariance. 
\paragraph{Message-passing based GNN}
We define MPNNs that satisfy the form defined in GIN~\cite{GIN}:
\begin{equation}
\label{eq:mpnn}
\begin{aligned}
&\mathbf{m}_{u}^{(k)}=\operatorname{MSG}^{(k)}\left(\mathbf{h}_{u}^{(k-1)}\right) \\
&\mathbf{h}_{v}^{(k)}=\operatorname{AGG}^{(k)}\left(\left\{\mathbf{m}_{u}^{(k)}, u \in \mathcal{N}(v)\right\}, \mathbf{h}_{v}^{(k-1)}\right),
\end{aligned}
\end{equation}
where $\mathbf{h}_{v}^{(k)}$ is node $v$'s feature embedding in the $k$-th layers.
$\mathbf{h}_{v}^{(0)}=\mathbf{x}_{v}$, and $x_v$ is node $v$'s feature embedding.
$\mathbf{m}_{v}^{(k)}$ is the message embedding, and $\mathcal{N}(v)$ is the set of (1-hop) neighborhoods of $v$.
Different MPNNs differ in MSG and AGG. 

\section{EDEN:An Equivariant Distance Encoding}
\label{sec:EDEN}
In this section, we describe EDEN in detail. 
We propose our crucial insight that \textbf{the position of a node in the graph is an inherently universal feature that can be represented by measuring the distance to other nodes}.
Inspired by the sequential positional encoding method in Transformer \cite{transformer} and the distance-based GNNs \cite{DEGNN, PGNN}, we generalize the cosine positional encoding to non-Euclidean spaces and utilize mathematical transformations to make it permutation-equivariant. 

Let $\boldsymbol{D} \in \mathbb{R}^{n\times n}$ denote the distance matrix, where ``distance'' means the shortest path length between two nodes in $G$ obtained by the \emph{Dijkstra's algorithm} \cite{Dijkstra}. The diagonal elements of $\boldsymbol{D}$ are all 0. We define the ``diameter'' $d_i$ of node $v_i$ as follows:

\begin{definition}
 \textbf{Diameter} is the longest distance among reachable nodes of a given node, 
\begin{equation}
	\label{eq:diameter}
	d_i = \max_{j} \boldsymbol{D}_{i,j}, \quad \text{where } \boldsymbol{D}_{i,j}\neq \infty.
\end{equation}
\end{definition}

We map the $i$-th row of the distance matrix $\boldsymbol{D}$ \textit{w.r.t.} its diameter $d_i$ by the cosine function to obtain a normalized matrix $\hat{\boldsymbol{D}}$. We call this operation \textbf{phase propagation}:
\begin{equation}
	\label{eq:phase}
	\hat{\boldsymbol{D}}_{i,j}=\left\{
	\begin{aligned}
	    & \cos(\frac{\pi}{d_i} \times \boldsymbol{D}_{i,j})  ,&\boldsymbol{D}_{i,j}\neq \infty\\
	    & -1.5  ,&\boldsymbol{D}_{i,j} = \infty.
	\end{aligned}
	\right.
\end{equation}
where $\frac{\pi}{d_i} \times \boldsymbol{D}_{i,j}$ means node $v_i$'s all reachable nodes (including $v_i$ itself) are mapped to phase $[0, \pi]$. 
Due to the nature of the cosine function, the distance between node $v_i$ and its closest neighbor (it is actually itself) is mapped to 1, and the distance between $v_i$ and its farthest reachable node is mapped to -1. We use -1.5 (as long as it is less than -1) to represent the distance between unreachable nodes. 


The processed matrix $\hat{\boldsymbol{D}}$ has the following properties:
\begin{itemize}[leftmargin=*]
	\item \textbf{Normalization}: The computed phase $\frac{\pi}{dia_i} \times \boldsymbol{D}_{i,j}$ in Eq.  (\ref{eq:phase}) ranges in $[0, \pi]$, so $\hat{\boldsymbol{D}}_{i,j} \in [-1,1]$.
	\item \textbf{Distinctiveness}: Each row in $\hat{\boldsymbol{D}}$ is distinguished from others because only $\hat{\boldsymbol{D}}_{i,i} = 1$ and only $\boldsymbol{D}_{i,i} = 0$.
	\item \textbf{Interpretability}: The element $\hat{\boldsymbol{D}}_{i,j}$ is an interpretable transformation of the distance between node $i$ and $j$.
\end{itemize}
 Matrix $\boldsymbol{D}$ and $\hat{\boldsymbol{D}}$ do not have permutation equivariance if we consider each row as the feature embedding of the corresponding node.
To develop an intuition for this phenomenon, consider that we swap node $v_0$ and $v_1$, and the corresponding rows $\boldsymbol{D}_{0,:}$ and $\boldsymbol{D}_{1,:}$ should also be swapped if $\boldsymbol{D}$ is permutation-equivariant. Nevertheless, this contradicts the fact that the diagonal elements $\boldsymbol{D}_{i,i}=0$ (\textit{i.e.}, \textbf{Distinctiveness}). In addition, $\hat{\boldsymbol{D}}$  as the embedding matrix with the dimension of $n$ is undoubted too large for downstream tasks. So we employ the Principal Component Analysis (PCA) \cite{PCA} to reduce the dimension of $\hat{\boldsymbol{D}}$ to the size of $n \times m$, where $m$ is a manually set hyperparameter: 
\begin{equation}
\label{eq:pca}
    \boldsymbol{F_{\hat{\boldsymbol{D}}}}= \text{PCA} (\hat{\boldsymbol{D}}).
\end{equation}
$\boldsymbol{F_{\hat{\boldsymbol{D}}}}$ is our proposed \textbf{Equivariant Distance ENcoding} (EDEN). 

\subsection{A theoretical proof of EDEN's equivalence}
\label{sec:proof}
\begin{definition}
 \textbf{Permutation matrix} $\boldsymbol{P}$ is a square binary matrix that has exactly one entry of 1 in each row and each column and 0s elsewhere \cite{AlgebraicGraph}. 
\end{definition}
\begin{lemma}
\label{the:orthogonal}
As permutation matrices are orthogonal matrices (\textit{i.e.}, $\boldsymbol{P}_{\sigma} \boldsymbol{P}_{\sigma}^{\top}=\boldsymbol{I}$ ), the inverse matrix exists and can be written as
$ \boldsymbol{P}^{-1}=\boldsymbol{P}^{\top}$.
\end{lemma}

\begin{definition}
\label{def:sim}
Matrices $\boldsymbol{A}, \boldsymbol{B} \in \mathbb{R}^{n\times n}$ are called \textbf{similar} (\textit{i.e.}, $\boldsymbol{A}\sim \boldsymbol{B} $) if there exists an invertible matrix $\boldsymbol{T}\in \mathbb{R}^{n\times n}$ such that $\boldsymbol{B}=\boldsymbol{T}^{-1} \boldsymbol{A} \boldsymbol{T}$.
 And then  $\boldsymbol{A}$ and $\boldsymbol{B}$ share same eigenvalues \cite{linear}.
\end{definition}


Considering a distance matrix $\boldsymbol{D}$ and permutation matrix $\boldsymbol{P}$, the permuted distance matrix $\boldsymbol{D}'$ is:
\begin{equation}
    \label{eq:permute}
   \boldsymbol{D'} = \boldsymbol{P^{\top} D P}.
\end{equation}



\begin{lemma}
 \textbf{Node-level Permutation-equivariant} Considering a distance matrix $\boldsymbol{D}$ and permuted distance matrix $\boldsymbol{D'} = \boldsymbol{P^{\top} D P}$, let $ \boldsymbol{F_D}= \text{PCA} (\boldsymbol{D}) , \boldsymbol{F_{D'}}= \text{PCA} (\boldsymbol{D'}) $
, we have 
\begin{equation}
    \boldsymbol{F_{D'}} =  \boldsymbol{P^{\top} F_D}.
\end{equation}

\end{lemma}
\begin{proof}
Since 
  \begin{equation}
  \label{eq:D'D}
    \boldsymbol{D'^{\top} D'} =  (\boldsymbol{P^{\top} D P})^{\top} (\boldsymbol{P^{\top} D P}) = \boldsymbol{P^{-1} D^{\top} D P},
  \end{equation}
with Definition \ref{def:sim}, we know that $\boldsymbol{D'^{\top} D'}$ and $\boldsymbol{D^{\top} D}$ share same eigenvalues because $ \boldsymbol{D'^{\top} D'} \sim \boldsymbol{D^{\top} D}.$
  We perform eigen-decomposition on them to employ PCA,
    \begin{equation}
  \label{eq:eigen decomposition}
        \boldsymbol{D^{\top} D} = \boldsymbol{V \Sigma_{D} V^{\top}},
        \boldsymbol{D'^{\top} D'} = \boldsymbol{W \Sigma_{D'} W^{\top}}, \\
  \end{equation}
 the i-th columns of $\boldsymbol{V}$ and $\boldsymbol{W}$ are eigenvectors corresponding to the i-th largest eigenvalues of $\boldsymbol {D^{\top} D}$ and $\boldsymbol {D'^{\top} D'}$.
  $\boldsymbol{\Sigma_{D}}$ and  $\boldsymbol{\Sigma_{D'}}$ are diagonal matrices consisting of ranked eigenvalues. Notice that $\boldsymbol{ \Sigma_{D} = \Sigma_{D'} = \Sigma }$
  because $\boldsymbol{D'^{\top} D'}$ and $\boldsymbol{D^{\top} D}$ share same eigenvalues.
  Select the first $m$ columns of $\boldsymbol{V}$ and $\boldsymbol{W}$ , and then we have $\boldsymbol{\widetilde{V}}$ and $\boldsymbol{\widetilde{W}}$ which can be used to project the data from $n$ down to $m$ dimensions according to the definition of PCA \cite{PCA}: 
  \begin{equation}
  \begin{aligned}
   \boldsymbol{F_D} &= \text{PCA}(\boldsymbol{D}) = \boldsymbol{D \widetilde{V}}\\
      \boldsymbol{F_{D'}} &= \text{PCA}(\boldsymbol{D'}) = \boldsymbol{D' \widetilde{W}},\\
  \end{aligned}
  \end{equation}{}
  since \eqref{eq:D'D} and \eqref{eq:eigen decomposition},
\begin{equation}
  \boldsymbol{D'^{\top} D'} = \boldsymbol{P^{\top} D^{\top} D P} = \boldsymbol{P^{\top} V \Sigma V^{\top} P},
  \end{equation}
we have $\boldsymbol{W = P^{\top} V}$.
The left multiplication matrix $\boldsymbol{P^{\top}}$ doesn't affect the operation of taking the first $\boldsymbol{m}$ columns of the matrix by row transformation, so
\begin{equation}
    \boldsymbol{ \widetilde{W} = \widetilde{P^{\top} V} = P^{\top} \widetilde{V}}.
\end{equation}{}
In summary, 

\begin{equation*}
\boldsymbol{F_{D'}} = \boldsymbol{D' \widetilde{W}} = \boldsymbol{D' P^{\top} \widetilde{V}}= \boldsymbol{P^{\top} D P P^{\top} \widetilde{V}} = \boldsymbol{P^{\top} F_D}.
 \end{equation*}
 
It is clear that swapping the order of phase propagation in Eq.~\eqref{eq:phase} and matrices permutation in Eq.~\eqref{eq:permute} does not change the result:
\begin{equation}
   \boldsymbol{\hat {D'} } = \boldsymbol{\widehat {P^{\top} D P}} = \boldsymbol{P^{\top} \hat {D} P},
\end{equation}
we can draw a similar conclusion that $\boldsymbol{F_{\hat{D'}} = P^{\top} F_{\hat{D}} }$ is node-level permutation-equivariant.
\end{proof}
\begin{lemma}
\textbf{The eigen-decomposition is not unique.}
\label{ED no unique}

Proof.
We assume that the eigenvectors are standard vectors. Considering the following two cases: (1) When the eigenvalue is unique (i.e., the eigenvalue is not equal to other eigenvalues), its corresponding eigenvector has a sign difference: if $X_i$ is the eigenvector corresponding to the i-th eigenvalue, then $-X_i$ is also the eigenvector. (2) When the eigenvalues are not unique (i.e., there are $p$ eigenvalues with the same value), the dimension $k$ of the eigenvector space corresponding to that eigenvalue will be less than or equal to $p$ because the geometric multiplicity is less than or equal to the algebraic multiplicity. 
The basis of the eigenvector space is the eigenvector. When $k>1$, there are infinite possibilities of bases, and different bases can be interconverted by orthogonal matrices.
\end{lemma}

It is important to note that the sign flipping problem mentioned in the first case of lemma \ref{ED no unique} can be avoided in practice.
The deterministic outputs of PCA or the eigen-decomposition operation used in Eq.~\eqref{eq:eigen decomposition} can be ensured by a given rule\footnote{ \url{https://github.com/scikit-learn/scikit-learn/blob/baf0ea25d/sklearn/decomposition/_pca.py#L519}} to guarantee EDEN's equivariance.

\subsection{Complexity analysis}
\label{sec:complexity}
For a graph with $n$ nodes, the computational complexity of EDEN mainly lies in obtaining the distance matrix $\boldsymbol{D}$ and PCA, both of which have the complexity of $O(n^3)$.
The space complexity required to store and transform $\boldsymbol{D}$ is $O(n^2)$. 
Although the total complexity reaches $O(n^3)$, we show that it is still efficient to apply EDEN in practice:
(1) Recent advanced expressive GNNs (e.g., F-GNN \cite{FGNN} and GNNML3 \cite{breaking}) modify the message-passing scheme, and their time complexity and space complexity are $O(n^3)$ and $O(n^2)$ respectively through training and inference. But EDEN does not change the computational primitives (i.e., message-passing) when applied to deep graph learning. It works as a simple plug-in, so the total complexity of EDEN is a ``one-shot deal'', which means the additional time complexity incurred during inference is only $O(k \times n)$. 
(2) The time complexity of most existing eigenvalue-based decomposition methods for position encoding is $O(n^3)$, e.g., the state-of-the-art Laplacian PE \cite{benchmarking}, and the space complexity of Laplacian PE is at least $O(n^2)$. 
(3) EDEN only takes $O(k \times n)$ space after PCA, and $k \ll n$ is fully applicable for large real-world graphs, but recent advanced high-order GNNs require larger memory resources (See Tab. \ref{tab:experiment}).

\section{Experiments}
We provide the pseudo code and core python implementation of EDEN in Appendix \ref{app:code}.
\label{sec:Experiments}
\subsection{EDEN for Graph Isomorphic Test}
\label{exp:iso}

  

\begin{table}

  \centering
  \small
  \begin{tabular}{llll}
  \toprule
    Model     & GRAPH8C     & SR25 & EXP \\
    \midrule
    MLP & 293K  & 105  & 600\  \\
    GCN     & 4775 & 105   &  600 \\
    GAT      & 1828       & 105 & 600 \\
    GIN      & 386       & 105 & 600\\
    CHEBNET      & 44       & 105 & 71\\
    F-GNN     & \textbf{0}       & 105 & \textbf{0}\\
    GNNML1     & 333       & 105 & 600\\
    GNNML3     & \textbf{0}       & 105 & \textbf{0} \\
    \textbf{EDEN+GIN}     & \textbf{0}       &  \textbf{101} & \textbf{0}\\
    \bottomrule
    \end{tabular}
    \caption{
    The results of graph isomorphic test.
    The total number of pairs is 61M, 105, and 600 for the GRAPH8C, SR25, and EXP datasets, respectively. The data in the table represents the number of MISJUDGED graph pairs.}
     \label{tab:isomorphic}
\end{table}

  We test the expressive power of EDEN on graph isomorphism test datasets of GRAPH8C, SR25  \footnote{\url{http://users.cecs.anu.edu.au/∼bdm/data/graphs.html}}, and EXP \cite{randinit}. 
  GRAPH8C has 11,177 non-isomorphic graphs with eight nodes, which contains $C^{2}_{11117}=62,401,191$ possible different pairs. In GRAPH8C, 312 of 62M pairs are 1-WL equivalent, which means only GNNs with at least 2-WL expressiveness can distinguish them. 
  The SR25 dataset contains 15 strong 12-regular graphs with 25 nodes, which contains $C^{2}_{15}=105$ possible non-isomorphic pairs. The EXP dataset has 600 pairs of 1-WL equivalent non-isomorphic graphs. 
  
Following the previous settings \cite{breaking}, we compare EDEN (with a GIN readout) with the following baselines: MLP, GCN \cite{GCN}, GAT \cite{GAT}, GIN \cite{GIN}, CHEBNET \cite{cheb}, F-GNN \cite{FGNN},  GNNML1 and GNNML3 \cite{breaking}. We iteratively initialize GNN models 100 times with different random seeds to readout graphs. A pair of graphs are considered non-isomorphic if they can be separated at least once. 

The results are shown in Tab. \ref{tab:isomorphic}, which describes the number of MISJUDGED graph pairs. Note that F-GNN and GNNML3 are designed \textit{w.r.t.} the 3-WL test. It shows that EDEN can help traditional MPNN (GIN) surpass recent advanced GNNs and improve the expressive power beyond the 1-WL test. 
Fig. \ref{Fig.eden} illustrates a real example of using EDEN on three 4-regular graphs, where nodes' colors are normalized 3-dimensional EDEN and are presented as the RBG brightness to facilitate visualization. 
It can be observed that nodes' colors on the two non-isomorphic graphs are entirely different, and the isomorphic pair of graphs share the same colors. 
Intuitively, we can identify graph pairs by judging the colors of nodes on them. We also use EDEN on other graphs and find that up to 2-WL equivalent regular graphs can be correctly judged by nodes' colors. See Appendix \ref{app:regular} for more examples (1-WL, 2-WL, and 3-WL tests), and Appendix \ref{app:GIT} for detailed reproducible settings.

\begin{table*}[t]
\centering
\small
\begin{tabular}{ccccccc}
\toprule
\multirow{2}{*}{\textbf{}} & \multicolumn{2}{c}{\textbf{Node classification}} & \multicolumn{2}{c}{\textbf{Link prediction}} & \multicolumn{2}{c}{\textbf{Graph classification}} \\ 
\cline{2-7}      & Cora                        & CiteSeer                   & ENZYMES                   & PROTEINS                 & MUTAG                   & PROTEINS                    \\ \midrule
GCN                        & 0.864$\pm$0.019             & 0.713$\pm$0.012            & 0.657$\pm$0.004           & 0.613$\pm$0.023          & 0.837$\pm$0.027         & 0.704$\pm$0.008                 \\
SAGE                       & 0.858$\pm$0.016             & 0.728$\pm$0.031            & 0.643$\pm$0.003           & 0.658$\pm$0.002          & 0.842$\pm$0.050         & 0.714$\pm$0.042                 \\
GAT                        & 0.859$\pm$0.028             & 0.721$\pm$0.052            & 0.613$\pm$0.008           & 0.610$\pm$0.002          & 0.811$\pm$0.031         & 0.726$\pm$0.009        \\
GIN                        & 0.854$\pm$0.019             & 0.716$\pm$0.018            & 0.663$\pm$0.006 & 0.614$\pm$0.015          & 0.847$\pm$0.087                   & 0.718$\pm$0.012                 \\ \midrule
F-GNN                       & OOM                           & OOM                          & 0.790$\pm$0.053                         & 0.795$\pm$0.042                        & 0.906$\pm$0.087        & 0.772$\pm$0.047                 \\ 
ID-GNN                     & \textit{0.878$\pm$0.010}     &\textit{0.742$\pm$0.010}   & \textit{0.846$\pm$0.010}   & \textit{0.886$\pm$0.020} & \textit{\textbf{0.965$\pm$0.032}} & 0.780$\pm$0.035                 \\ 
GraphSNN                   & 0.838$\pm$0.012             & 0.735$\pm$0.016            & -                         & -                        & 0.947$\pm$0.019         & \textit{0.784$\pm$0.027}                \\
Laplacian PE                   & 0.869$\pm$0.027             & 0.733$\pm$0.036            & 0.713$\pm$0.052                         & 0.722$\pm$0.069                        & 0.855$\pm$0.097         & 0.735$\pm$0.044                \\
\midrule
EDEN+GCN                   & 0.873$\pm$0.023             & 0.748$\pm$0.013            & 0.869$\pm$0.001           & 0.879$\pm$0.001          & 0.926$\pm$0.031                   & 0.781$\pm$0.014        \\
EDEN+SAGE                  & \textbf{0.885$\pm$0.021}   & \textbf{0.750$\pm$0.020}  & \textbf{0.887$\pm$0.001} & \textbf{0.917$\pm$0.001}& 0.932$\pm$0.054                  & 0.755$\pm$0.019                 \\
EDEN+GAT                   & 0.870$\pm$0.019             & 0.734$\pm$0.020            & 0.876$\pm$0.002           & 0.881$\pm$0.001          & 0.911$\pm$0.027                   & \textbf{0.788$\pm$0.013}       \\
EDEN+GIN                   & 0.879$\pm$0.020             & 0.722$\pm$0.023            & 0.878$\pm$0.001           & 0.874$\pm$0.003          & 0.937$\pm$0.086                   & 0.744$\pm$0.076                 \\ 
\bottomrule
\end{tabular}

\caption{Performance comparison.
\textbf{Bold} denotes the highest score, \textit{italic} denotes the best performing baseline. Since GraphSNN is designed for node- and edge-level tasks, we only report the best performance on them.}
\label{tab:experiment}
\end{table*}

\subsection{EDEN for Downstream Tasks}
\label{exp:tasks}
We conduct experiments on five real-world datasets, including two citation networks of Cora and CiteSeer \cite{cora}, two protein datasets of ENZYMES \cite{ENZYMES} with 600 graphs and PROTEINS \cite{PROTEINS} with 1113 graphs, and a chemical dataset of MUTAG \cite{TU} with 188 graphs.

 In addition to the baseline methods GCN, GAT, GIN and F-GNN used in section~\ref{exp:iso}, we add conventional GNN of GraphSage~\cite{Sage}, recent advanced GNNs of ID-GNN~\cite{IDGNN}, GraphSNN~\cite{GraphSNN}, and the state-of-the-art Laplacian PE~\cite{benchmarking} as new baselines.
 We apply a random 80/20\%  train/validation split,  train all models by the Adam optimizer \cite{Adam} with the learning rate of 0.01, repeat experiments with 10 random seeds and report the average results. Please refer to Appendix \ref{app:details} for more detailed settings.

Tab. \ref{tab:experiment} describes the results. We can observe that EDEN helps MPNNs to gain universal improvements in all three tasks, and MPNNs with EDEN achieve state-of-the-art performance compared with beyond-1-WL GNN counterparts on all datasets except the MUTAG graph classification dataset. This is because the graphs contained in MUTAG are very small, and there are only 17.93 nodes per graph on average.
EDEN's global information becomes less useful for very small graphs.
Combining EDEN with the traditional MPNNs of GCN, GraphSAGE, GAT, and GIN can bring promising improvements, especially on the link prediction task.
Note that F-GNN requires higher-order computations during inference and raises the Out-Of-Memory (OOM) error. The performance of Laplacian PE \cite{benchmarking} has an obvious gap between EDEN in all tasks.

\begin{figure*}
    \centering  
    \subfigure[EDEN indicates position]{
        \label{Fig.sub.1}
        \includegraphics[width=0.23\textwidth]{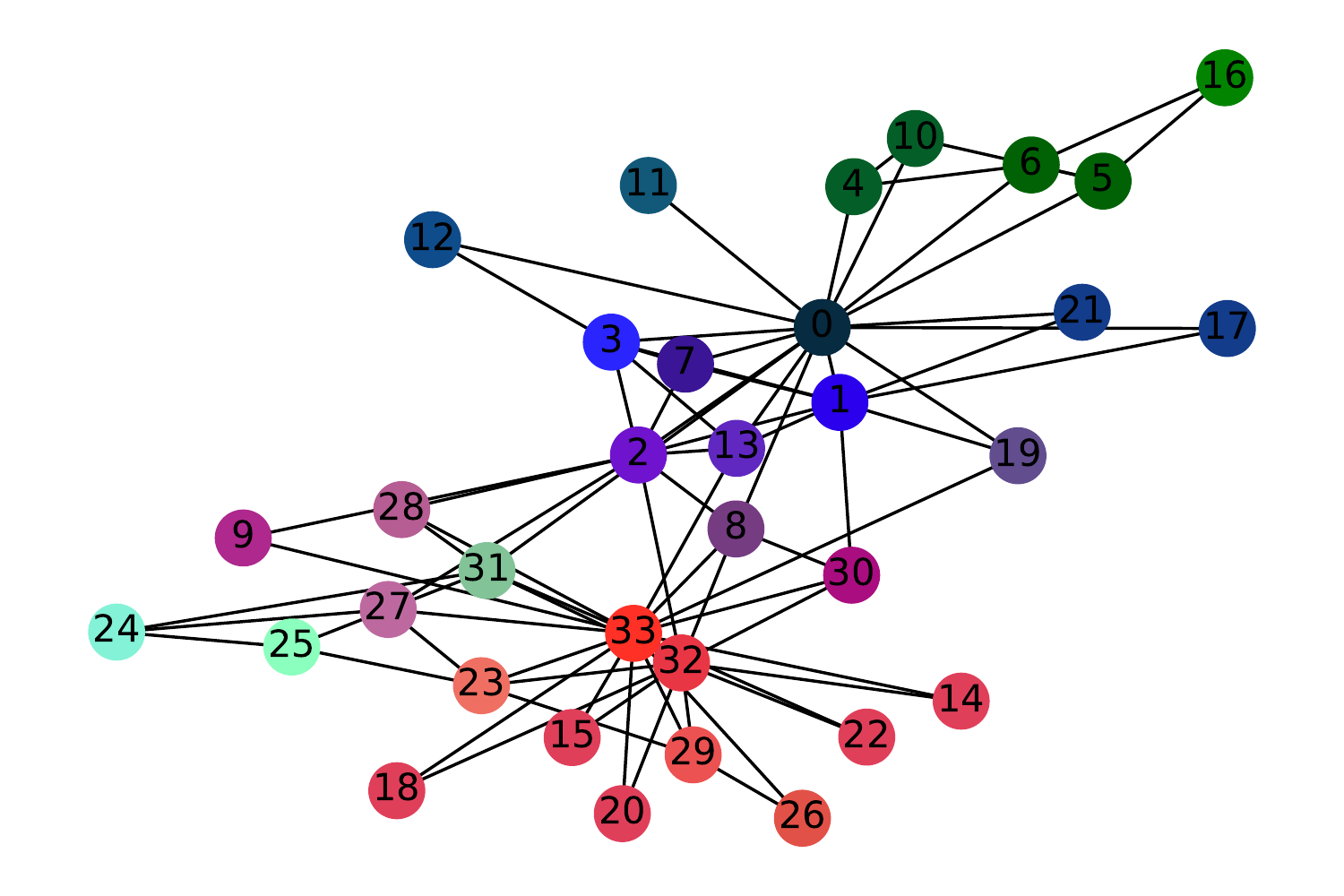}
        }
    \subfigure[3D Laplacian]{
        \label{Fig.ablation.1}
        \includegraphics[width=0.23\textwidth]{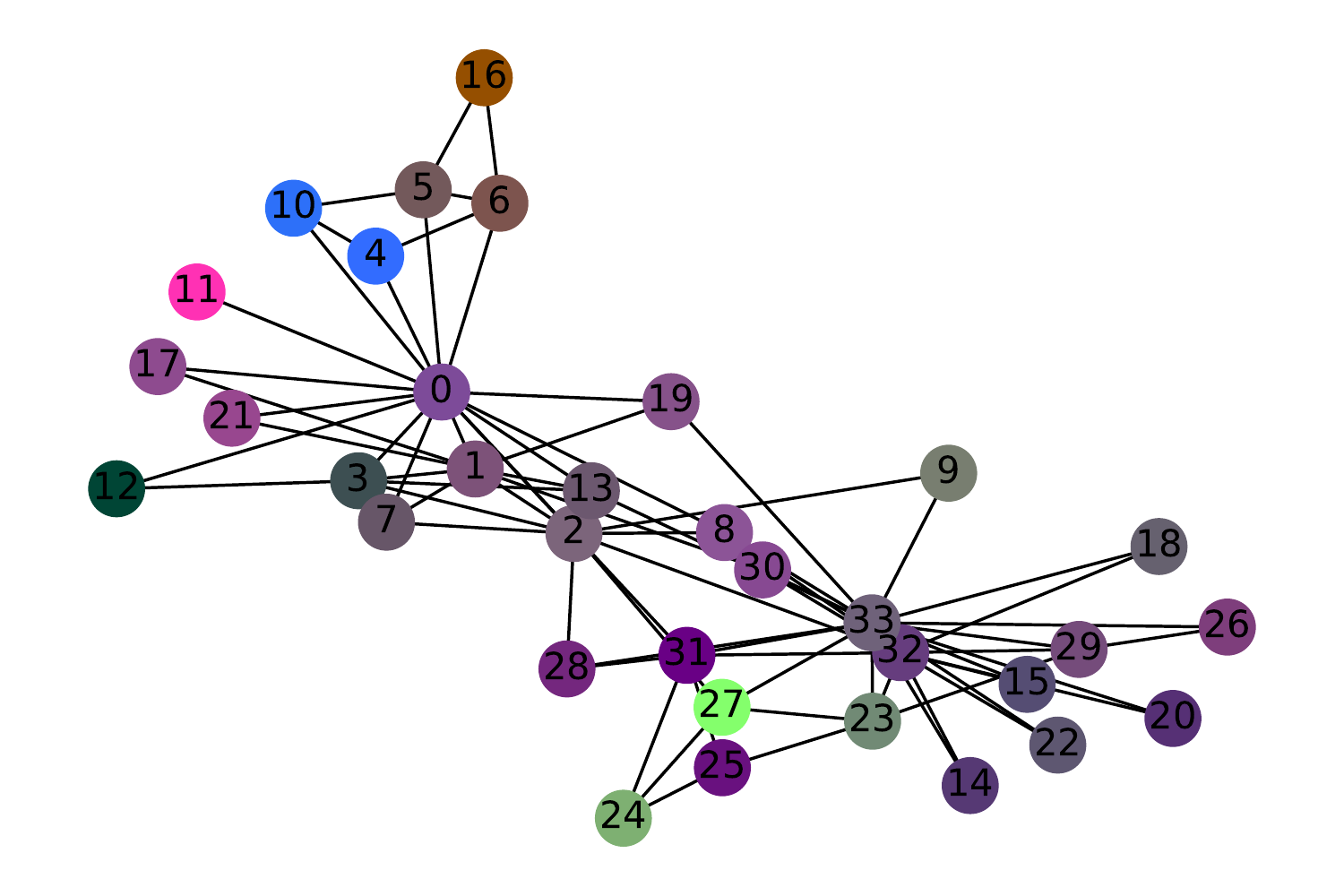}
        }            
    \subfigure[2D EDEN \textbf{WITHOUT USING LABELS}.]{
        \label{Fig.sub.2}
        \includegraphics[width=0.22\textwidth]{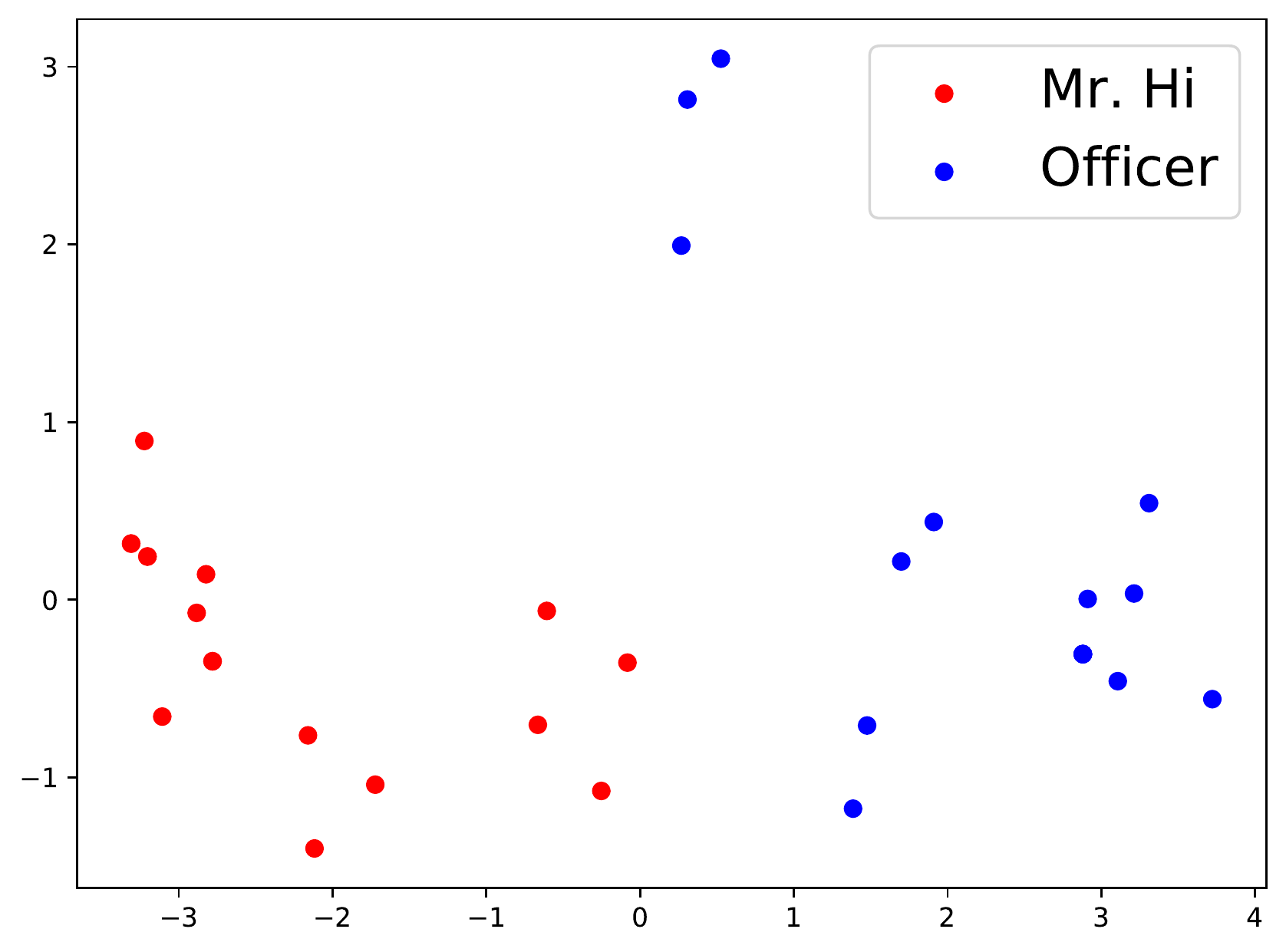}
        }
    \subfigure[2D Laplacian]{
        \label{Fig.ablation.2}
        \includegraphics[width=0.247\textwidth]{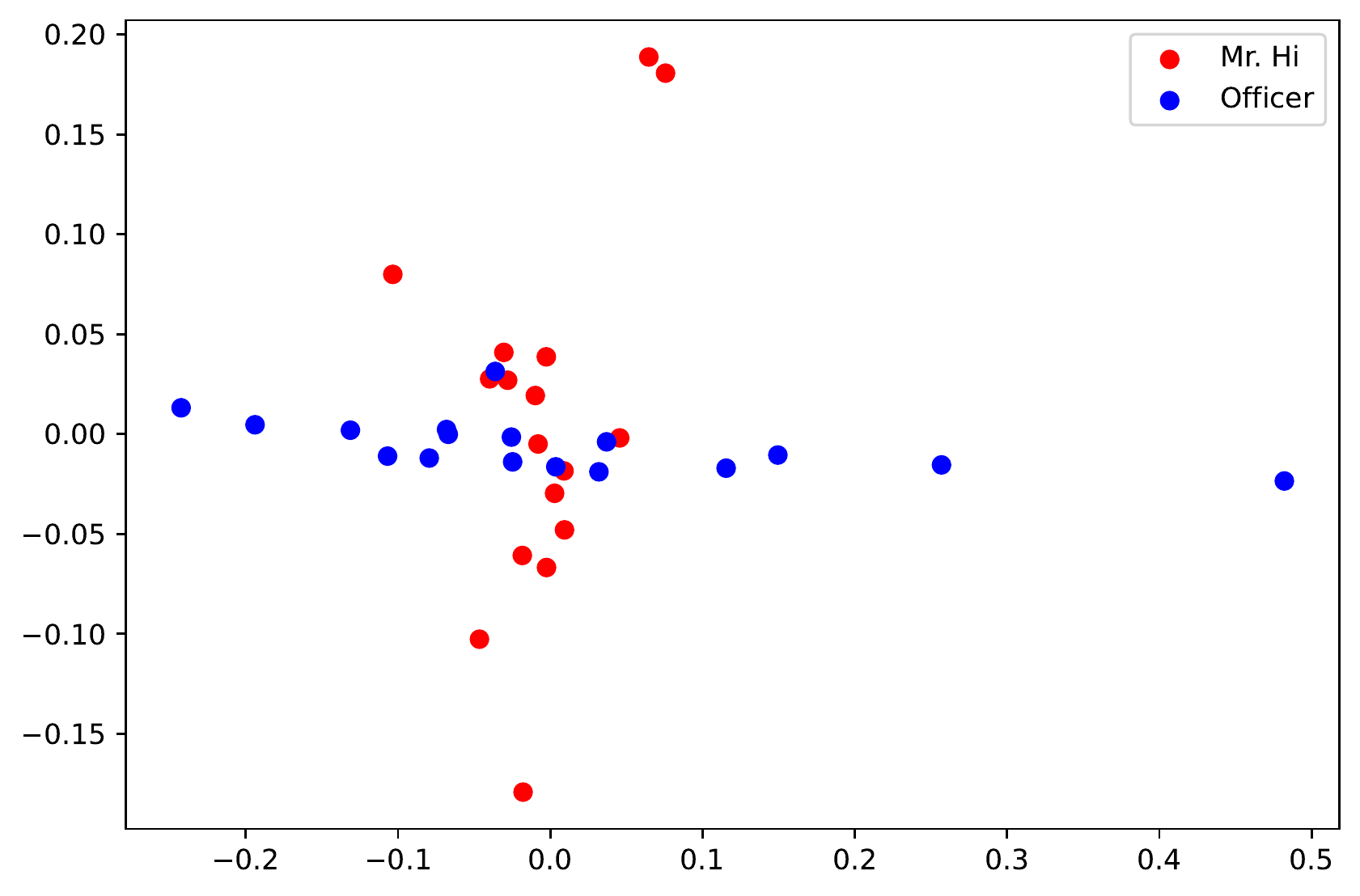}
        }        
    \caption{The expressive power of EDEN $\boldsymbol{F_{\hat{\boldsymbol{D}}}}$ and Laplacian PE on Zachary's karate club network \cite{karate}. (a) EDEN as a Red-Green-Blue value to color each node. 
    (b) Laplacian PE makes most nodes have similar colors. (c) The 2D scatter of EDEN. Nodes are linearly separable WITHOUT labels (red or blue).
    (d) The result of 2D Laplacian PE.}
    \label{Fig.ablation}
\end{figure*}

%


\subsection{Ablation Study}

\begin{figure}
    \centering  
    \subfigure[Laplacian $G$]{
        \label{Fig.abllap.1}
        \includegraphics[width=0.305\linewidth]{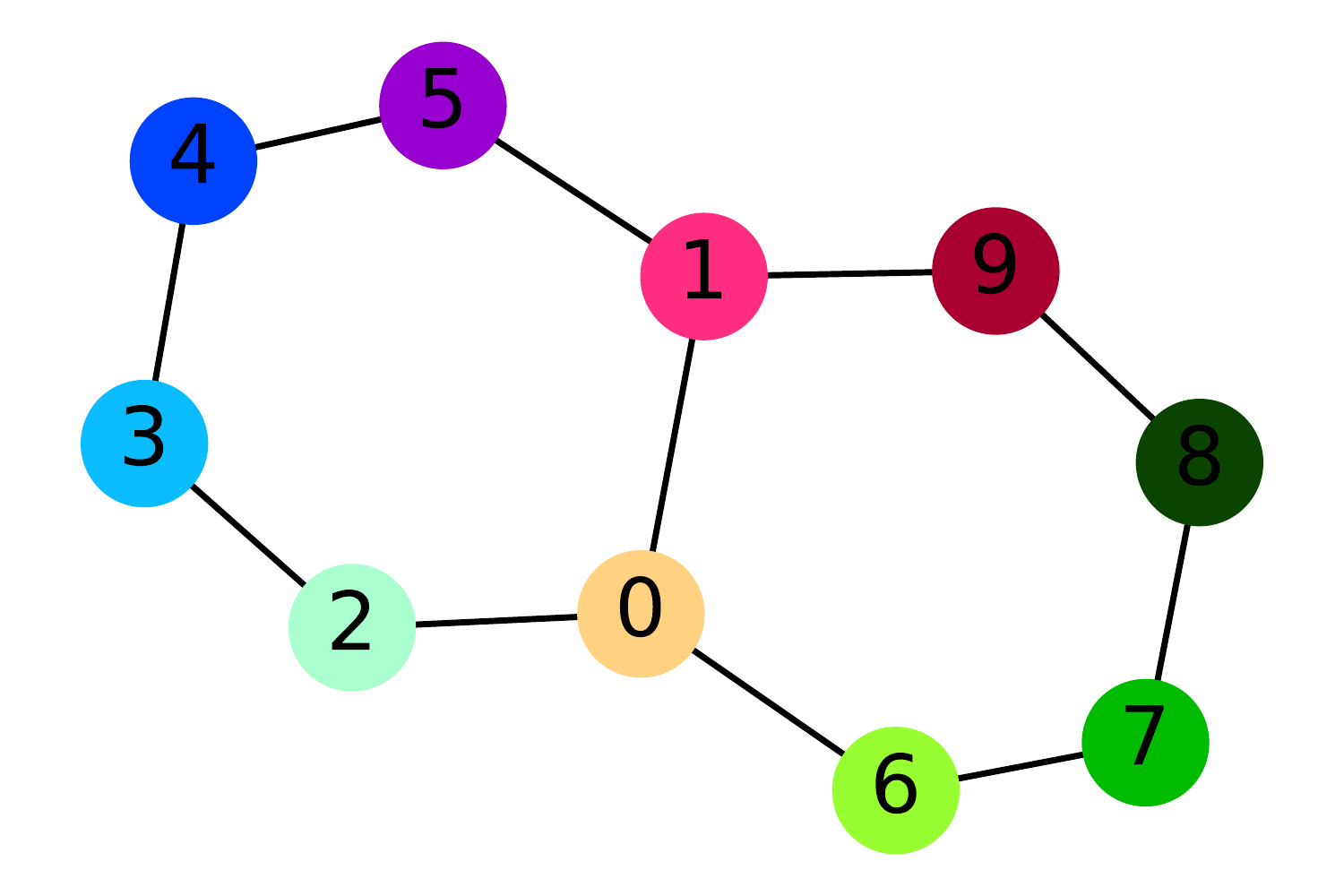}
        }
    \subfigure[Laplacian $H$]{
        \label{Fig.abllap.2}
        \includegraphics[width=0.305\linewidth]{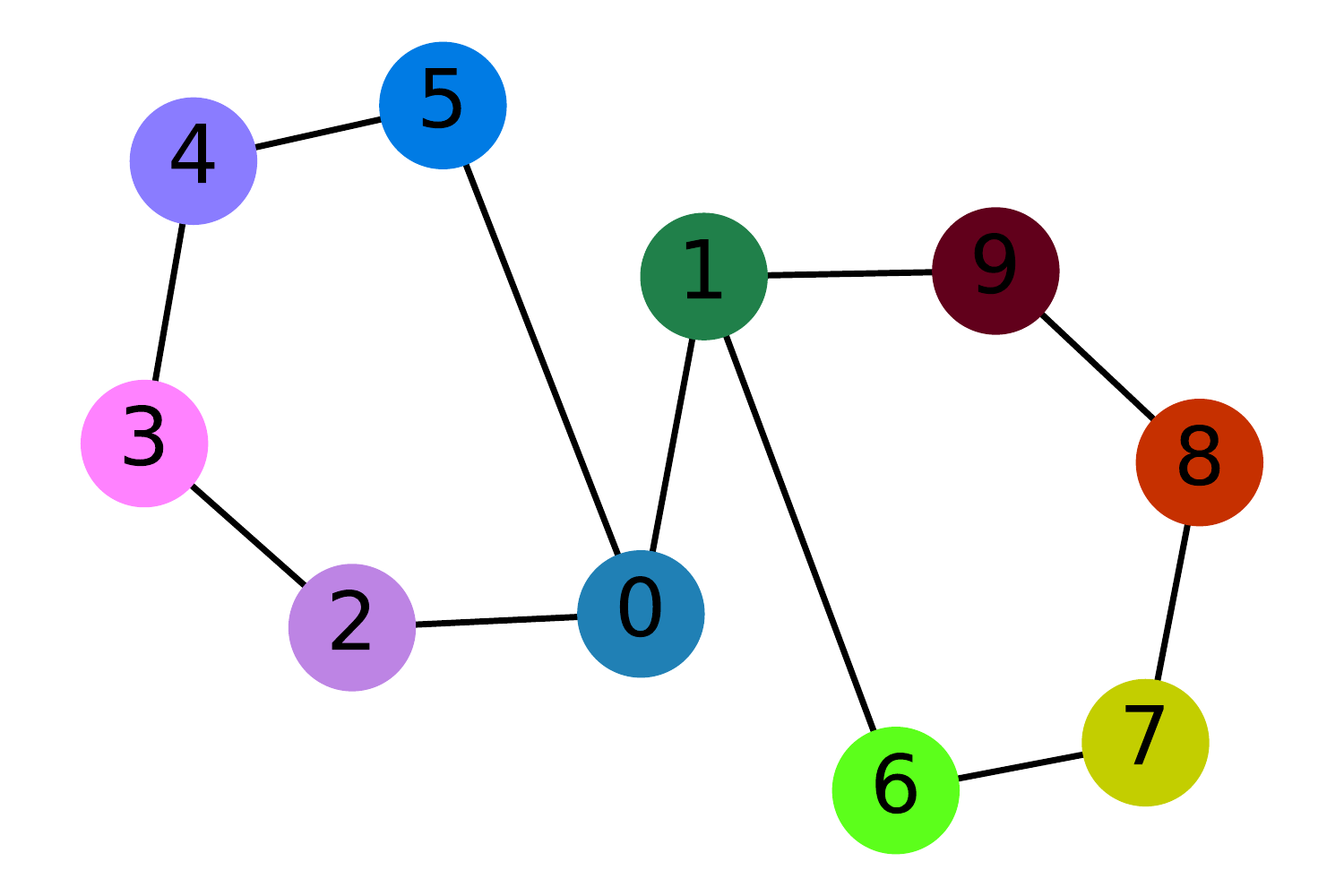}
        }
    \subfigure[Laplacian $G'$]{
        \label{Fig.abllap.3}
        \includegraphics[width=0.305\linewidth]{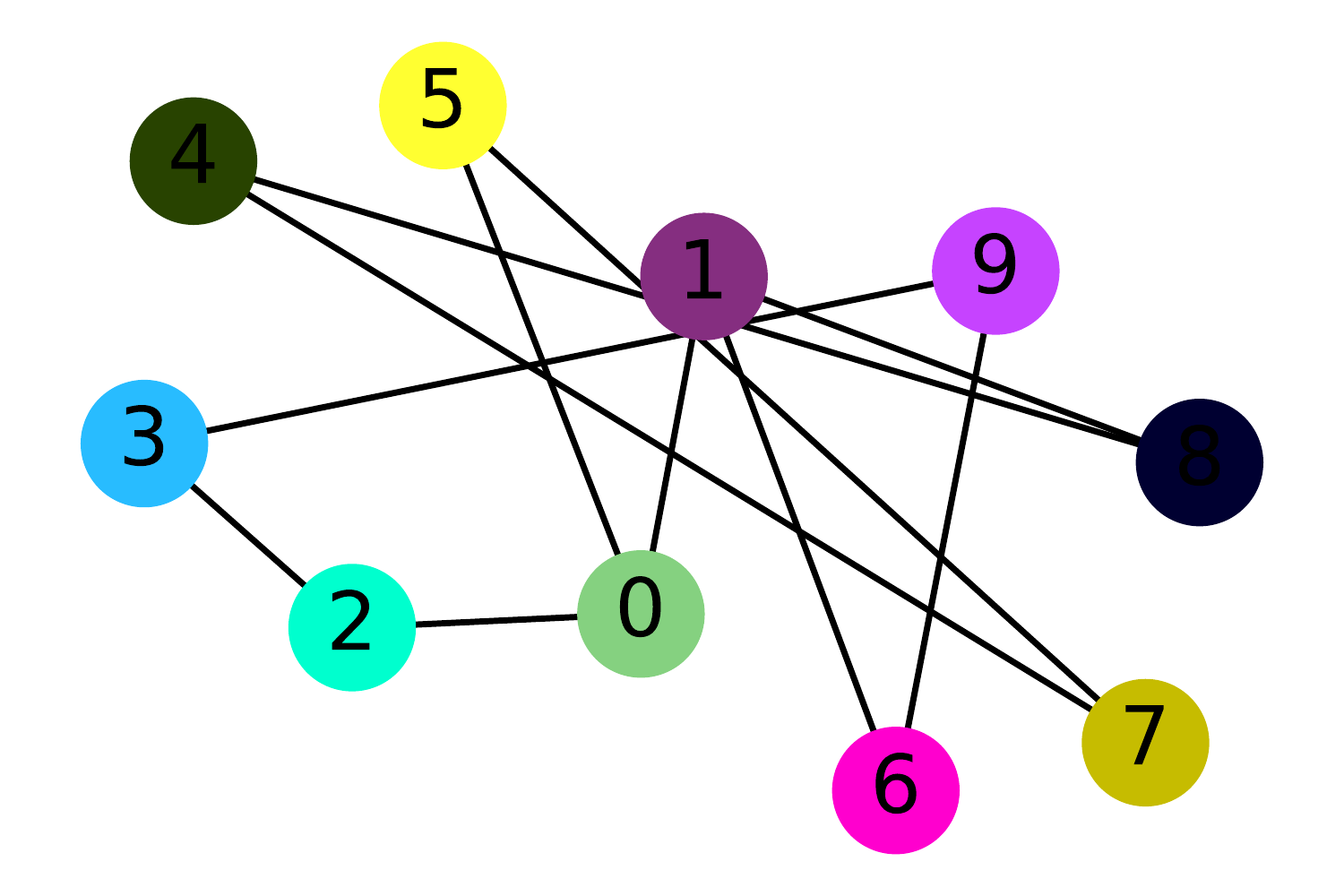}
        }
    \subfigure[EDEN $G$]{
        \label{Fig.ableden.1}
        \includegraphics[width=0.305\linewidth]{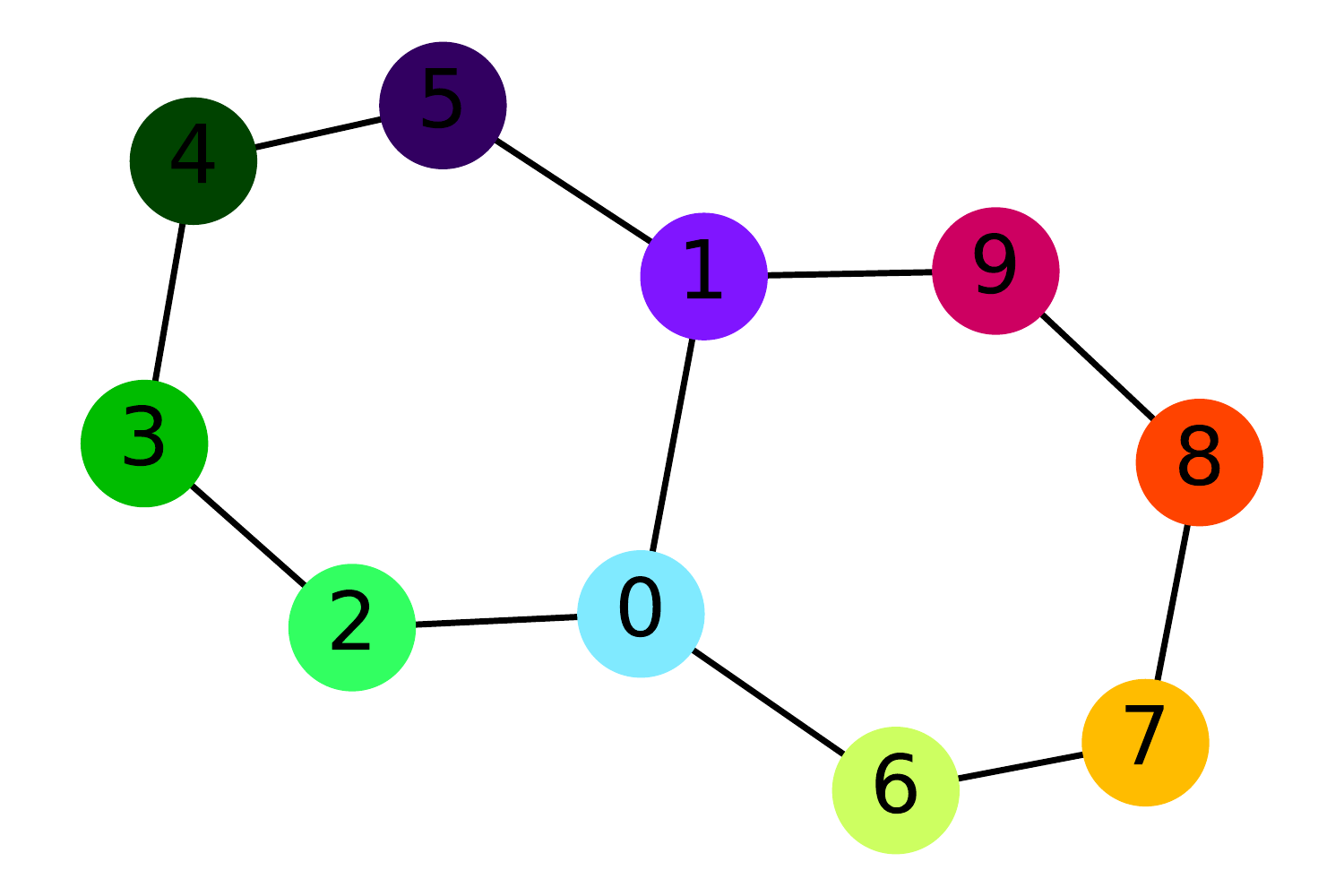}
        }
    \subfigure[EDEN $H$]{
        \label{Fig.ableden.2}
        \includegraphics[width=0.305\linewidth]{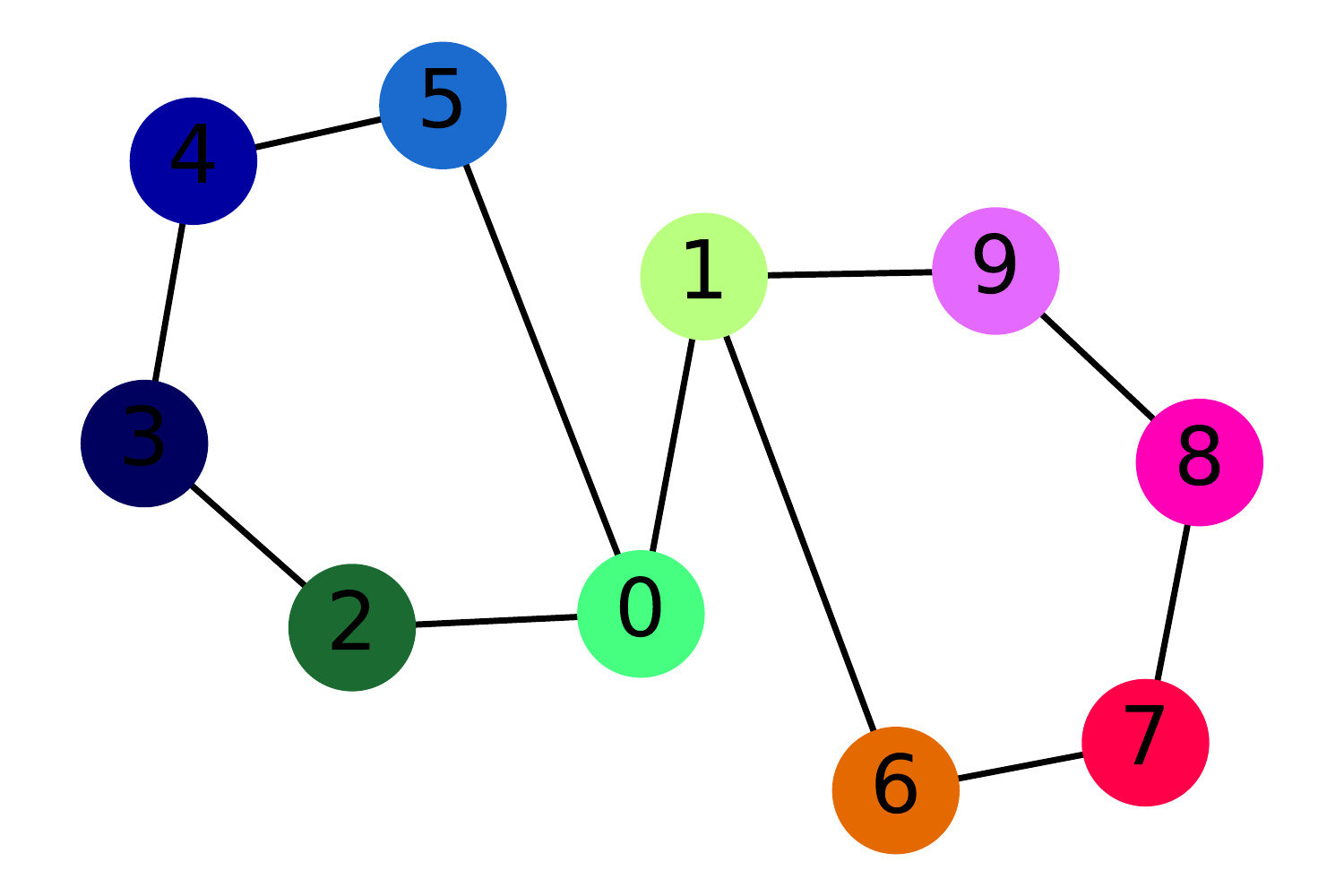}
        }
    \subfigure[EDEN $G'$]{
        \label{Fig.ableden.3}
        \includegraphics[width=0.305\linewidth]{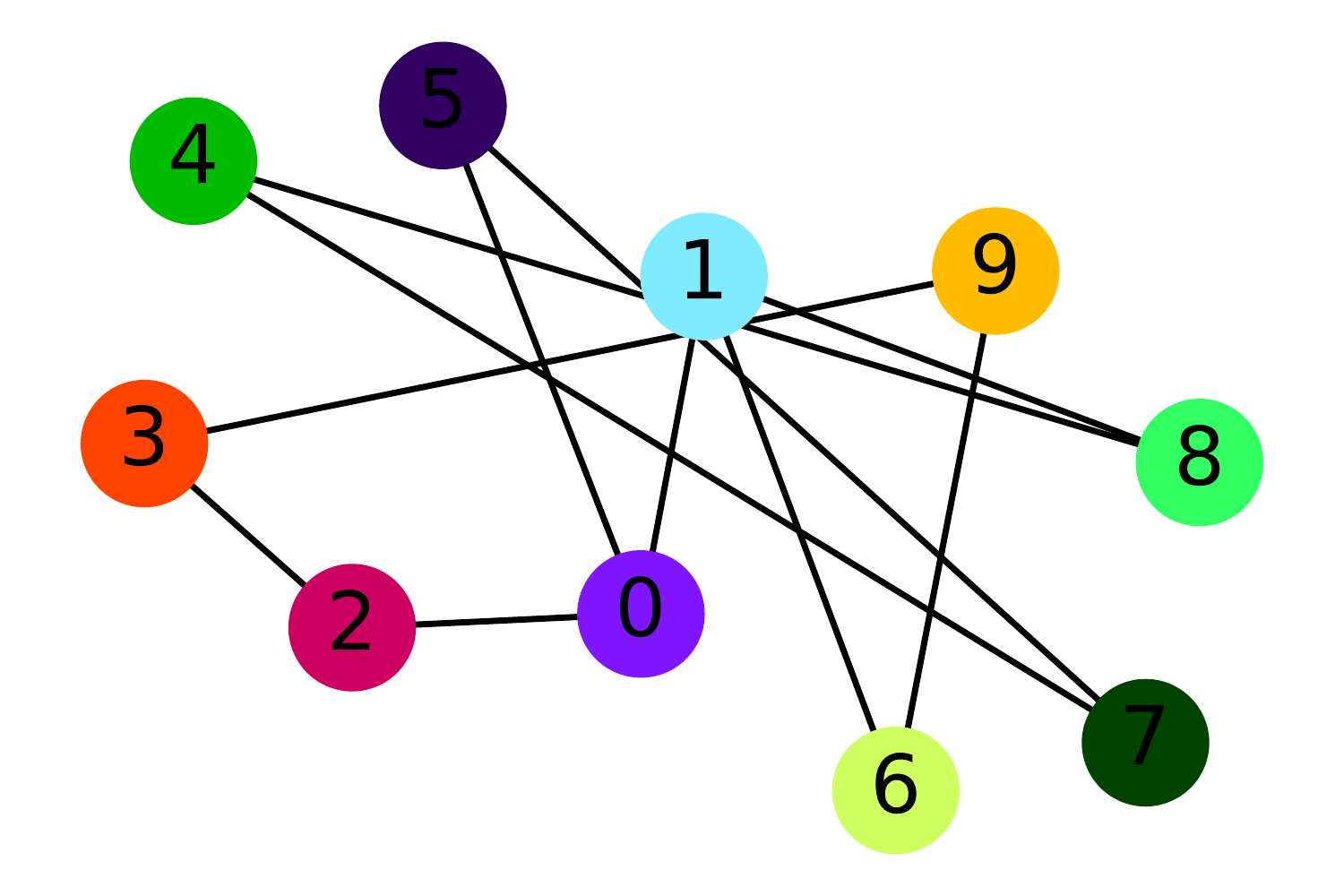}
        }
    \caption{Ablation study on 1-WL equivalent graph pair.}
    \label{Fig.ablation2}
\end{figure}



\subsubsection{Distance-based vs. Degree-based eigen-decomposition.}
Since we introduce the length of the shortest-path between nodes as the global information to enhance expressiveness and perform PCA to guarantee the equivariance of EDEN, we call it distance-based eigen-decomposition.
In this section, we analyze and compare EDEN in depth with another close method, i.e., Laplacian PE, which relies on the eigen-decomposition on the Graph Laplacian matrix $\boldsymbol{L}$.
We call it degree-based eigen-decomposition since the $\boldsymbol{L}$ only provides the degree information of nodes, which is believed as a part of 1-WL's expressiveness \cite{breaking}.
As described in Table~\ref{tab:experiment}, EDEN outperforms Laplacian PE by a large margin. We conduct an in-depth analysis of Laplacian PE and find that it selects the smallest eigenvalues for feature encoding after eigen-decomposition, which leads to the problem of the over-smooth encoding of adjacent nodes \cite{benchmarking}. While the alternative way of selecting the largest eigenvalues leads to the problem that few nodes with the highest degrees will be mapped far away from the others (See Appendix \ref{app:lap}), we believe this situation makes Laplacian PE less meaningful.


To support our insights, we further provide a visual comparison on real-world datasets (Fig. \ref{Fig.ablation}) and 1-WL equivalent graph pairs (Fig. \ref{Fig.ablation2}). 
In the real-world case, we first compute the 3-dimensional EDEN and Laplacian PE, then use them as the Red-Green-Blue values to color each node for an intuitive visualization.
Compared with EDEN's coloring effect in Fig. \ref{Fig.sub.1}, the Laplacian PE in Fig. \ref{Fig.ablation.1} colors most of the nodes with similar colors (purple or gray), which indicates the over-smoothing problem as mentioned before.
Next, we draw 2-dimensional EDEN and Laplacian PE scatters for all nodes. Nodes with the same labels are marked with the same color. 
Although the positional encodings are computed without supervision (labels), EDEN in Fig. \ref{Fig.sub.2} maps nodes uniformly into the 2D Euclidean space and makes nodes linearly separable according to their labels. The Laplacian PE in Fig. \ref{Fig.ablation.2} cannot achieve either.

Lastly, we choose a 1-WL equivalent graph pair, \textit{i.e.,} Decalin ($G$) and Bicyclopentyl ($H$) graph to test graph isomorphism. 
Due to Lemma. \ref{ED no unique}, the failure of graph isomorphism test lies in the inability to map isomorphic graphs to the same results, rather than mapping non-isomorphic graphs to different results.
So we permute nodes in graph $G$ randomly to construct an isomorphic graph $G'$ and then color these three graphs with 3D Laplacian PE and EDEN. As shown in Fig. \ref{Fig.ablation2}, the Laplacian PE fails to maintain deterministic outputs between $G$ and $G'$ (\emph{e.g.,} there is a new color yellow in $G'$'s node of No.5). At the same time, the outputs of EDEN are still deterministic. There is a one-to-one correspondence between the nodes of EDEN in an isomorphic graph pair.
More test examples are provided in Appendix \ref{app:lap}. We find out that the Laplacian PE can neither go beyond 1-WL expressiveness nor be as meaningful as EDEN.


\begin{table}
 \centering
 \renewcommand\arraystretch{1.1}
 \resizebox{0.95\linewidth}{!}{
\begin{tabular}{cccccc}
\toprule
Datasets   & Baseline        &  \begin{tabular}[c]{@{}c@{}}Direct \\ Distance\end{tabular}  & \begin{tabular}[c]{@{}c@{}}MinMax\\ Distance\end{tabular} & \begin{tabular}[c]{@{}c@{}}R-MinMax\\ Distance\end{tabular} & EDEN                     \\ \hline
Cora       & 0.859 & 0.862  & 0.868      & 0.873                                           & \textbf{0.877}  \\
CiteSeer & 0.720 & 0.720 & 0.724 & \textbf{0.740} & 0.739 \\
ENZYMES  & 0.644 & 0.856 & 0.856 & 0.862 & \textbf{0.871}\\
LPROTEINS & 0.624 & 0.795 & 0.871 & 0.880 & \textbf{0.888} \\ 
MUTAG       & 0.834 & 0.851 & 0.871 & 0.916 & \textbf{0.927} \\ 
GPROTEINS  & 0.716 & 0.721 & 0.752 & 0.752& \textbf{0.767}\\ \bottomrule
\end{tabular}
}
\caption{Ablation study on real-world datasets. }
\label{tab:ablation}
\end{table}

\subsubsection{Cosine phase propagation}
To investigate the effectiveness of cosine phase propagation, we use three distance encoding schemes, namely, \textbf{direct distance ($S_1$)},  \textbf{Min-Max normalized distance ($S_2$)}, and \textbf{reversed Min-Max normalized distance ($S_3$)}, to replace the cosine phase propagation and compare their performance.

 \textbf{$S_1$: Direct distance.}
 First, we perform PCA directly on the distance matrix $\boldsymbol{D}$ as an extra feature.
 \begin{equation}
	\hat{\boldsymbol{D}}_{i,j}=\left\{
	\begin{aligned}
	    &  \boldsymbol{D}_{i,j}  ,&\boldsymbol{D}_{i,j}\neq \infty\\
	    & -1  ,&\boldsymbol{D}_{i,j} = \infty.
	\end{aligned}
	\right.
	\label{eq:directdis}
\end{equation}
Here we set the distance between a pair of unreachable nodes $v_i$ and $v_j$ to -1.
\begin{equation}
     \boldsymbol{F} = \text{PCA}(\hat{\boldsymbol{D}}).
 \end{equation}
 In this case, the norm of $\boldsymbol{F}$ could be very large. 
 
\textbf{$S_2$: Min-Max normalized distance.}
Next, we perform the min-max normalization on $D$ before we employ PCA,
  \begin{equation}
	\hat{\boldsymbol{D}}_{i,j}=\left\{
	\begin{aligned}
	    & \frac{\boldsymbol{D}_{i,j}- \min(\boldsymbol{D}) }{\max(\boldsymbol{D})-\min(\boldsymbol{D})}   ,&\boldsymbol{D}_{i,j}\neq \infty\\
	    & -1  ,&\boldsymbol{D}_{i,j} = \infty.
	\end{aligned}
	\right.
	\label{eq:minmax}
\end{equation}
In this case, the values of all elements in $D$ range in [0,1], and larger values indicate longer distances.
 
\textbf{$S_3$: Reversed Min-Max normalized distance.}
Last, we transform $\hat{\boldsymbol{D}}$ in the following way,
   \begin{equation}
	\hat{\boldsymbol{D}}_{i,j}=\left\{
	\begin{aligned}
	    & 1- \frac{\boldsymbol{D}_{i,j}- \min(\boldsymbol{D}) }{\max(\boldsymbol{D})-\min(\boldsymbol{D})}   ,&\boldsymbol{D}_{i,j}\neq \infty\\
	    & -1 ,&\boldsymbol{D}_{i,j} = \infty.
	\end{aligned}
	\right.
	\label{eq:rminmax}
	\end{equation}
 This is similar to the second case, but the difference is that a larger value of the element in $\hat{D}$ represents a shorter distance between two nodes.
 
We choose four MPNNs (GCN, SAGE, GAT, and GIN) and employ the above three distance encoding schemes on real-world datasets. For a fair comparison, the hyper-parameters are the same as EDEN. Results are shown in the Tab. \ref{tab:ablation} (``LPROTEINS'' and ``GPROTEINS'' denote link-prediction and graph classification tasks on the PROTEINS dataset, respectively). Note that the ``Baseline'' column in Tab. \ref{tab:ablation} describes the average performance of MPNNs, and detailed experimental results are provided in Appendix \ref{app:phase}.
 
It can be observed from Tab. \ref{tab:ablation} that EDEN achieves the best performance except on the CiteSeer dataset. We analyze the possible reasons: (1) In $S_1$, the values of some elements in unnormalized matrix $\boldsymbol{D}$ are large, resulting in large norms of learned positional encoding vectors, which leads to suboptimal results. (2) In $S_1$ and $S_2$, the value of $\boldsymbol{\hat{D}}_{ij}$ is proportional to the distance (hops) between node $v_i$ and $v_j$. But it cannot express the distance between two unreachable nodes (\emph{i.e.,} the distance is defined as infinite), so we need an opposite way of representing distance, \emph{i.e.,} use a smaller value to represent a longer distance, and use -1 to represent the distance between unreachable nodes. This means that $S_1$ and $S_2$ are less interpretable for distance encoding. (3) $S_3$ is close to our proposed cosine phase propagation, and it also has the property of normalization and interpretability that makes it superior to $S_1$ and $S_2$. But $S_3$ is still a linear transformation. Distance encoding using trigonometric functions is considered to have better properties compared to linear mapping \cite{transformer}.

\subsection{Case study}
We choose the Cospectral graph ($G$) and a 4-regular graph ($H$) as a 2-WL equivalent pair for graph isomorphism test. As shown in Fig. \ref{Fig.2wltesteden}, EDEN distinguishes the non-isomorphism pair and keeps the isomorphism pair $G$ and $G'$ mapped to the same output. We provide a detailed analysis and more permutation examples of $G'$ in Appendix \ref{app:regular}.
Fig. \ref{Fig.2deden} illustrates the result of employing EDEN on the real-world dataset of Cora \cite{cora}. We find that the 2-dimensional EDEN is capable of providing semantic information. The contours of different clusters according to node labels can be observed via EDEN.
\begin{figure}
   \centering  
    \subfigure[$G$]{
        \label{Fig.2wltesteden.1}
        \includegraphics[width=0.305\linewidth]{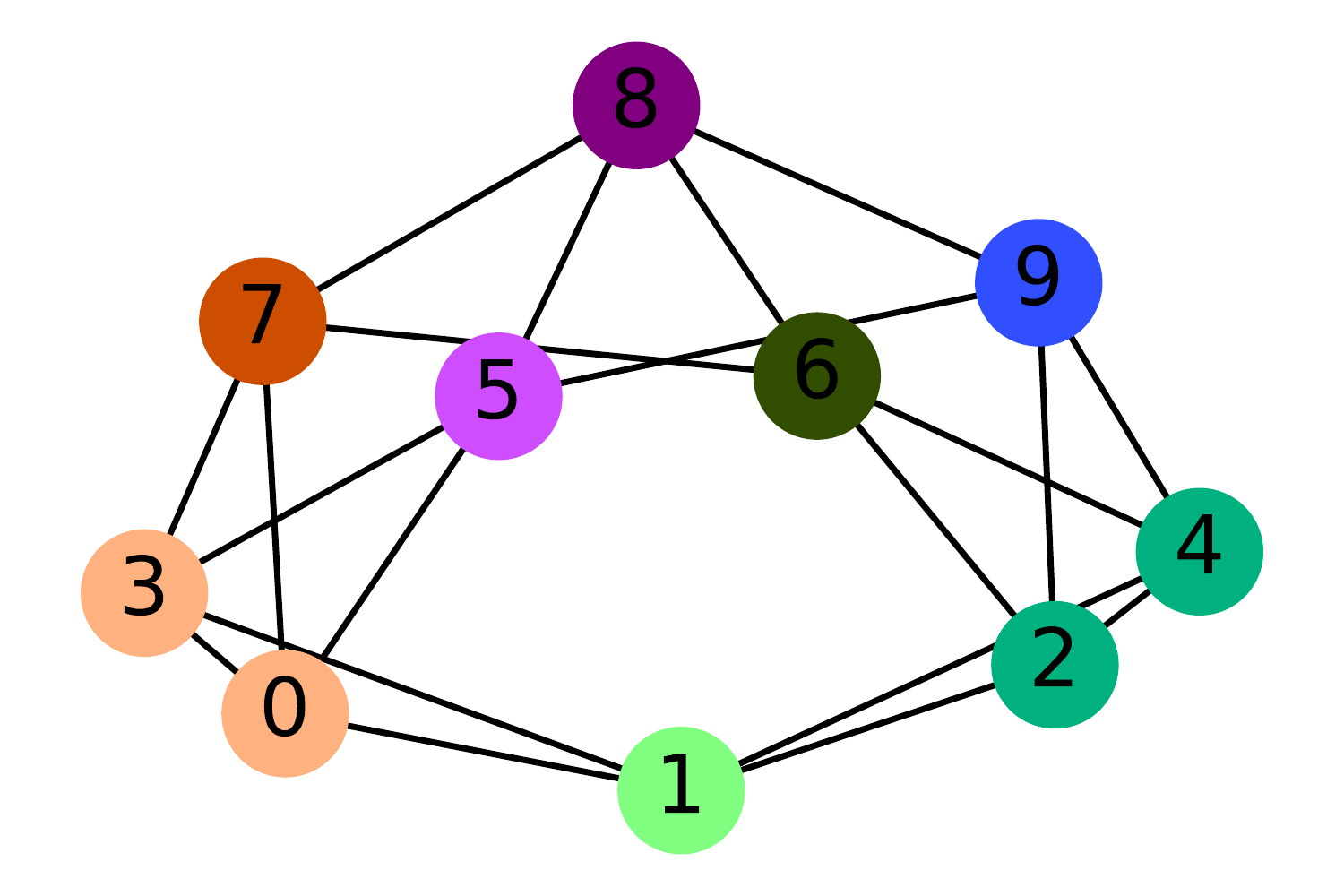}
        }
    \subfigure[$H$]{
        \label{Fig.2wltesteden.2}
        \includegraphics[width=0.305\linewidth]{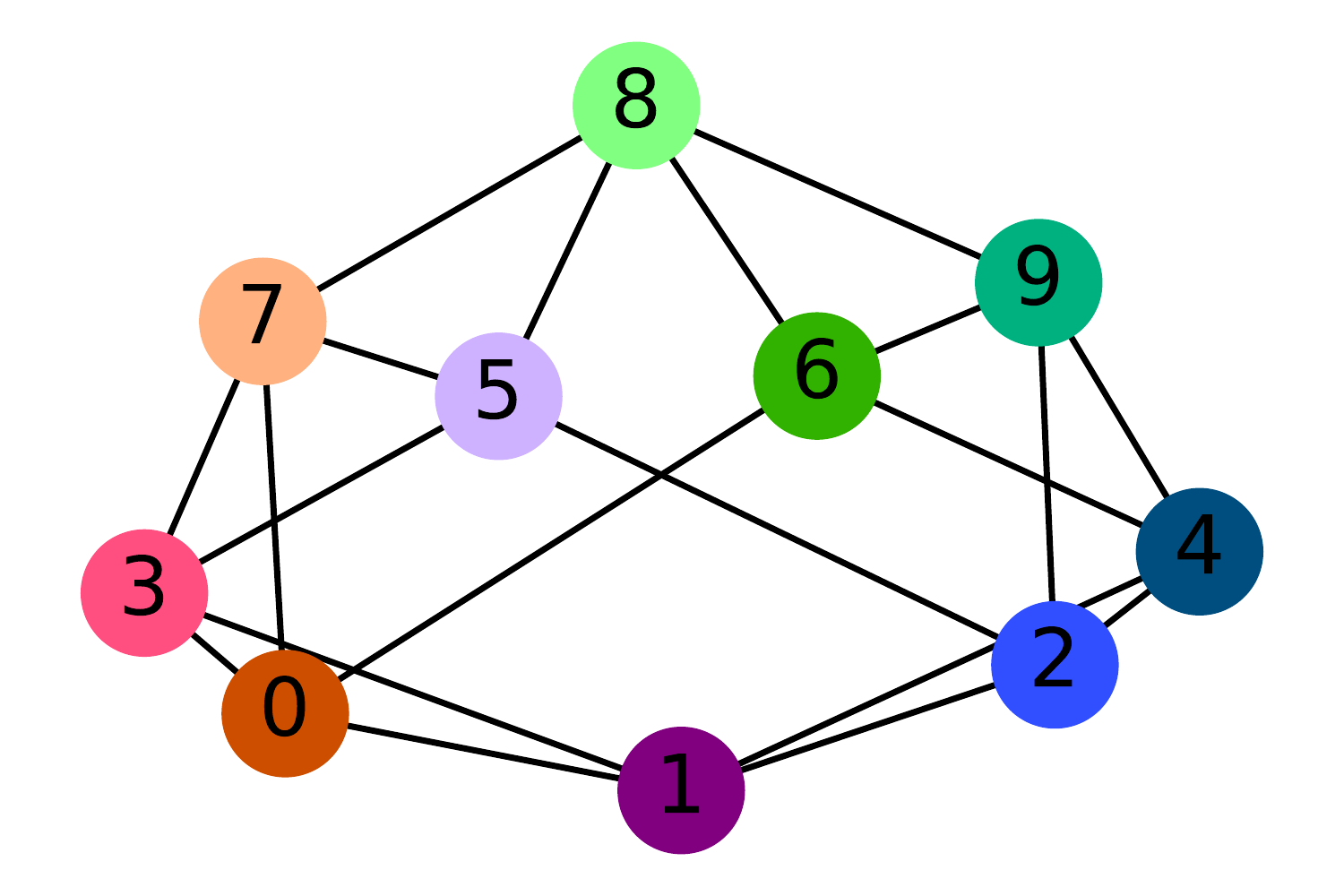}
        }
    \subfigure[$G'$]{
        \label{Fig.2wltesteden.3}
        \includegraphics[width=0.305\linewidth]{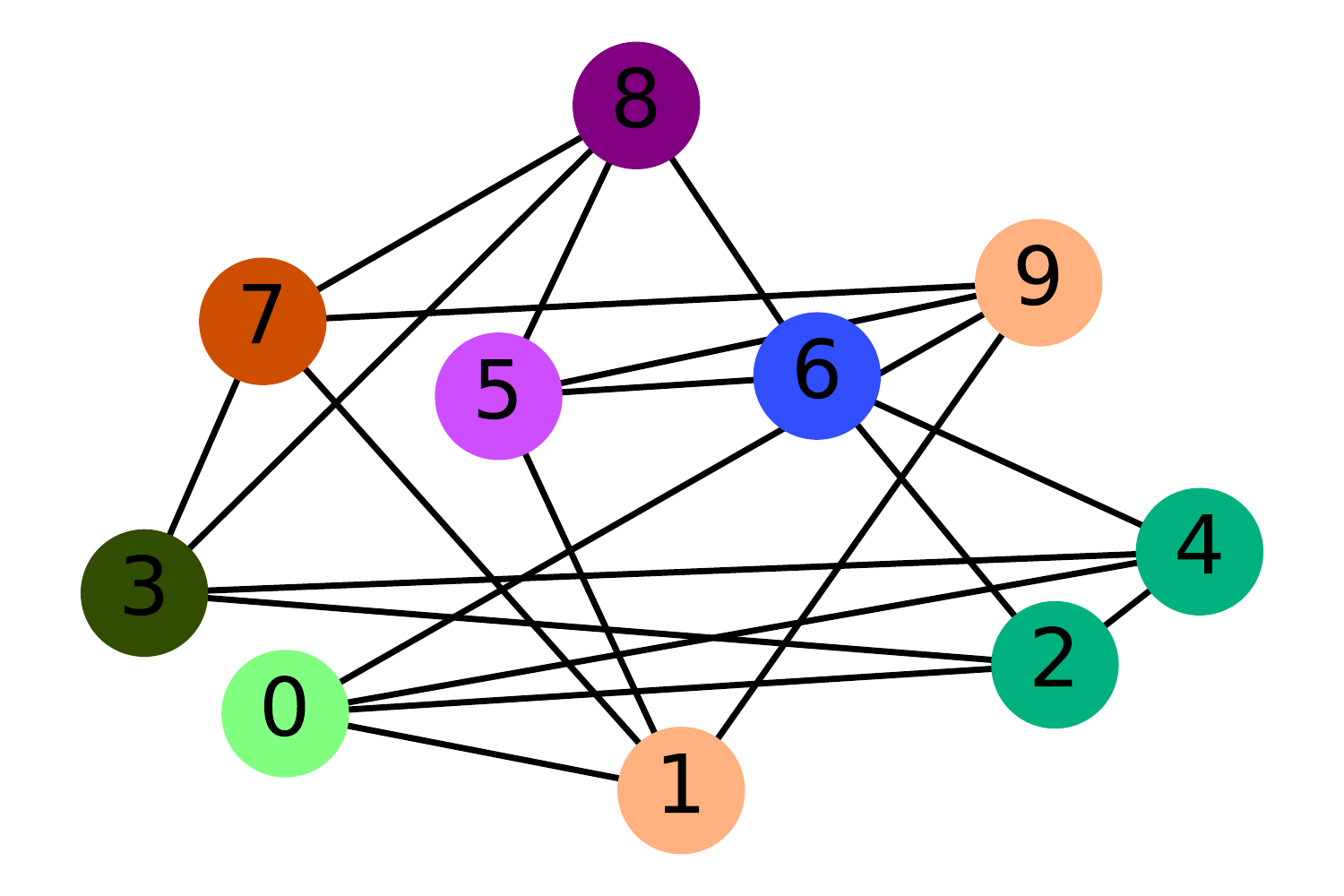}
        }
    \caption{EDEN on 2-WL equivalent graph pair.}
    \label{Fig.2wltesteden}
    \centering  
\end{figure}
\begin{figure}
   \centering  
    \includegraphics[width=0.75\linewidth]{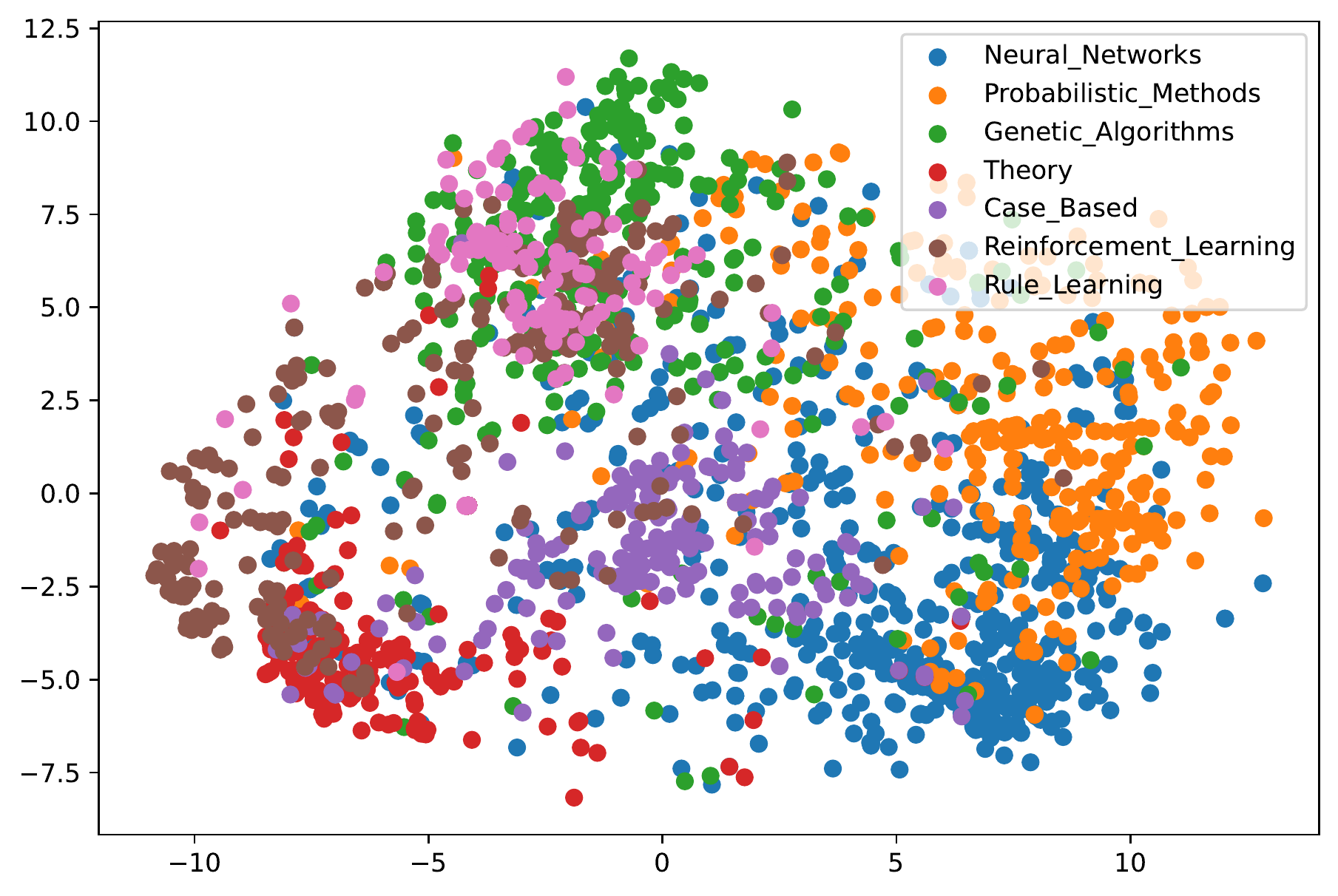}
    \caption{EDEN on Cora dataset \cite{cora}.}
    \label{Fig.2deden}
    \centering  
\end{figure}

\section{Related Works}
\label{sec:related}
\textbf{Non-equivariant GNN}
Representing non-trivial position information for Non-Euclidean structured data is still challenging. 
PGNN \cite{PGNN} tackles this problem via randomly selecting nodes as anchors and aggregating anchors' information to the target node,
but a graph with $n$ nodes requires at least $O(\log^2 n)$ anchors to guarantee the expressiveness \cite{bourgain1985lipschitz}. 
DE-GNN \cite{DEGNN} proves its expressive power in the circumstance where parts of nodes' representations in a graph are given, and distance-based aggregation is applied to learn the remaining nodes' representations. 
Nevertheless, these methods rely on selecting ``anchors'', which means they cannot be transferred to downstream tasks and do not have the property of permutation-equivariance.


\textbf{Higher-order GNN}
It is natural to design advanced GNNs by enhancing expressiveness. F-GNN \cite{FGNN} introduces the k-Folklore graph isomorphism test to improve the expressive power. Recent work \cite{azizian2020expressive} proves that F-GNN has the same expressiveness as KGNN \cite{KGNN}.
GNNML3 \cite{breaking} ensures its expressiveness with the Matrix Language theory. 
K-MPNNs \cite{DBLP:journals/corr/abs-2204-04661} analyses variants of Higher-order GNNs \cite{123gnn} with Tensor Language.
The above methods have high computational complexity and are time-consuming during training.

\textbf{Non-MPNNs}
Identity-aware GNNs (ID-GNNs) \cite{IDGNN} are trained with ego graph-based aggregation functions.
One of its two variants, ID-GNN-full, can work as a plug-in to the aggregation functions of existing MPNNs. 
The other of its two variants, ID-GNN-fast, can be 
viewed as a kind of distance encoding.
But the disadvantages are that training ID-GNN-full is expensive because it requires ego graphs on all nodes, while ID-GNN-fast tends to have low prediction accuracy. 
Similar to EDEN, GraphSNN \cite{GraphSNN} builds a new perspective on the graph isomorphism test. It reconstructs the aggregation function with overlapping subgraphs, resulting in a lower computational complexity than other higher-order methods. It works well on graph-level tasks but is not good at local-view tasks such as node classification. Local Relation Pooling (LPR) \cite{Subcount} is a universal approximator for permutation-invariant functions based on local subgraph counts. Natural GNN \cite{Natural} passes messages with kernels that depend on the local graph structure. 

\textbf{Plug-in encodings} 
It's practical to employ matrix decomposition to obtain Positional Encodings (PE), such as the Graph Transformer
\cite{GraphTransformers} .
The Laplacian PE~\cite{benchmarking} is a representative eigen-decomposition-based method. But Laplacian PE cannot help the model achieve expressive power beyond the 1-WL and has less semantic information in real-world datasets. 
More details can be found in Appendix \ref{app:lap}.
It is also worth mentioning that all eigenvalue decomposition-based methods have a time complexity of at least $O(n^3)$. 


\section{Conclusion}
This paper presents a new distance encoding method called EDEN.
We theoretically prove that EDEN has the property of permutation-equivariance, which is unavailable for existing plug-in encoding methods for GNN.
We demonstrate EDEN's powerful expressiveness beyond the 1-WL test and its superiority in dealing with real-world graphs.
EDEN can be employed as a plug-in extension to existing MPNNs. The experimental results show that EDEN can improve the performance of conventional MPNNs to the state-of-the-art level in node-, edge-, and graph-level tasks.
Meanwhile, the time complexity of obtaining the distance matrix is still high, and employing PCA is less efficient. 
In future work, we will focus on how to reduce the computational complexity of EDEN. 

\bibliography{aaai23}
\clearpage
\appendix

\section{The Algorithm}
\label{app:code}

The pseudo-code of EDEN is summarized as  Algorithm \ref{algorithm:eden}. 
We implement our code in Python.  
The ``\emph{zeros}'' and ``\emph{nanmax}'' are functions in NumPy\footnote{https://github.com/numpy/numpy}, the former assigns elements to zero, and the latter obtains the maximum number while ignoring NaN (Not a Number) in Python.
The \emph{Dijkstra algorithm} is implemented by NetworkX\footnote{https://github.com/networkx/networkx}, and ``PCA'' is called from the Scikit-learn\footnote{https://github.com/scikit-learn/scikit-learn} package. We provide core python codes.

\begin{python}[t]
from sklearn.decomposition import PCA
import networkx as nx
import numpy as np

def graph_to_eden(G, k=3):
    Distance = dict(
    nx.all_pairs_shortest_path_length(G)
    )
    num_node = len(G)
    D = np.zeros([num_node, num_node])
    for i in range(num_node):
        for j in range(num_node):
            D[i,j] = Distance[i][j]
    PCA1 = PCA(n_components=k)
    COS_D = np.cos(np.pi * D /...
                   np.nanmax(D, axis=1,
                   keepdims=True))
    COS_D[np.isnan(cos_dis)]=-1.5
    EDEN = PCA1.fit_transform(COS_D)
    return EDEN
\end{python}
\begin{algorithm}
	\caption{Equivariant distance encoding
	}
	\label{algorithm:eden}
    \KwIn{The adjacency matrix $\boldsymbol{A}$, dimension $k$, the number of nodes $N$}
    \KwOut{Eden's matrix $\boldsymbol{F_{\hat{\boldsymbol{D}}}}$}
    $\boldsymbol{D}$ = \emph{Dijkstra}($\boldsymbol{A}$); \tcp{Employ \emph{Dijkstra algorithm} to obtain the distance matrix}
    $\boldsymbol{d}$ = \emph{zeros}($[N, 1]$);\tcp{Initialize the diameter vector $\boldsymbol{d}$}
    $\hat{\boldsymbol{D}}$ = \emph{zeros}($[N, N]$);\tcp{Initialize the phase propagation matrix $\hat{\boldsymbol{D}}$}
    \For{$i=1,2,\cdots,N$\tcp{Perform phase propagation in Eq.\eqref{eq:phase}}}
    {$d_i = \text{\emph{nanmax}}(\boldsymbol{D_{i,:})} $\; 
    
    \For{$j=1,2,\cdots,N$ }{
    \eIf{$\boldsymbol{D}_{i,j}\neq \text{NaN}$}{$\hat{\boldsymbol{D}}_{i,j}= \cos{(\frac{\pi \times \boldsymbol{D}_{i,j}}{d_i})}$\; }{$\hat{\boldsymbol{D}}_{i,j}=-1.5$}
	}
	}
	$\boldsymbol{F_{\hat{\boldsymbol{D}}}}=$PCA($\hat{\boldsymbol{D}}$, $k$); \tcp{Perform $k$ dimensional PCA in Eq.\eqref{eq:pca}}
	\Return $\boldsymbol{F_{\hat{\boldsymbol{D}}}}$
\end{algorithm}

\section{EDEN on Different Graph Pairs}
\label{app:regular}
 We plot the result of a 4-regular graph isomorphism test in Figure \ref{Fig.eden}. In this part, we describe the computational details and show more examples with different equivalent levels, where the $k$-WL equivalent means non-less than the $k+1$-WL test can distinguish them \cite{breaking}. 
\subsection{Computational details}
Considering two graphs $G$ and $H$ with adjacency matrices $A_G$ and $A_H$, respectively. We apply Algorithm \ref{algorithm:eden} to compute the 3-dimensional EDENs for these graphs:
\begin{equation}
\label{eq:appeden1}
    \begin{aligned}
    \boldsymbol{F_{\hat{\boldsymbol{D}}}}_G = \text{EDEN}(\boldsymbol{A}_G,3)&, \\
    \boldsymbol{F_{\hat{\boldsymbol{D}}}}_H = \text{EDEN}(\boldsymbol{A}_H,3)&,
    \end{aligned}
\end{equation}
where $\text{EDEN}(\cdot,3)$ denotes getting the 3-dimensional EDENs in Algorithm \ref{algorithm:eden}.
Then the computed EDENs are normalized within $[0,1]$ :
\begin{equation}
\label{eq:appeden2}
    \begin{aligned}
    \boldsymbol{F_{\hat{\boldsymbol{D}}}}'_G = \frac{\boldsymbol{F_{\hat{\boldsymbol{D}}}}_G - \min(\boldsymbol{F_{\hat{\boldsymbol{D}}}}_G)}{\max(\boldsymbol{F_{\hat{\boldsymbol{D}}}}_G) - \min(\boldsymbol{F_{\hat{\boldsymbol{D}}}}_G)}, &\\
    \boldsymbol{F_{\hat{\boldsymbol{D}}}}'_H = \frac{\boldsymbol{F_{\hat{\boldsymbol{D}}}}_H - \min(\boldsymbol{F_{\hat{\boldsymbol{D}}}}_H)}{\max(\boldsymbol{F_{\hat{\boldsymbol{D}}}}_H) - \min(\boldsymbol{F_{\hat{\boldsymbol{D}}}}_H)}&,
    \end{aligned}
\end{equation}
 where min($\cdot$) and max($\cdot$) denote the minimum and maximum values in each row, respectively.
 We obtain the graph $G'$ by a random permutation $\sigma$ on graph $G$,
 \begin{equation}
     G' = \sigma \star G.
 \end{equation}
 Similarly, the normalized EDEN $\boldsymbol{F_{\hat{\boldsymbol{D}}}}'_{G'}$ can be computed by Eq. (\ref{eq:appeden1}) and Eq. (\ref{eq:appeden2}).
 We plot $G, H$, and $G'$ with the same layout, which means nodes with the same ID will have their specific positions in different graphs. 
 To facilitate visualization, we treat the 3-dimensional EDEN as the Red-Green-Blue (RGB) brightness to color each node.
 We expect nodes in isomorphic graph pairs (\textit{i.e.,} $G$ and $G'$) to have same colors but with different IDs, and the neighbor relationship of nodes with different colors remain unchanged.
 For example, if a green node in graph $G$ is connected with an orange node and an aqua blue node, the green node in graph $G'$ should also be connected with an orange node and an aqua blue node.
There is no such correspondence between the colors of nodes in non-isomorphic graph pairs, (\emph{e.g.,} $G$ and $H$).
 We provide more visual examples of graph pairs that can show the expressive power of EDEN.
 
\subsection{1-WL equivalent graphs}

There are three graphs in Fig. \ref{Fig.wl1}, Decalin ($G$), Bicyclopentyl  ($H$), and randomly permuted Decalin ($G'$). 
$G$ and $H$ are 1-WL equivalent graph pairs and the adjacency matrices $\boldsymbol{A}_G$ and $\boldsymbol{A}_H$ are written as: 

\scriptsize{$$
A_{G}=\left(\begin{array}{llllllllll}
0 & 1 & 1 & 0 & 0 & 0 & 1 & 0 & 0 & 0 \\
1 & 0 & 0 & 0 & 0 & 1 & 0 & 0 & 0 & 1 \\
1 & 0 & 0 & 1 & 0 & 0 & 0 & 0 & 0 & 0 \\
0 & 0 & 1 & 0 & 1 & 0 & 0 & 0 & 0 & 0 \\
0 & 0 & 0 & 1 & 0 & 1 & 0 & 0 & 0 & 0 \\
0 & 1 & 0 & 0 & 1 & 0 & 0 & 0 & 0 & 0 \\
1 & 0 & 0 & 0 & 0 & 0 & 0 & 1 & 0 & 0 \\
0 & 0 & 0 & 0 & 0 & 0 & 1 & 0 & 1 & 0 \\
0 & 0 & 0 & 0 & 0 & 0 & 0 & 1 & 0 & 1 \\
0 & 1 & 0 & 0 & 0 & 0 & 0 & 0 & 1 & 0
\end{array}\right), $$
$$
A_{H}=\left(\begin{array}{llllllllll}
0 & 1 & 1 & 0 & 0 & 1 & 0 & 0 & 0 & 0 \\
1 & 0 & 0 & 0 & 0 & 0 & 1 & 0 & 0 & 1 \\
1 & 0 & 0 & 1 & 0 & 0 & 0 & 0 & 0 & 0 \\
0 & 0 & 1 & 0 & 1 & 0 & 0 & 0 & 0 & 0 \\
0 & 0 & 0 & 1 & 0 & 1 & 0 & 0 & 0 & 0 \\
1 & 0 & 0 & 0 & 1 & 0 & 0 & 0 & 0 & 0 \\
0 & 1 & 0 & 0 & 0 & 0 & 0 & 1 & 0 & 0 \\
0 & 0 & 0 & 0 & 0 & 0 & 1 & 0 & 1 & 0 \\
0 & 0 & 0 & 0 & 0 & 0 & 0 & 1 & 0 & 1 \\
0 & 1 & 0 & 0 & 0 & 0 & 0 & 0 & 1 & 0
\end{array}\right).
$$}
\begin{figure*}
    \centering  
    \subfigure[$G$]{
        \label{Fig.wl1.1}
        \includegraphics[width=0.3\textwidth]{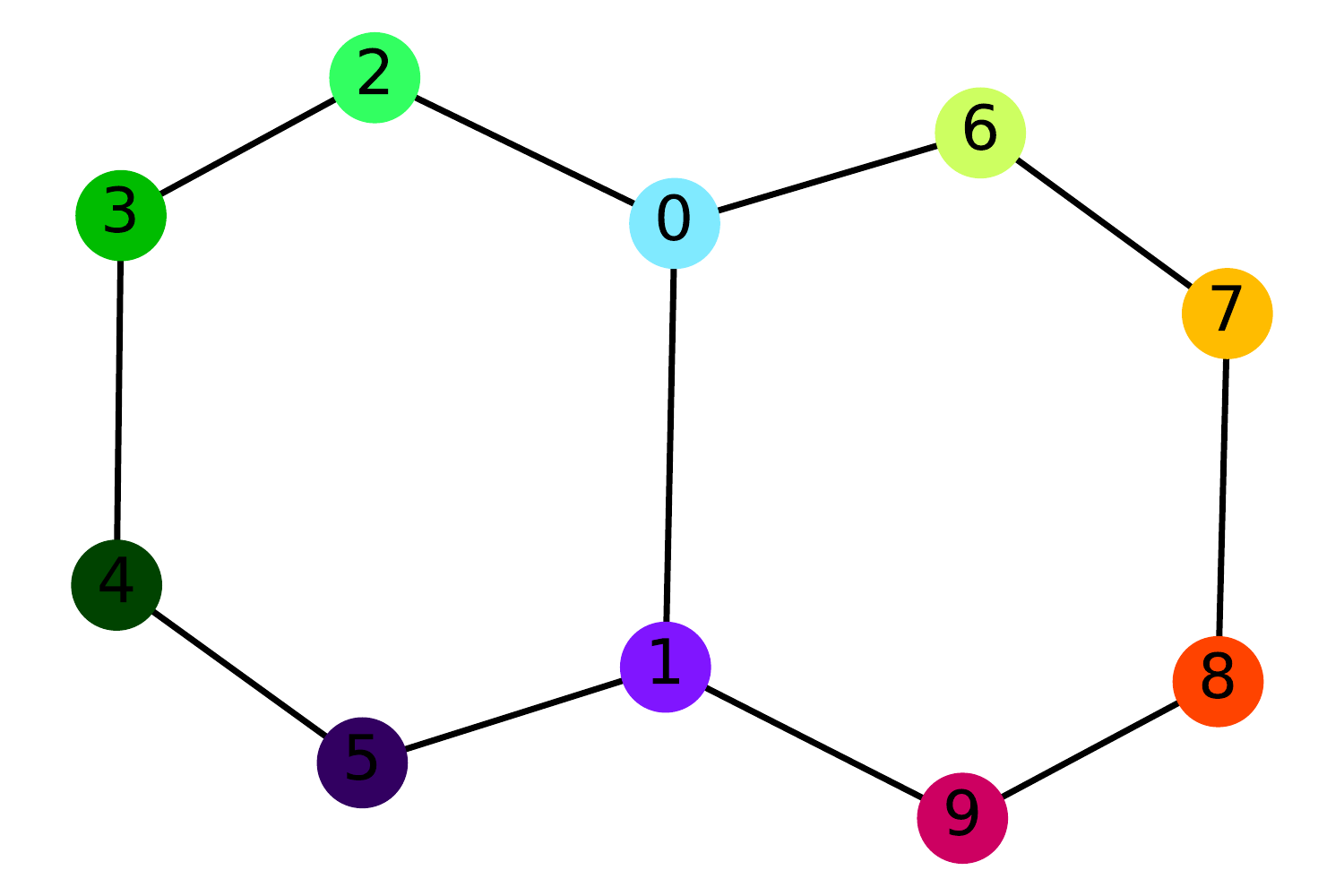}
        }
    \subfigure[$H$]{
        \label{Fig.wl1.2}
        \includegraphics[width=0.3\textwidth]{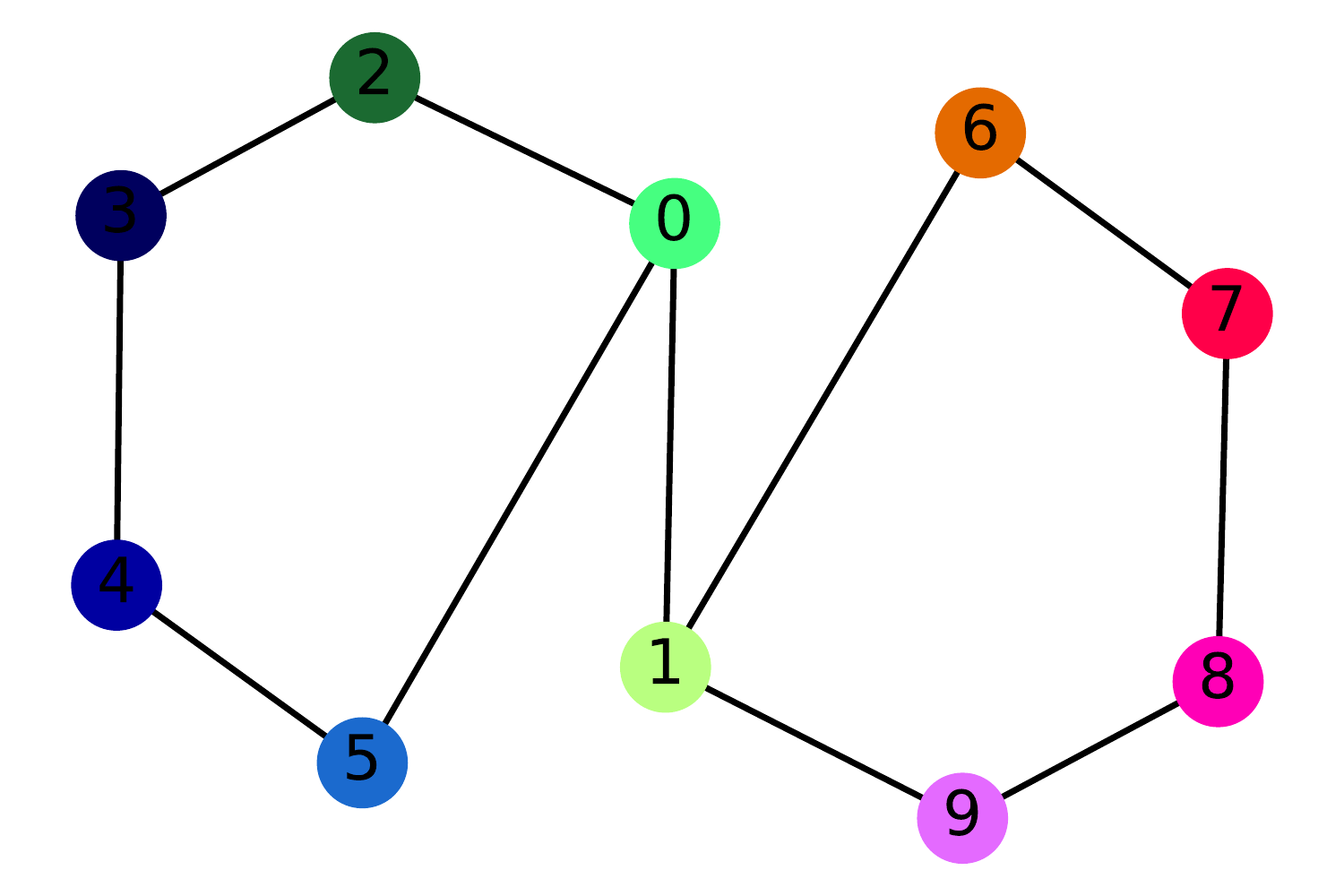}
        }
    \subfigure[$G'$]{
        \label{Fig.wl1.3}
        \includegraphics[width=0.3\textwidth]{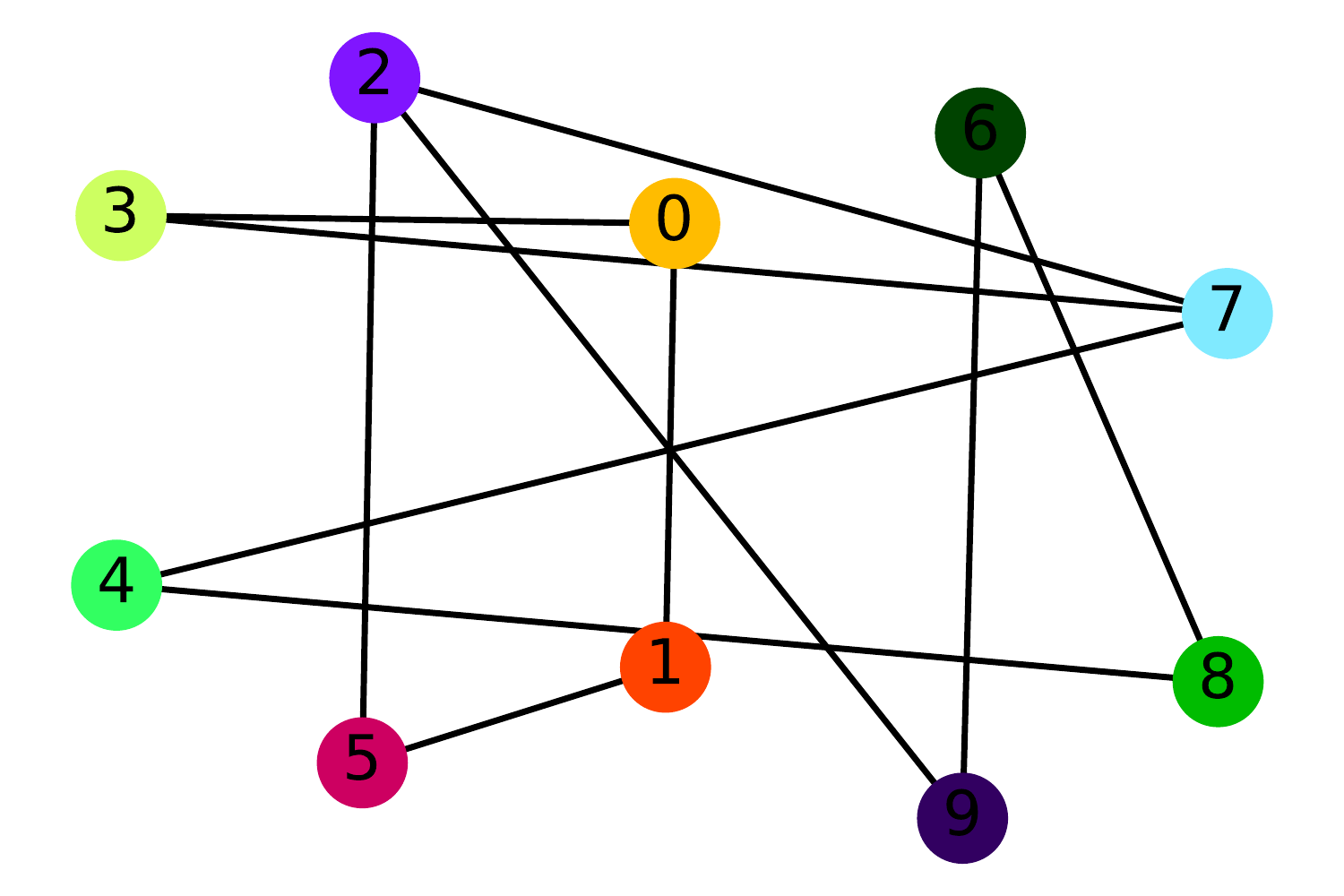}
        }
    \caption{Decalin and Bicyclopentyl graphs are 1-WL equivalent \cite{breaking}.}
    \label{Fig.wl1}
\end{figure*}
\subsection{2-WL equivalent graphs}
\begin{figure*}
    \centering  
    \subfigure[$G$]{
        \label{Fig.wl2.1}
        \includegraphics[width=0.3\textwidth]{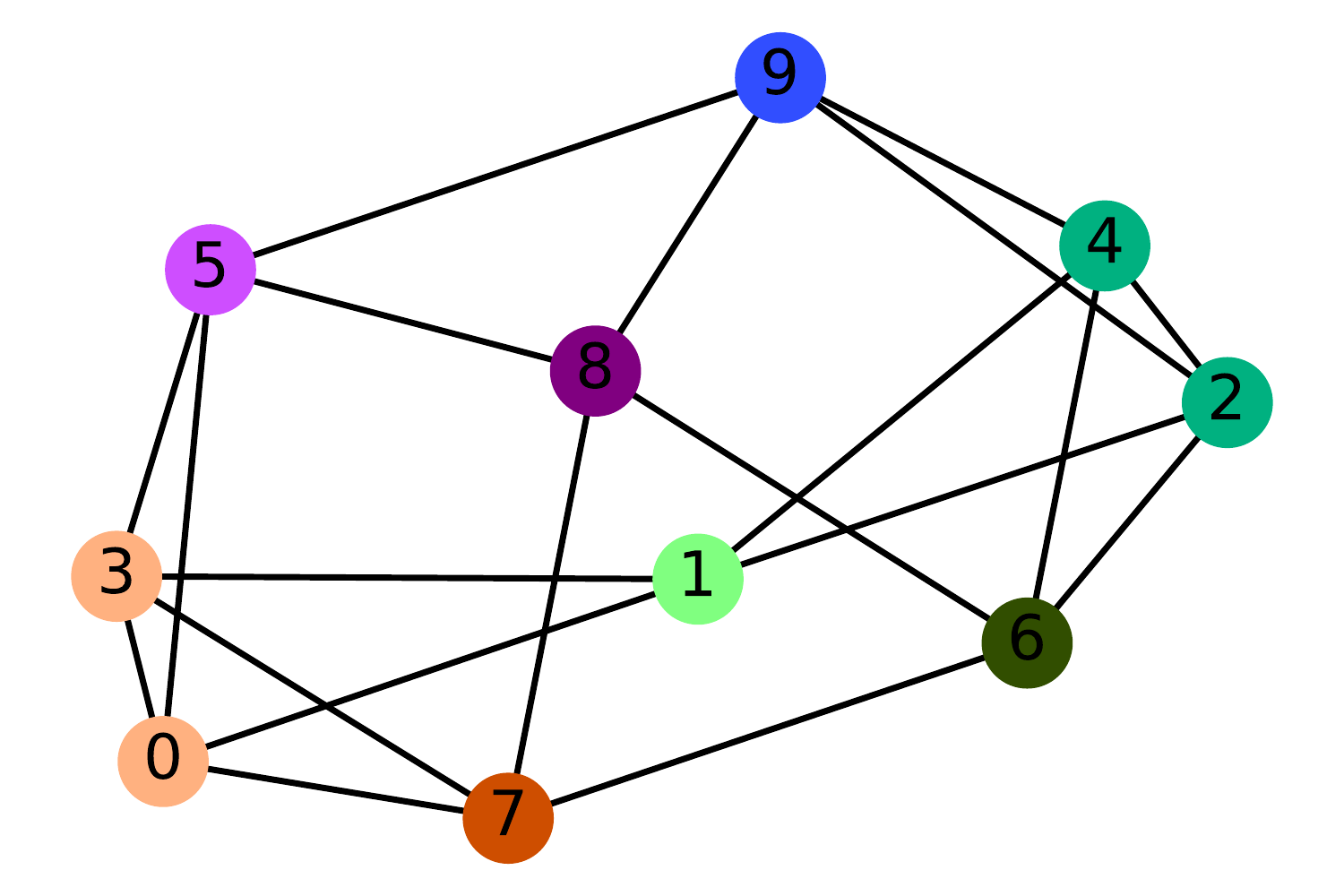}
        }
    \subfigure[$H$]{
        \label{Fig.wl2.2}
        \includegraphics[width=0.3\textwidth]{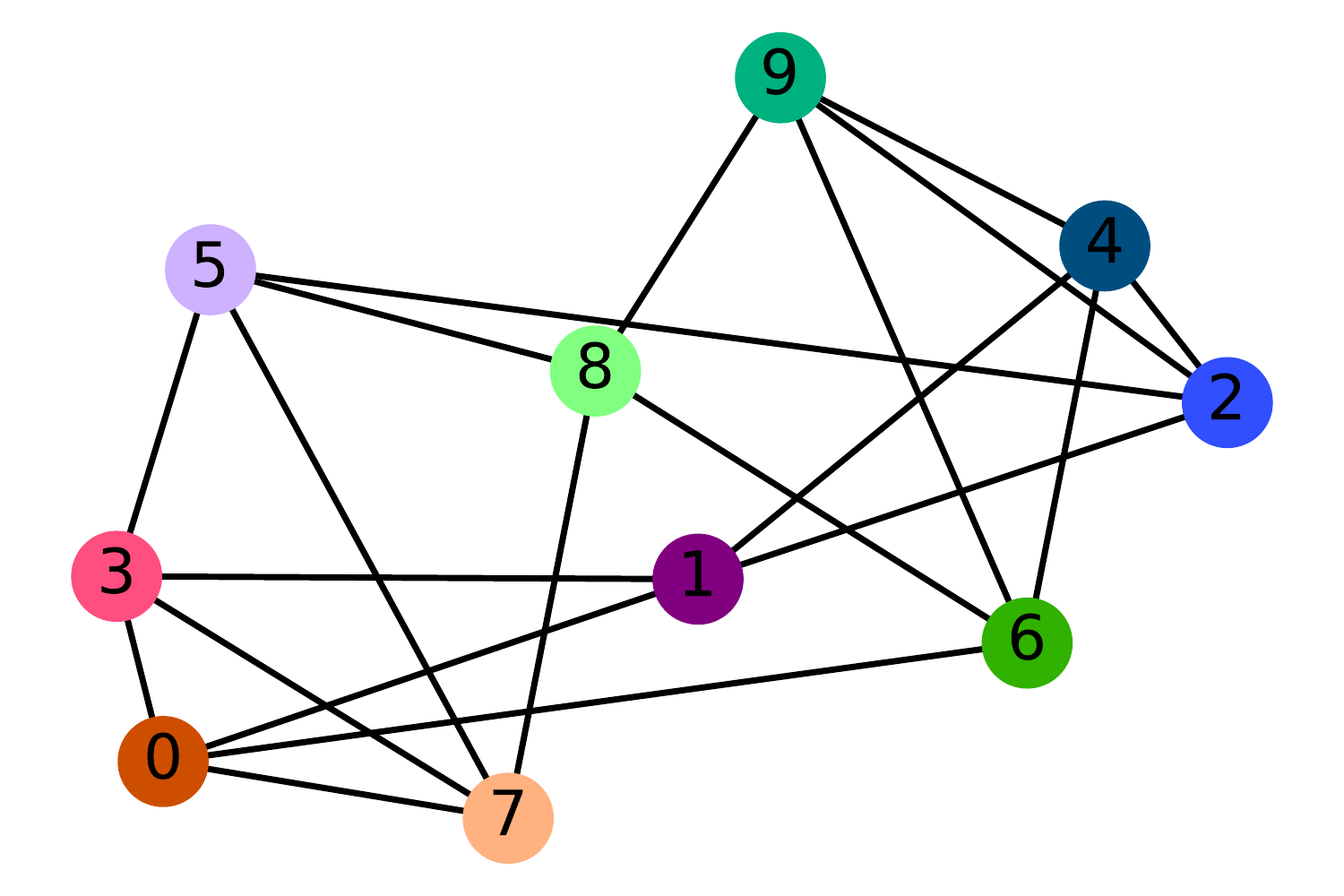}
        }
    \subfigure[$G'$]{
        \label{Fig.wl2.3}
        \includegraphics[width=0.3\textwidth]{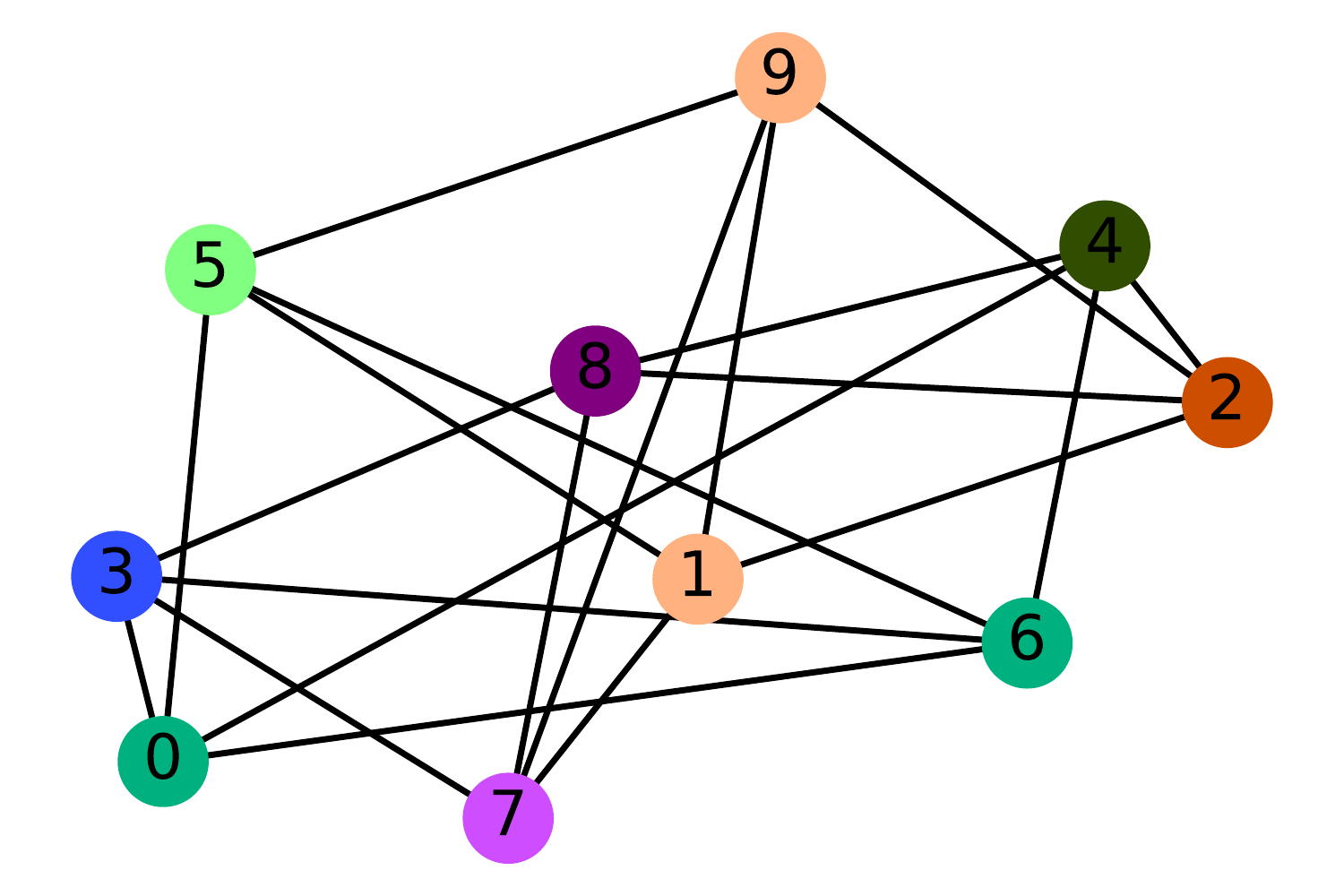}
        }
    \caption{Cospectral and 4-regular graphs are 2-WL equivalent \cite{Dam2003WhichGA}.}
    \label{Fig.wl2}
\end{figure*}
Fig. \ref{Fig.wl2} shows non-isomorphic 2-WL equivalent graphs Cospectral ($G$) and 4-regular graph ($H$) from \cite{Dam2003WhichGA}. The adjacency matrices $\boldsymbol{A}_G$ and $\boldsymbol{A}_H$ can be written as:
\scriptsize{
$$
A_{G}=\left(\begin{array}{llllllllll}
0 & 1 & 0 & 1 & 0 & 1 & 0 & 1 & 0 & 0 \\
1 & 0 & 1 & 1 & 1 & 0 & 0 & 0 & 0 & 0 \\
0 & 1 & 0 & 0 & 1 & 0 & 1 & 0 & 0 & 1 \\
1 & 1 & 0 & 0 & 0 & 1 & 0 & 1 & 0 & 0 \\
0 & 1 & 1 & 0 & 0 & 0 & 1 & 0 & 0 & 1 \\
1 & 0 & 0 & 1 & 0 & 0 & 0 & 0 & 1 & 1 \\
0 & 0 & 1 & 0 & 1 & 0 & 0 & 1 & 1 & 0 \\
1 & 0 & 0 & 1 & 0 & 0 & 1 & 0 & 1 & 0 \\
0 & 0 & 0 & 0 & 0 & 1 & 1 & 1 & 0 & 1 \\
0 & 0 & 1 & 0 & 1 & 1 & 0 & 0 & 1 & 0
\end{array}\right), $$
$$
A_{H}=\left(\begin{array}{llllllllll}
0 & 1 & 0 & 1 & 0 & 0 & 1 & 1 & 0 & 0 \\
1 & 0 & 1 & 1 & 1 & 0 & 0 & 0 & 0 & 0 \\
0 & 1 & 0 & 0 & 1 & 1 & 0 & 0 & 0 & 1 \\
1 & 1 & 0 & 0 & 0 & 1 & 0 & 1 & 0 & 0 \\
0 & 1 & 1 & 0 & 0 & 0 & 1 & 0 & 0 & 1 \\
0 & 0 & 1 & 1 & 0 & 0 & 0 & 1 & 1 & 0 \\
1 & 0 & 0 & 0 & 1 & 0 & 0 & 0 & 1 & 1 \\
1 & 0 & 0 & 1 & 0 & 1 & 0 & 0 & 1 & 0 \\
0 & 0 & 0 & 0 & 0 & 1 & 1 & 1 & 0 & 1 \\
0 & 0 & 1 & 0 & 1 & 0 & 1 & 0 & 1 & 0
\end{array}\right).
$$
}
We also plot EDENs for $G$ and $H$, and randomly permute nodes in $G$ to get $G'$.
EDEN can distinguish non-isomorphic graphs under the premise that isomorphic graphs have same color layout, Up to 2-WL Equivalent.

\subsection{3-WL equivalent graphs}
We propose the fail example of our EDEN on strongly regular graphs, which are 3-WL equivalent in Fig. \ref{Fig.wl3}.
\begin{figure*}
    \centering  
    \subfigure[$G$]{
        \label{Fig.wl3.1}
        \includegraphics[width=0.3\textwidth]{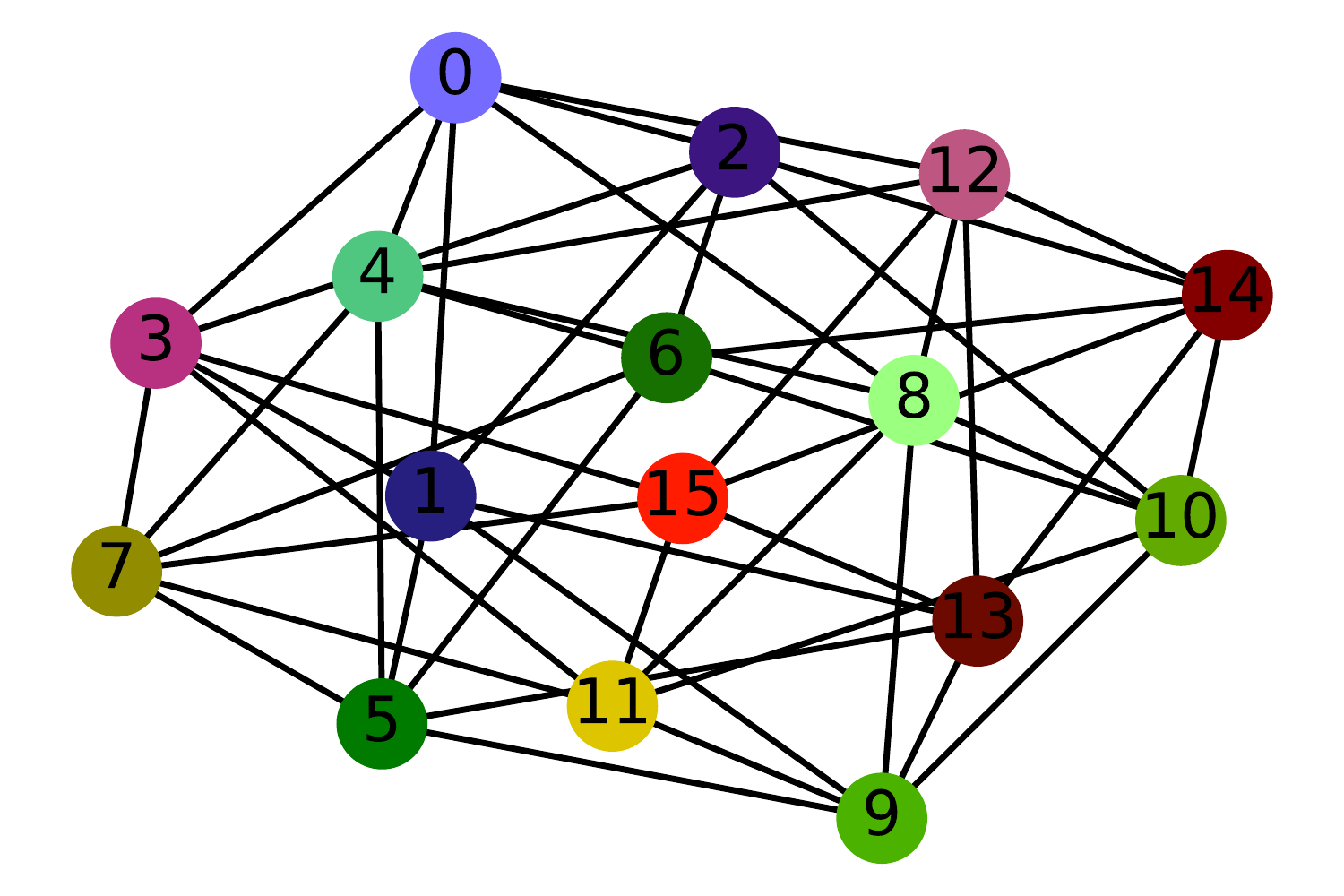}
        }
    \subfigure[$H$]{
        \label{Fig.wl3.2}
        \includegraphics[width=0.3\textwidth]{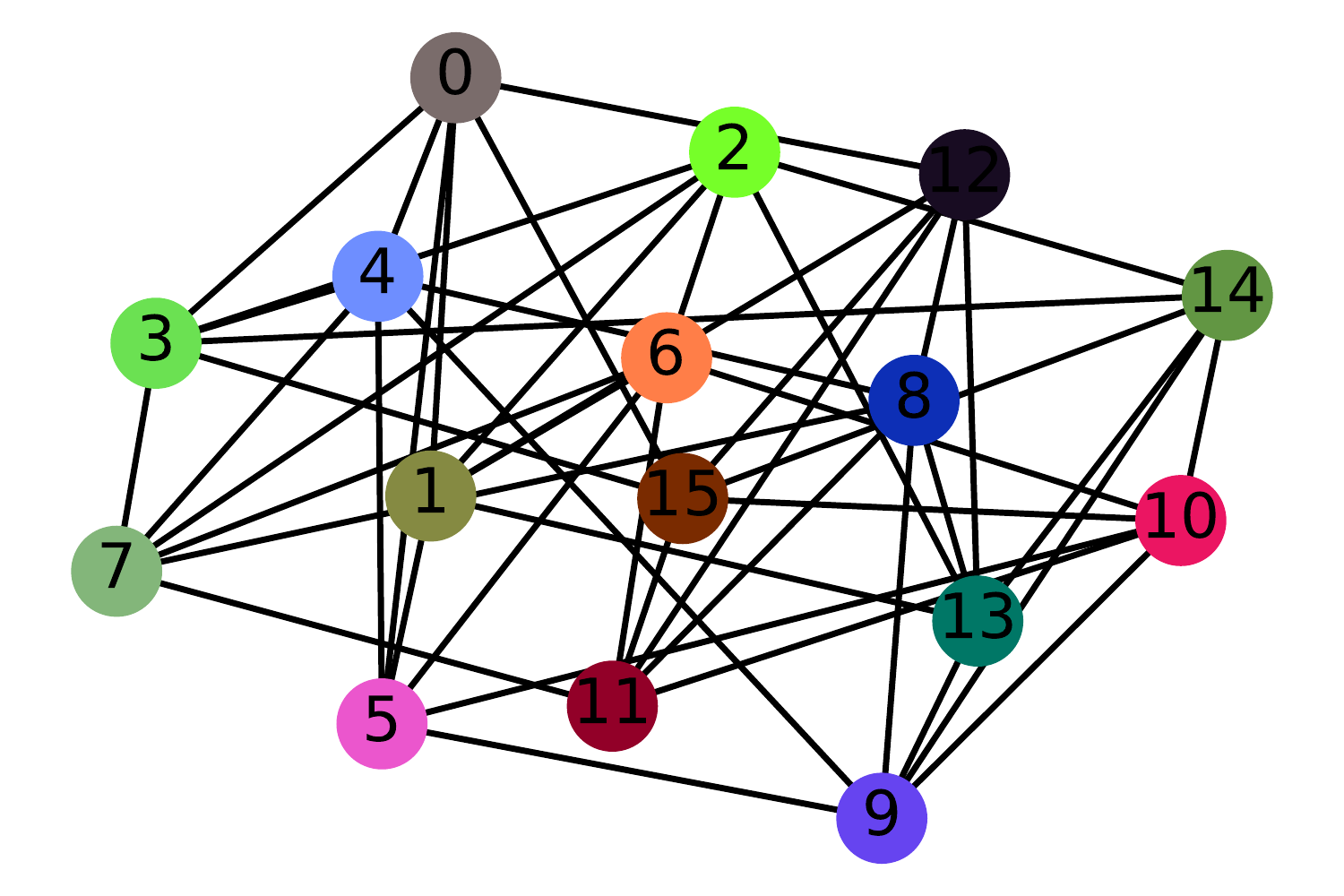}
        }
    \subfigure[$G'$]{
        \label{Fig.wl3.3}
        \includegraphics[width=0.3\textwidth]{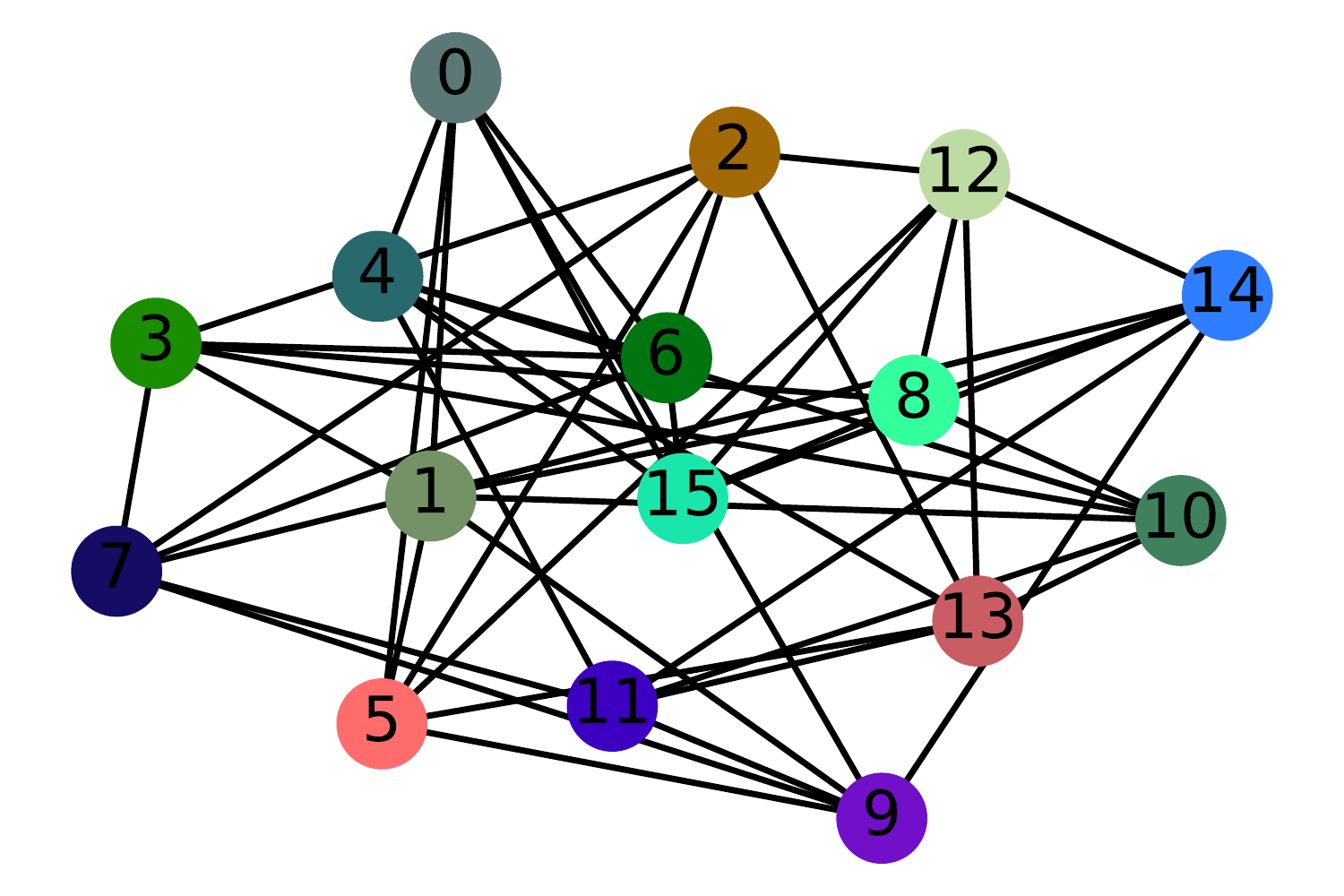}
        }
    \caption{The failure case on strongly regular graph pair. 4 $\times$ 4-rook's graph and the Shrikhande graph are 3-WL equivalent \cite{Arvind2020OnWI}.}
    \label{Fig.wl3}
\end{figure*}
 Although EDEN has different color layouts on the two non-isomorphic graphs ($G$ and $H$), the color layouts on the isomorphic graphs ($G$ and $G'$) are also different. This is because EDEN is very different from other algorithms for the graph isomorphism test, \emph{i.e.,} \textbf{EDEN may map isomorphic strongly regular graphs to different results}. In this case, the adjacency matrices $A_G$ and $A_H$ are as follows,
 
$A_G=$
$$
\tiny{
\left(\begin{array}{llllllllllllllll}
0 & 1 & 1 & 1 & 1 & 0 & 0 & 0 & 1 & 0 & 0 & 0 & 1 & 0 & 0 & 0 \\
1 & 0 & 1 & 1 & 0 & 1 & 0 & 0 & 0 & 1 & 0 & 0 & 0 & 1 & 0 & 0 \\
1 & 1 & 0 & 1 & 0 & 0 & 1 & 0 & 0 & 0 & 1 & 0 & 0 & 0 & 1 & 0 \\
1 & 1 & 1 & 0 & 0 & 0 & 0 & 1 & 0 & 0 & 0 & 1 & 0 & 0 & 0 & 1 \\
1 & 0 & 0 & 0 & 0 & 1 & 1 & 1 & 1 & 0 & 0 & 0 & 1 & 0 & 0 & 0 \\
0 & 1 & 0 & 0 & 1 & 0 & 1 & 1 & 0 & 1 & 0 & 0 & 0 & 1 & 0 & 0 \\
0 & 0 & 1 & 0 & 1 & 1 & 0 & 1 & 0 & 0 & 1 & 0 & 0 & 0 & 1 & 0 \\
0 & 0 & 0 & 1 & 1 & 1 & 1 & 0 & 0 & 0 & 0 & 1 & 0 & 0 & 0 & 1 \\
1 & 0 & 0 & 0 & 1 & 0 & 0 & 0 & 0 & 1 & 1 & 1 & 1 & 0 & 0 & 0 \\
0 & 1 & 0 & 0 & 0 & 1 & 0 & 0 & 1 & 0 & 1 & 1 & 0 & 1 & 0 & 0 \\
0 & 0 & 1 & 0 & 0 & 0 & 1 & 0 & 1 & 1 & 0 & 1 & 0 & 0 & 1 & 0 \\
0 & 0 & 0 & 1 & 0 & 0 & 0 & 1 & 1 & 1 & 1 & 0 & 0 & 0 & 0 & 1 \\
1 & 0 & 0 & 0 & 1 & 0 & 0 & 0 & 1 & 0 & 0 & 0 & 0 & 1 & 1 & 1 \\
0 & 1 & 0 & 0 & 0 & 1 & 0 & 0 & 0 & 1 & 0 & 0 & 1 & 0 & 1 & 1 \\
0 & 0 & 1 & 0 & 0 & 0 & 1 & 0 & 0 & 0 & 1 & 0 & 1 & 1 & 0 & 1 \\
0 & 0 & 0 & 1 & 0 & 0 & 0 & 1 & 0 & 0 & 0 & 1 & 1 & 1 & 1 & 0
\end{array}\right),
}
$$
$A_H=$
$$
\tiny{
\left(\begin{array}{llllllllllllllll}
0 & 1 & 0 & 1 & 1 & 1 & 0 & 0 & 0 & 0 & 0 & 0 & 1 & 0 & 0 & 1 \\
1 & 0 & 1 & 0 & 0 & 1 & 1 & 0 & 0 & 0 & 0 & 0 & 1 & 1 & 0 & 0 \\
0 & 1 & 0 & 1 & 0 & 0 & 1 & 1 & 0 & 0 & 0 & 0 & 0 & 1 & 1 & 0 \\
1 & 0 & 1 & 0 & 1 & 0 & 0 & 1 & 0 & 0 & 0 & 0 & 0 & 0 & 1 & 1 \\
1 & 0 & 0 & 1 & 0 & 1 & 0 & 1 & 1 & 1 & 0 & 0 & 0 & 0 & 0 & 0 \\
1 & 1 & 0 & 0 & 1 & 0 & 1 & 0 & 0 & 1 & 1 & 0 & 0 & 0 & 0 & 0 \\
0 & 1 & 1 & 0 & 0 & 1 & 0 & 1 & 0 & 0 & 1 & 1 & 0 & 0 & 0 & 0 \\
0 & 0 & 1 & 1 & 1 & 0 & 1 & 0 & 1 & 0 & 0 & 1 & 0 & 0 & 0 & 0 \\
0 & 0 & 0 & 0 & 1 & 0 & 0 & 1 & 0 & 1 & 0 & 1 & 1 & 1 & 0 & 0 \\
0 & 0 & 0 & 0 & 1 & 1 & 0 & 0 & 1 & 0 & 1 & 0 & 0 & 1 & 1 & 0 \\
0 & 0 & 0 & 0 & 0 & 1 & 1 & 0 & 0 & 1 & 0 & 1 & 0 & 0 & 1 & 1 \\
0 & 0 & 0 & 0 & 0 & 0 & 1 & 1 & 1 & 0 & 1 & 0 & 1 & 0 & 0 & 1 \\
1 & 1 & 0 & 0 & 0 & 0 & 0 & 0 & 1 & 0 & 0 & 1 & 0 & 1 & 0 & 1 \\
0 & 1 & 1 & 0 & 0 & 0 & 0 & 0 & 1 & 1 & 0 & 0 & 1 & 0 & 1 & 0 \\
0 & 0 & 1 & 1 & 0 & 0 & 0 & 0 & 0 & 1 & 1 & 0 & 0 & 1 & 0 & 1 \\
1 & 0 & 0 & 1 & 0 & 0 & 0 & 0 & 0 & 0 & 1 & 1 & 1 & 0 & 1 & 0
\end{array}\right).
}
$$
\subsection{Analysis on the above three cases}

The multiple roots of singular values may cause EDEN to fail on strongly regular graphs. We compute the three largest singular values of  $\boldsymbol{F_{\hat{\boldsymbol{D}}}}_G$ and $\boldsymbol{F_{\hat{\boldsymbol{D}}}}_H$ for each case:
\begin{itemize}
    \item 1-WL equivalent: [4.9790, 3.5061, 2.1254], [6.2486, 2.0653, 1.3309],
    \item 2-WL equivalent: [4.2360, 3.5615, 3.        ], [4.2360, 3.5615, 3.        ],
    \item 3-WL equivalent: [4., 4., 4.], [4., 4., 4.].  
\end{itemize}
Based on the above results, we have the following analysis. In the first case, EDENs of non-isomorphic graphs have different singular values. It is easy to judge that $G$ and $H$ are non-isomorphic. In the second case, although 4-regular graphs $G$ and $H$ share the same singular values, we can also distinguish them from EDEN. It indicates that EDEN is more expressive than using singular values as a graph readout. 
While for the last case, we find that the largest three singular values of both strongly regular graphs are equal to 4. 
As described in Eq. \eqref{eq:eigen decomposition}, PCA performs eigendecomposition of the original matrix,
the columns of the eigenvector matrix $\boldsymbol{V}$ can theoretically be sorted according to eigenvalues from largest to smallest. 
The two strong regular graphs in the last case have multiple identical eigenvalues at the same time. Therefore, the order of eigenvectors in $\boldsymbol{V}$ has randomness in the process of numerical calculation, which eventually leads to inconsistent results of the model.


\section{Graph Isomorphic Test}
\label{app:GIT}
Baseline methods utilize nodes' degrees as input features, which is considered the output of executing a 1-WL Test iteration \cite{breaking}.
But ENDEN obtains $k$-dimensional singular vectors after PCA, and isomorphic graphs may be projected to different embeddings because of the inevitable numerical precision caused by PCA.
Due to EDEN's nature, it is necessary to adjust the judgment threshold for EDEN.
We propose strict judgment conditions for graph isomorphism. To be specific, we randomly generate isomorphic graphs for each dataset, and apply EDEN to judge whether the graph pairs are non-isomorphic \textbf{under the premise that the known isomorphic pairs graphs are judged correctly}. 

\section{Experiments on Real-World Datasets}
\label{app:details}

\subsection{Experimental environments}
We implement codes with Python 3.7, PyTorch 1.8.0, and Cuda 11.3. Experiments are run on an Ubuntu 16.04 LTS server with the Intel i7-6900K CPU and four NVIDIA 1080 Ti GPUs.

\subsection{Hyper-parameters}
\paragraph{Determine hyper-parameters} 
We use a consistent data split strategy (\emph{i.e.,} 80/20\%  train/validation split), and set the batch size to 128 and the learning rate to 0.01 for all methods. We use Adam as the optimizer, ReLU as the activation function, and MLP as the readout function. The number of layers is set to 5 for the link prediction task and 3 for other tasks.
We use a grid search strategy to tune the embedding size of EDEN in $[2, 3, \cdots, 10]$.

The optimal embedding sizes of baseline methods on test datasets are as follows:
\begin{itemize}
    \item \textbf{Cora}: 32 for all MPNNs (\textit{i.e.}, GCN, GraphSAGE, GAT and GIN).
    \item \textbf{CiteSeer}: 32 for all MPNNs.
    \item \textbf{ENZYMES}: 128 for GraphSAGE, 256 for the other MPNNs.
    \item \textbf{PROTEINS (link prediction)}: 128 for all MPNNs.
    \item \textbf{MUTAG}: 256 for all MPNNs.
    \item \textbf{PROTEINS (graph classification)}: 128 for all MPNNs.
\end{itemize}

\subsection{Fair comparison}
We implement EDEN with Pytorch-Geometric \cite{pyg} and GraphGym \cite{gym}.
For a fair comparison, we use the official code of ID-GNN \cite{IDGNN} and manually fine-tune the hyper-parameter of feature dimension to achieve close accuracy reported by ID-GNN, because ID-GNN does not specify it in the original article. We re-implement the baselines in IDGNN's, and the maximum accuracy difference is less than $\pm0.01$. In addition, we choose the best results reported in related works. When implementing Laplacian PE \cite{benchmarking}, we keep the same settings as EDEN. Moreover, we choose the best results among the four GNN backbones.

\textbf{Task-specific settings.}
We adjust settings due to the different nature of tasks.
The MPNNs have three massage-passing layers for node and graph classification tasks, and they report the validation accuracy after 100 epochs.
The graph classification task needs an additional MLP for the readout operation.
We train models with five message-passing layers for the link prediction task, and report the ROC-AUC. 
The dimensions of EDEN and Laplacian PE are set to 8 for Cora and CiteSeer datasets, and 3 for other datasets.

\section{Limitations of Laplacian PE}
\label{app:lap}

\subsection{Laplacian PE can't go beyond 1-WL}
We perform similar visual results of other cases with Laplacian PE. As illustrated in Fig \ref{Fig.lapex1} to \ref{Fig.lap3}, \textbf{Laplacian PE fails in ALL cases, no matter how powerful the WL equivalence is.} Especially for Fig. \ref{Fig.lap1.1} and Fig. \ref{Fig.lap1.3}, just two nodes swapped (\#6 and \#7) in a 1-WL equivalent graph pair will result in great changes.
These failures behave the same as EDEN fails on the 3-WL equivalent example, because the eigenvalues of the Laplacian matrix have multiple roots. In essence, this means that the information contained in the Laplace matrix cannot exceed the 1-WL test.

\begin{figure*}
    \centering  
    \subfigure[$G$]{
        \label{Fig.lapex1.1}
        \includegraphics[width=0.3\textwidth]{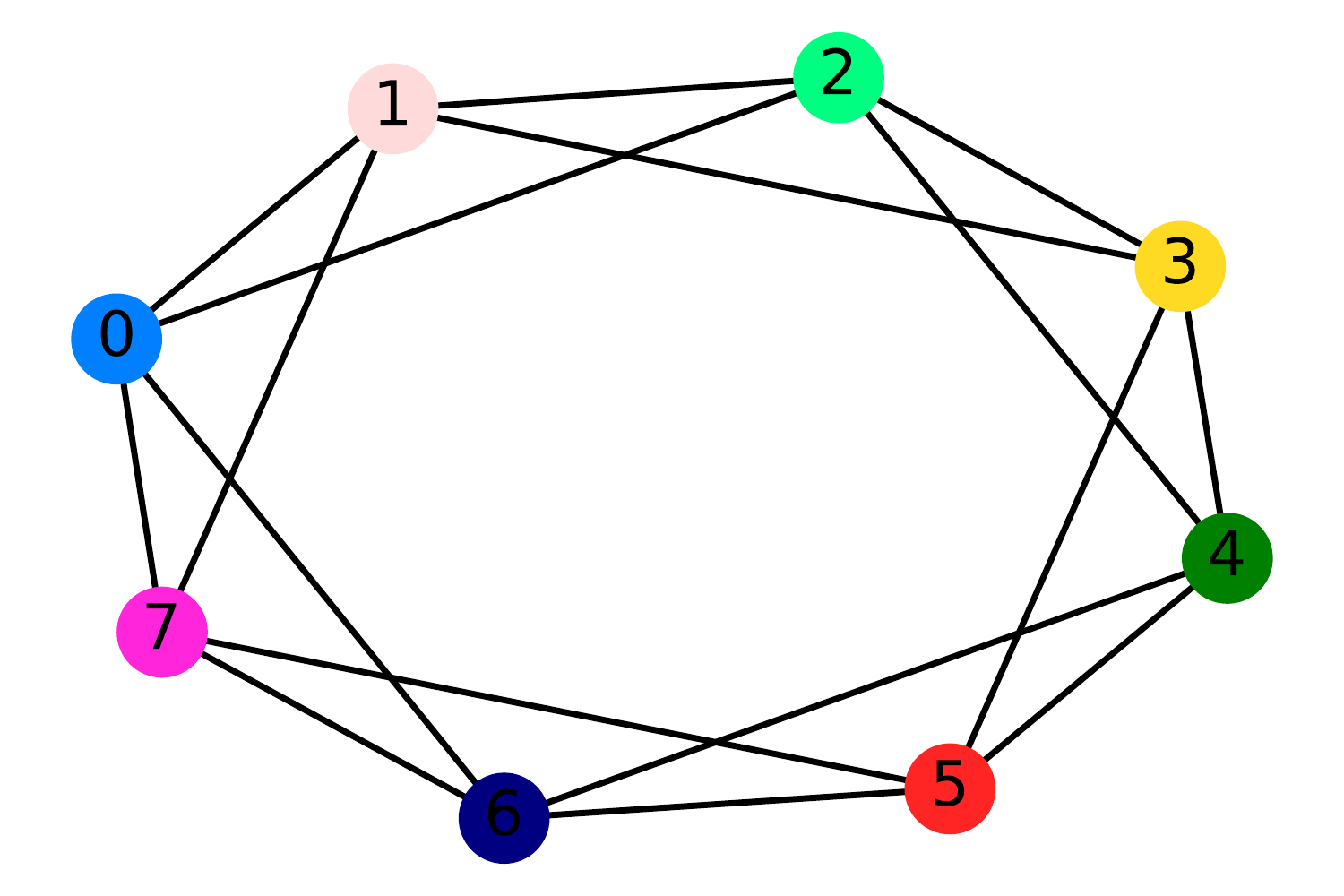}
        }
    \subfigure[$H$]{
        \label{Fig.lapex1.2}
        \includegraphics[width=0.3\textwidth]{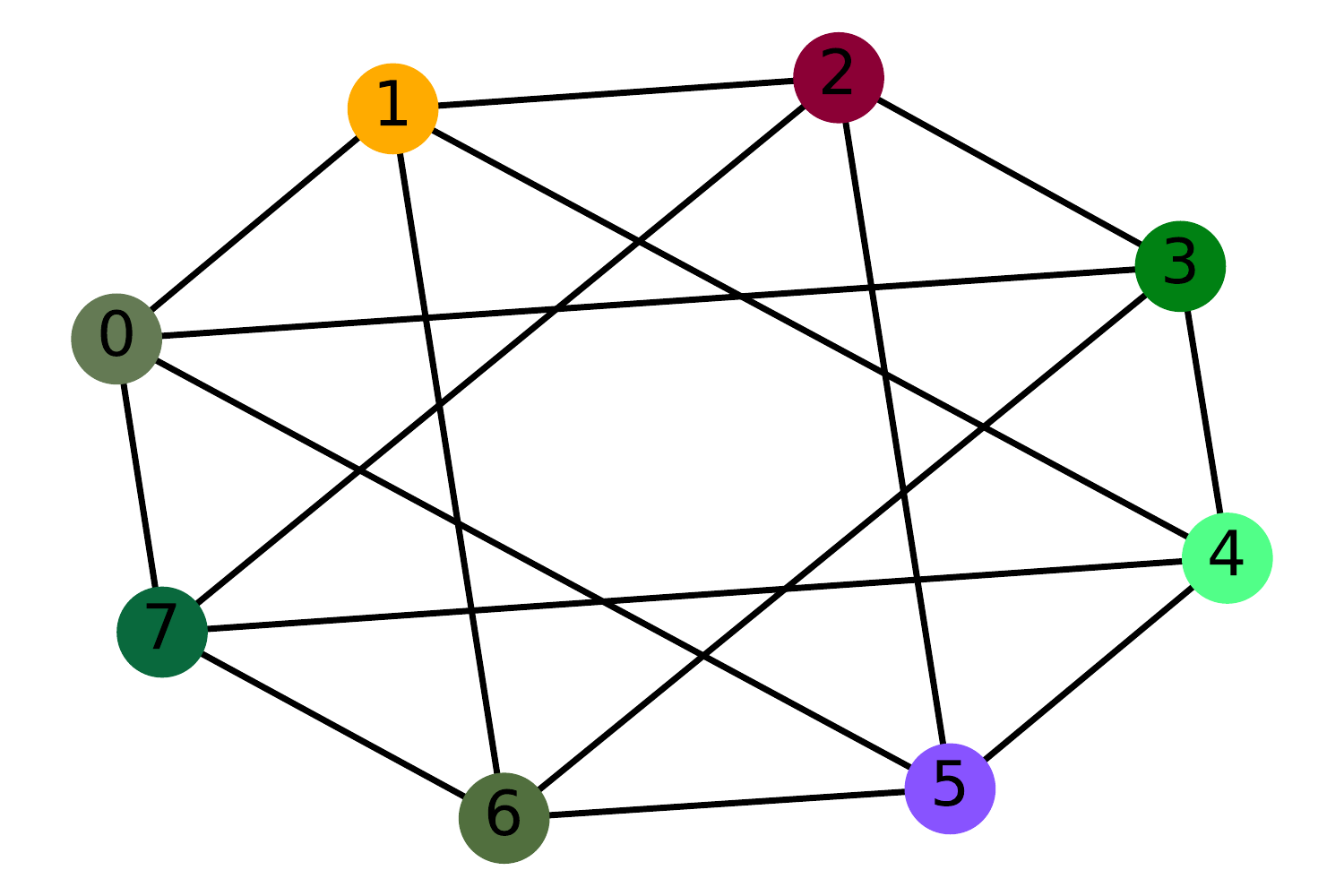}
        }
    \subfigure[$G'$]{
        \label{Fig.lapex1.3}
        \includegraphics[width=0.3\textwidth]{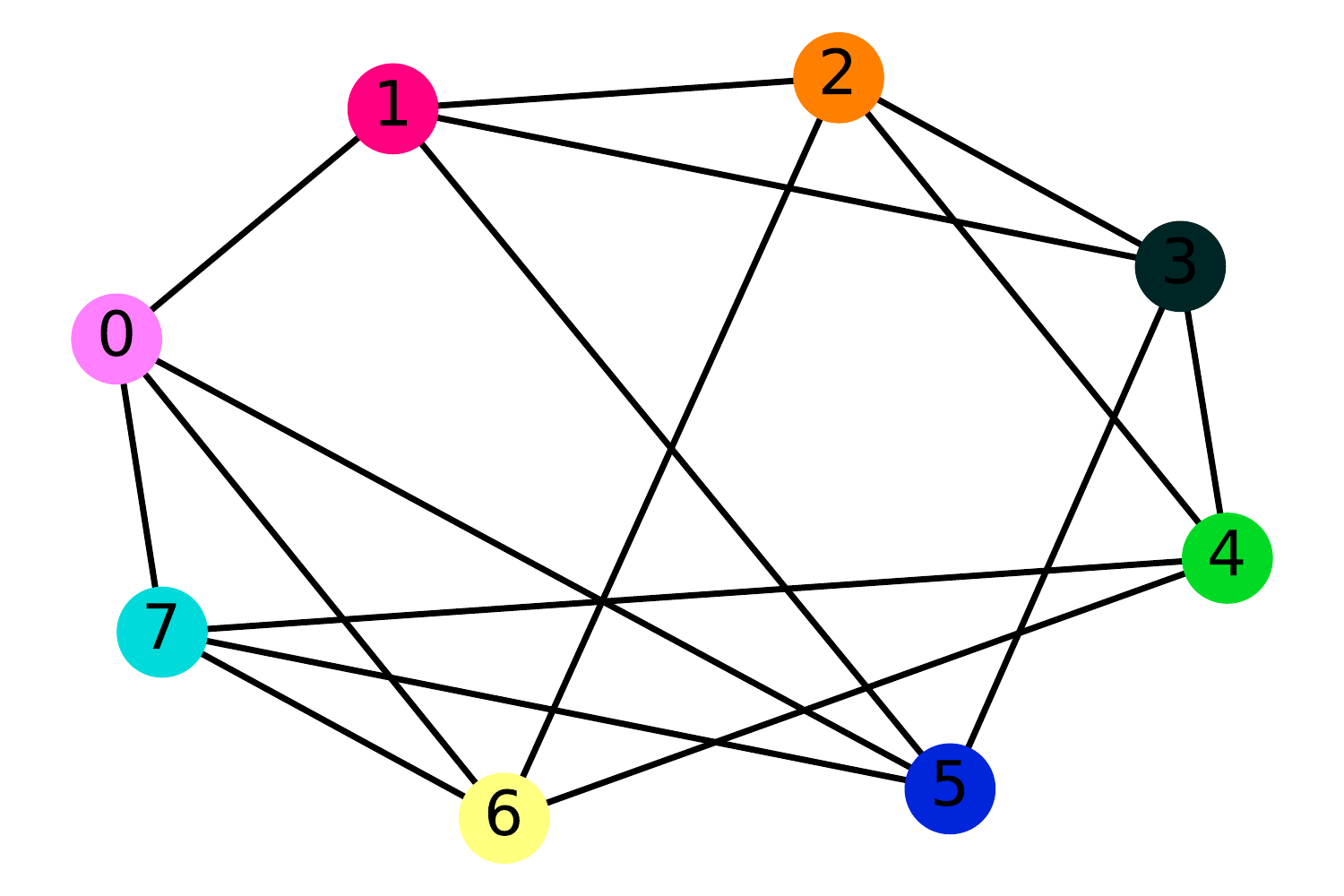}
        }
    \caption{Graph Laplacian failed to be deterministic in the case of regular graphs pair in Fig \ref{Fig.eden}.}
    \label{Fig.lapex1}
\end{figure*}

\begin{figure*}
    \centering  
    \subfigure[$G$]{
        \label{Fig.lap1.1}
        \includegraphics[width=0.3\textwidth]{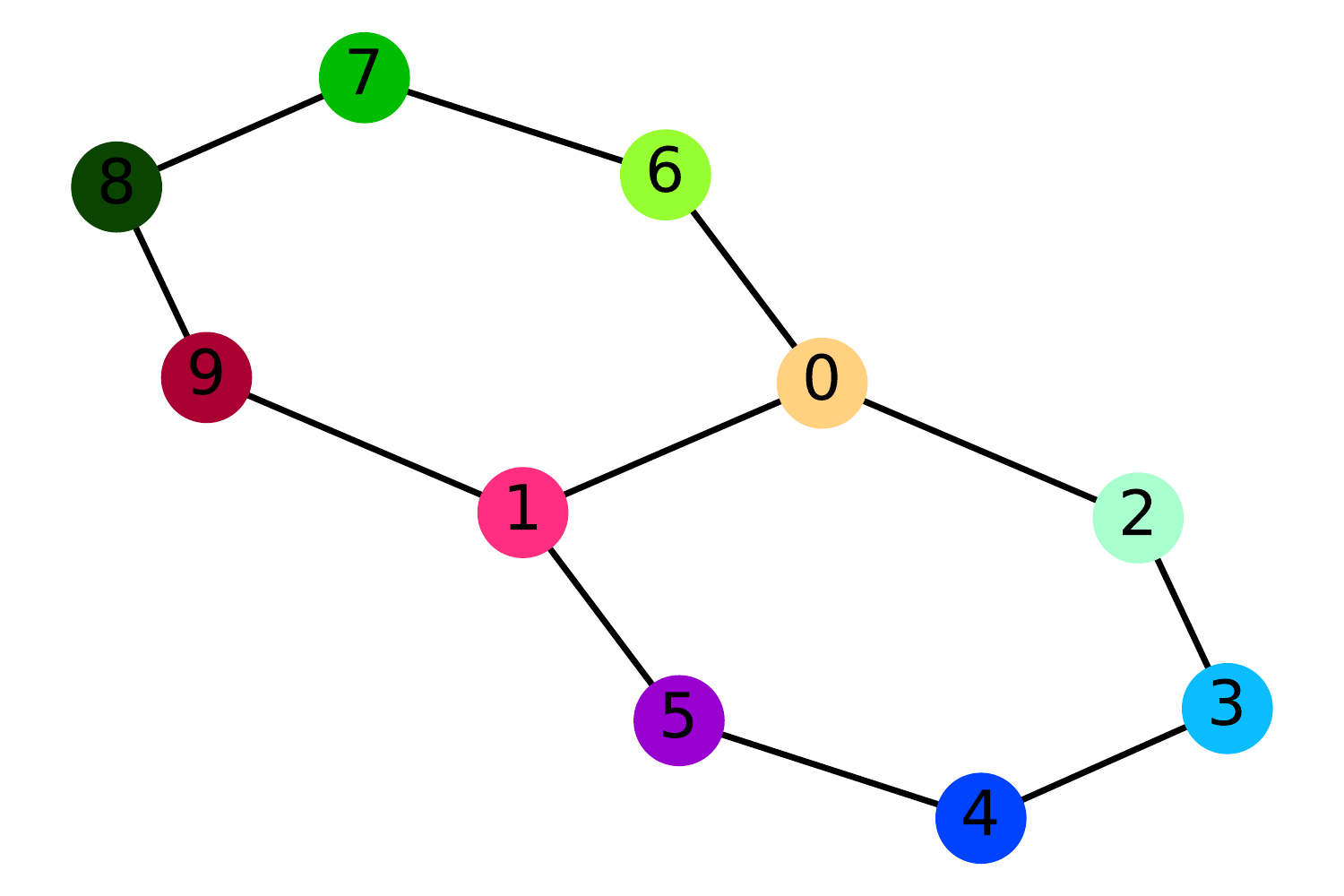}
        }
    \subfigure[$H$]{
        \label{Fig.lap1.2}
        \includegraphics[width=0.3\textwidth]{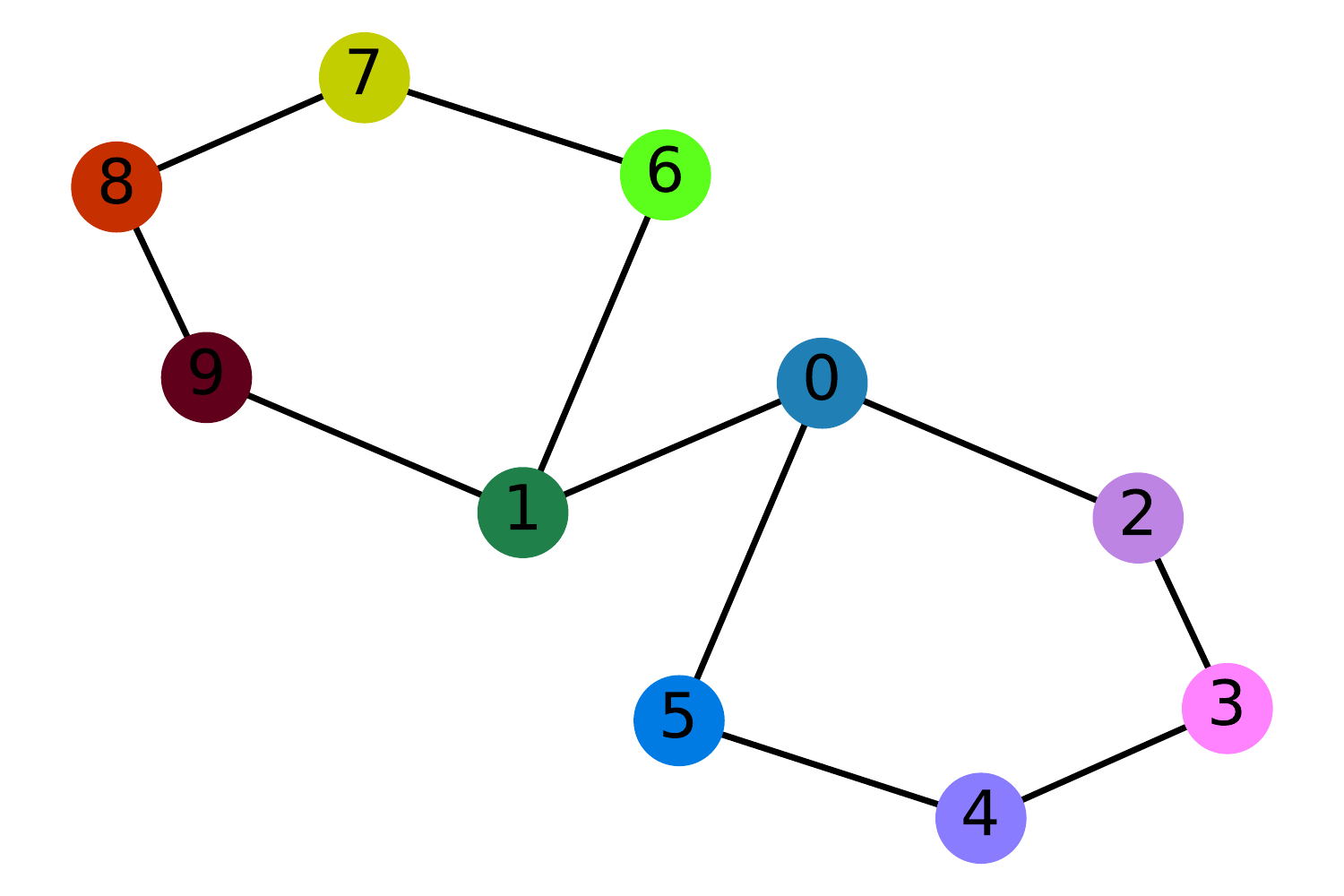}
        }
    \subfigure[$G'$]{
        \label{Fig.lap1.3}
        \includegraphics[width=0.3\textwidth]{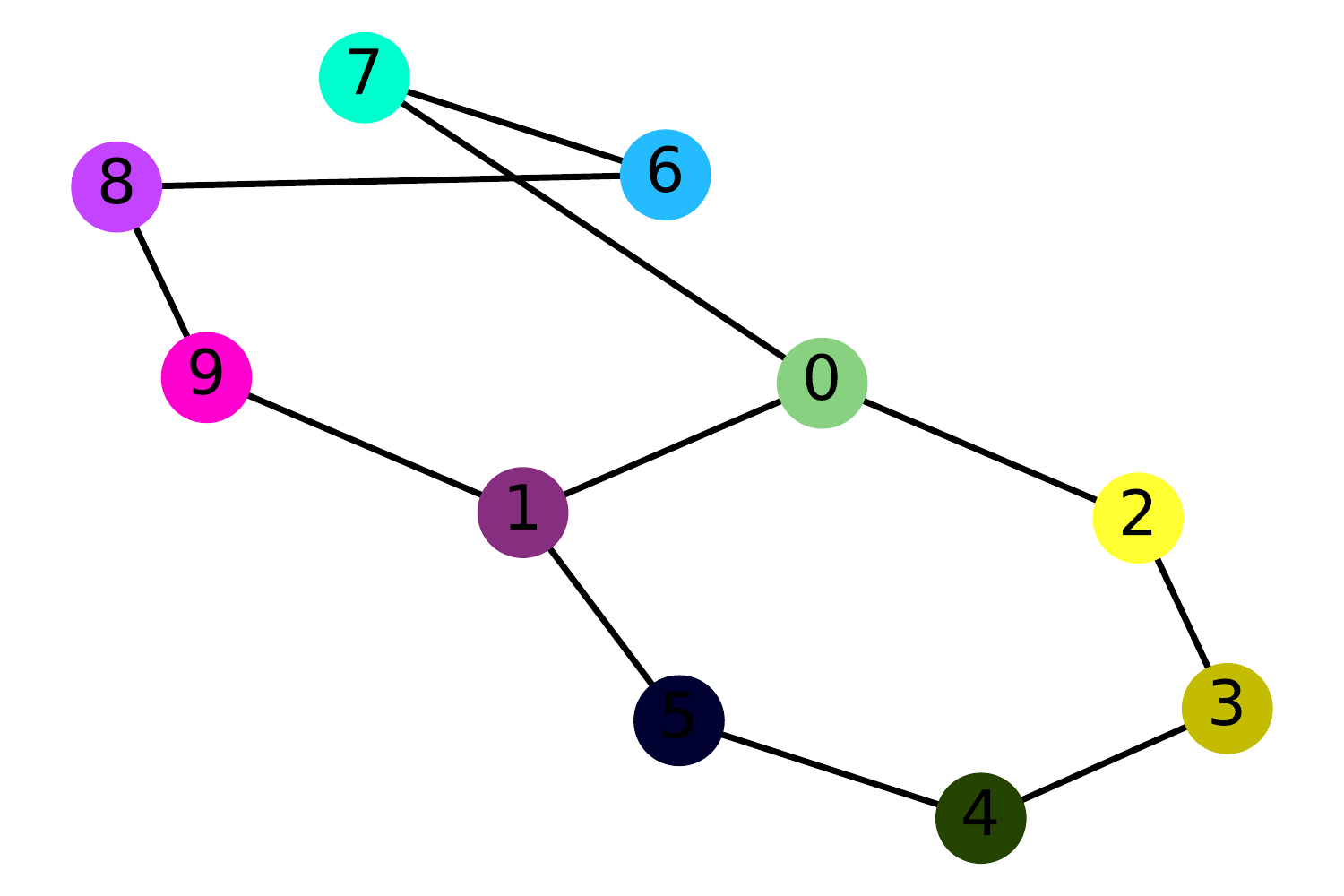}
        }
    \caption{Graph Laplacian failed to be deterministic in the case of 1-WL equivalent Decalin and Bicyclopentyl graphs. }
    \label{Fig.lap1}
\end{figure*}

\begin{figure*}
    \centering  
    \subfigure[$G$]{
        \label{Fig.lap2.1}
        \includegraphics[width=0.3\textwidth]{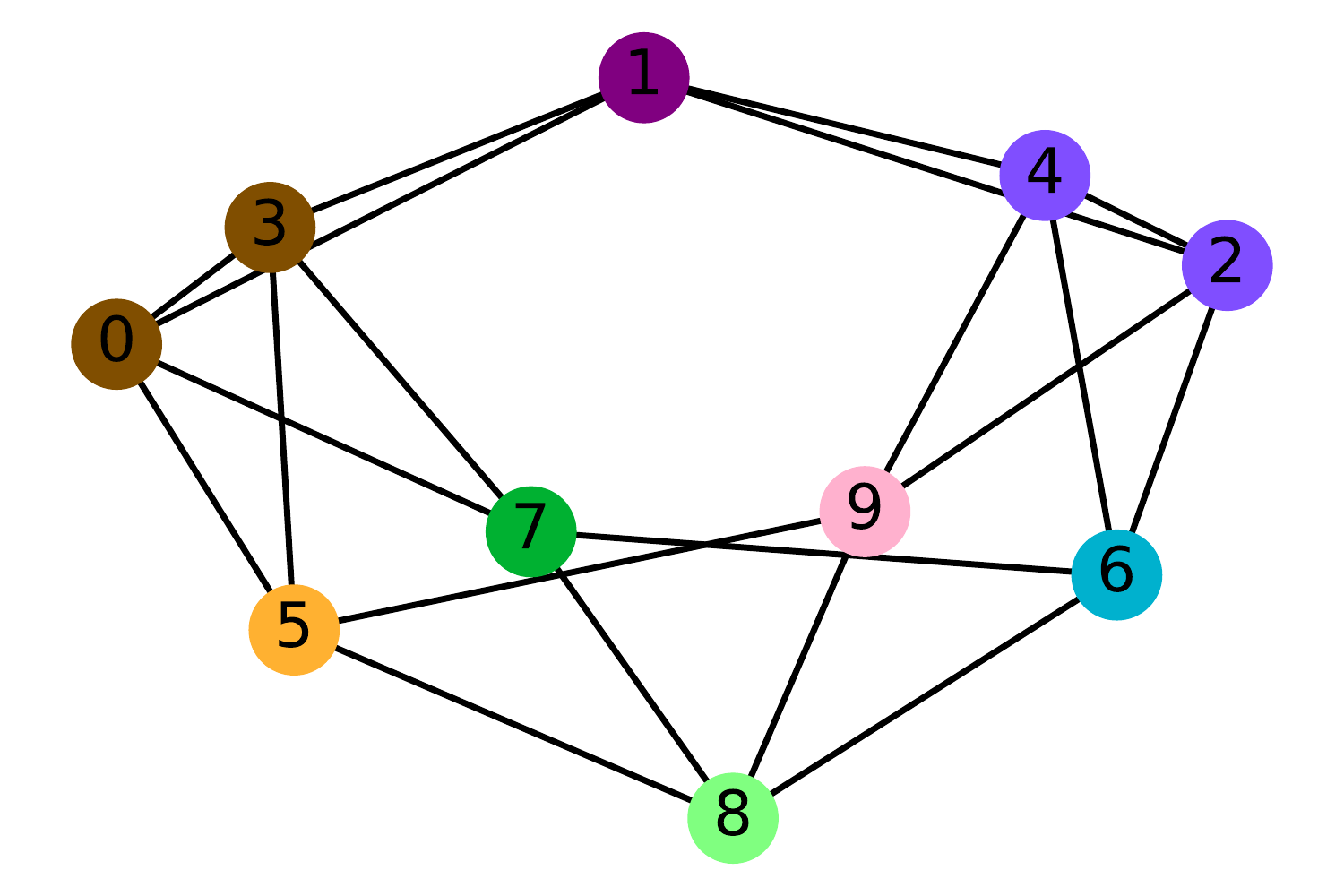}
        }
    \subfigure[$H$]{
        \label{Fig.lap2.2}
        \includegraphics[width=0.3\textwidth]{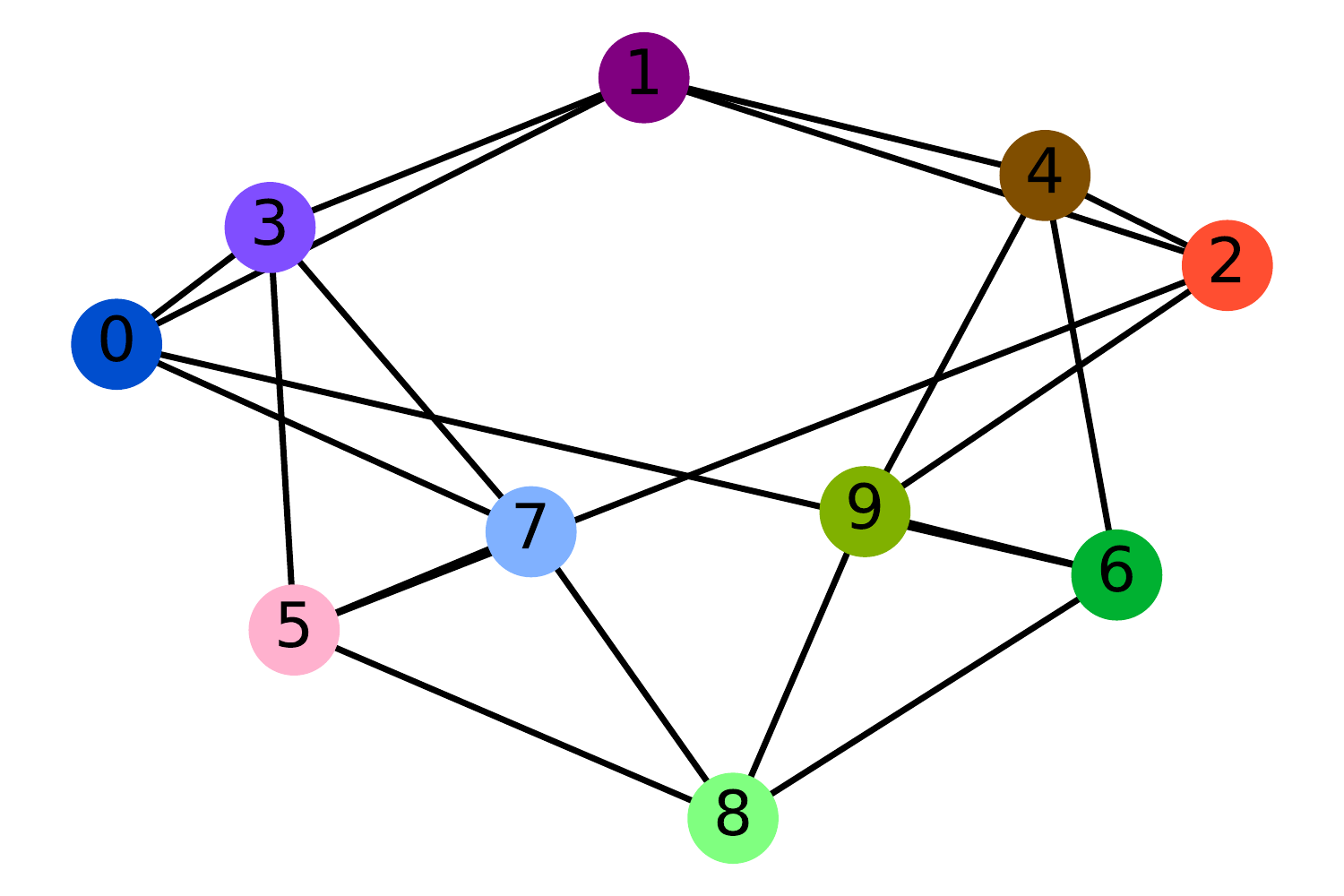}
        }
    \subfigure[$G'$]{
        \label{Fig.lap2.3}
        \includegraphics[width=0.3\textwidth]{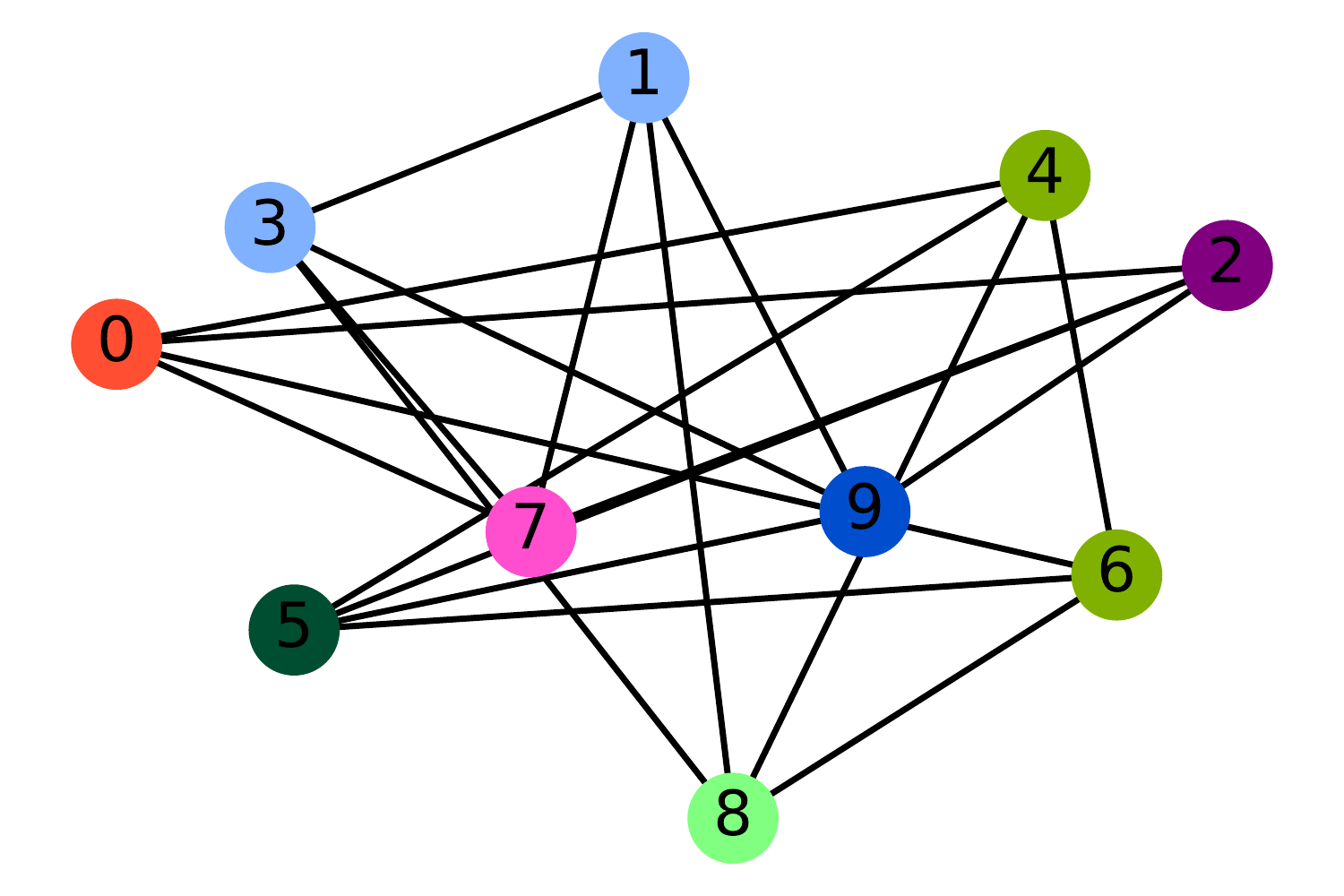}
        }
    \caption{Graph Laplacian failed to be deterministic in the case of 2-WL equivalent Cospectral and 4-regular graphs.}
    \label{Fig.lap2}
\end{figure*}

\begin{figure*}
    \centering  
    \subfigure[$G$]{
        \label{Fig.lap3.1}
        \includegraphics[width=0.3\textwidth]{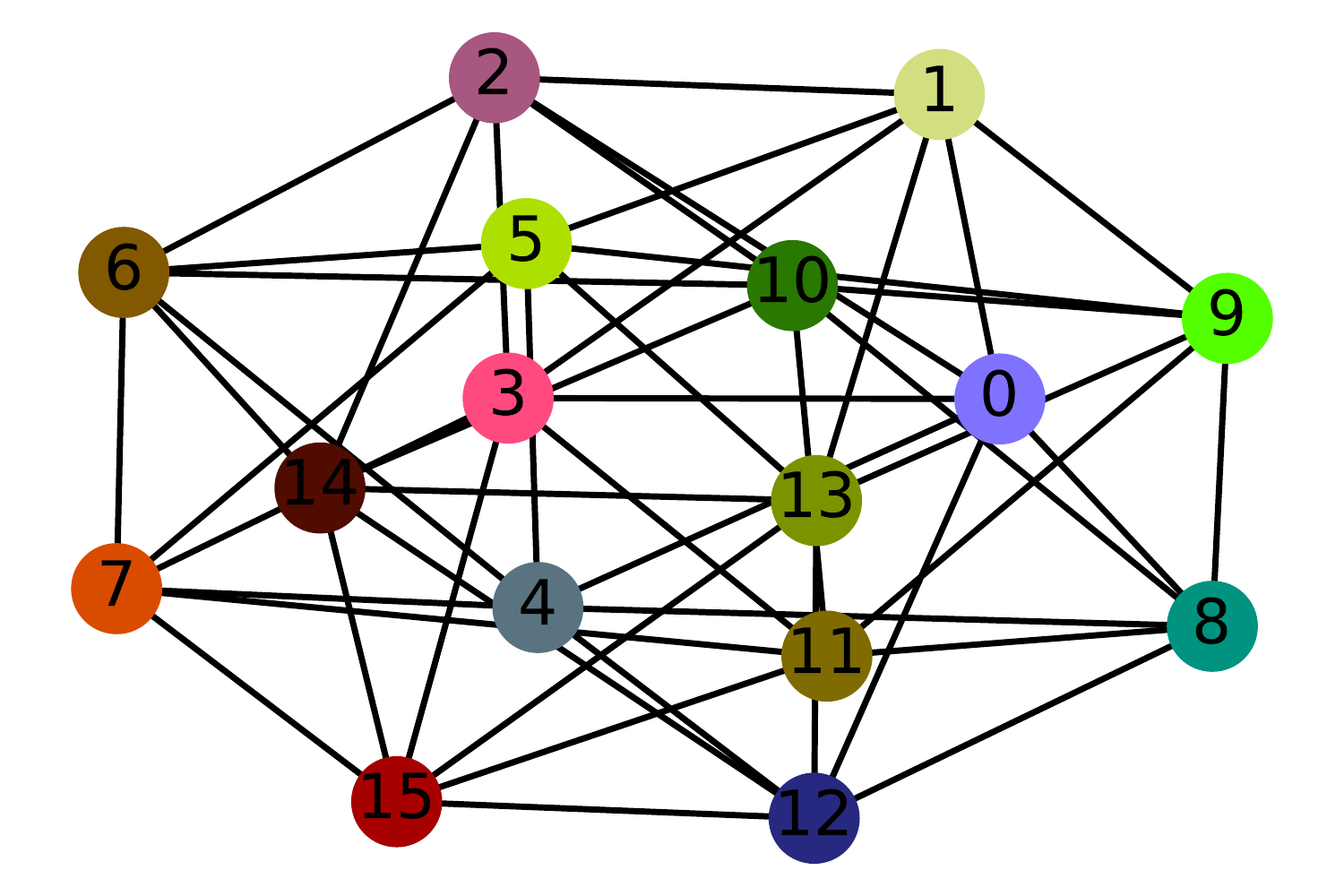}
        }
    \subfigure[$H$]{
        \label{Fig.lap3.2}
        \includegraphics[width=0.3\textwidth]{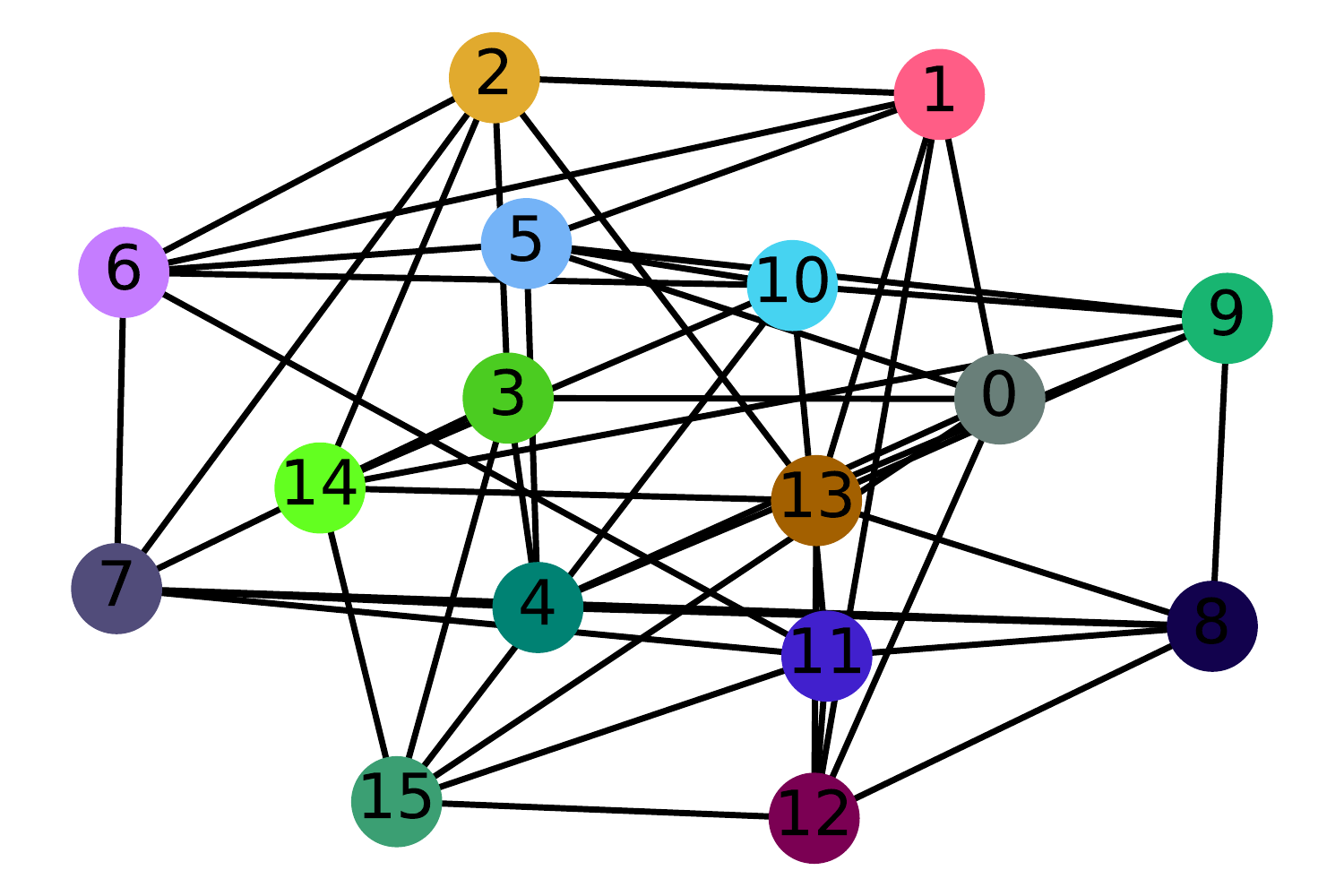}
        }
    \subfigure[$G'$]{
        \label{Fig.lap3.3}
        \includegraphics[width=0.3\textwidth]{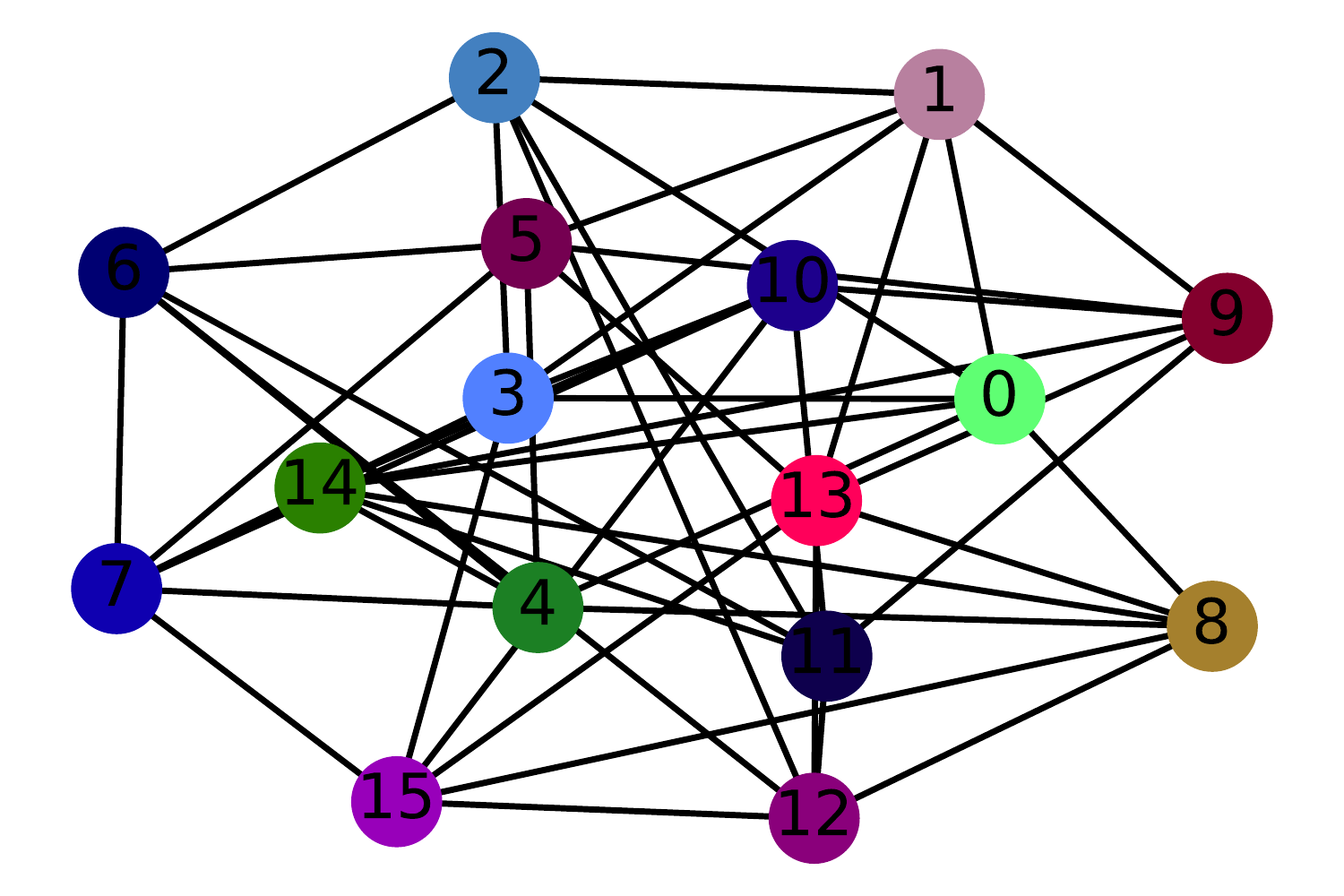}
        }
    \caption{Graph Laplacian failed to be deterministic in the case of 3-WL equivalent strong-regular graphs. }
    \label{Fig.lap3}
\end{figure*}

\subsection{Laplacian PEs are less meaningful in real-world graphs}
We plot Laplacian PE on real-word datasets in Fig. \ref{Fig.lap4} in the same way as Fig. \ref{Fig.ablation}. The layout of the nodes is somewhat different, but their IDs remain the same. We find that EDEN has more uniform and rich colors on the real-world dataset. It shows that EDEN maps nodes more evenly to a board Euclidean space. Nevertheless, the existing Laplace PEs, as shown in Fig. \ref{Fig.laprealmin}, are drabber in color (primarily gray or purple), which means that Laplacian PE maps most of the nodes into a small part of the Euclidean space. 
In addition, we take the decomposition results corresponding to the largest eigenvalues like EDEN. At this time, only the points with highest degrees have bright colors, and the colors of other nodes are almost indistinguishable (See Fig. \ref{Fig.laprealmax}). It is worth mentioning that all positional encodings are equally normalized to convert to RGB brightness.
We visualize the 2D EDEN and Laplacian PEs in Fig. \ref{Fig.lap5} without node labels. 
It can be observed that the results of EDEN are linearly separable according to actual node labels, but the Laplacian PE using minimal or maximal eigenvalues map all the nodes to a a small range in the Euclidean space.
In summary, the Laplace's PE is less meaningful than EDEN in real-world graphs.

\begin{figure*}
    \centering  
    \subfigure[EDEN]{
        \label{Fig.edenreal}
        \includegraphics[width=0.3\textwidth]{figures/katare_graph_random_color.pdf}
        }
    \subfigure[Minimal non-trivial eigenvalues]{
        \label{Fig.laprealmin}
        \includegraphics[width=0.3\textwidth]{figures/realword_PEmax.pdf}
        }
    \subfigure[Maximal non-trivial eigenvalues]{
        \label{Fig.laprealmax}
        \includegraphics[width=0.3\textwidth]{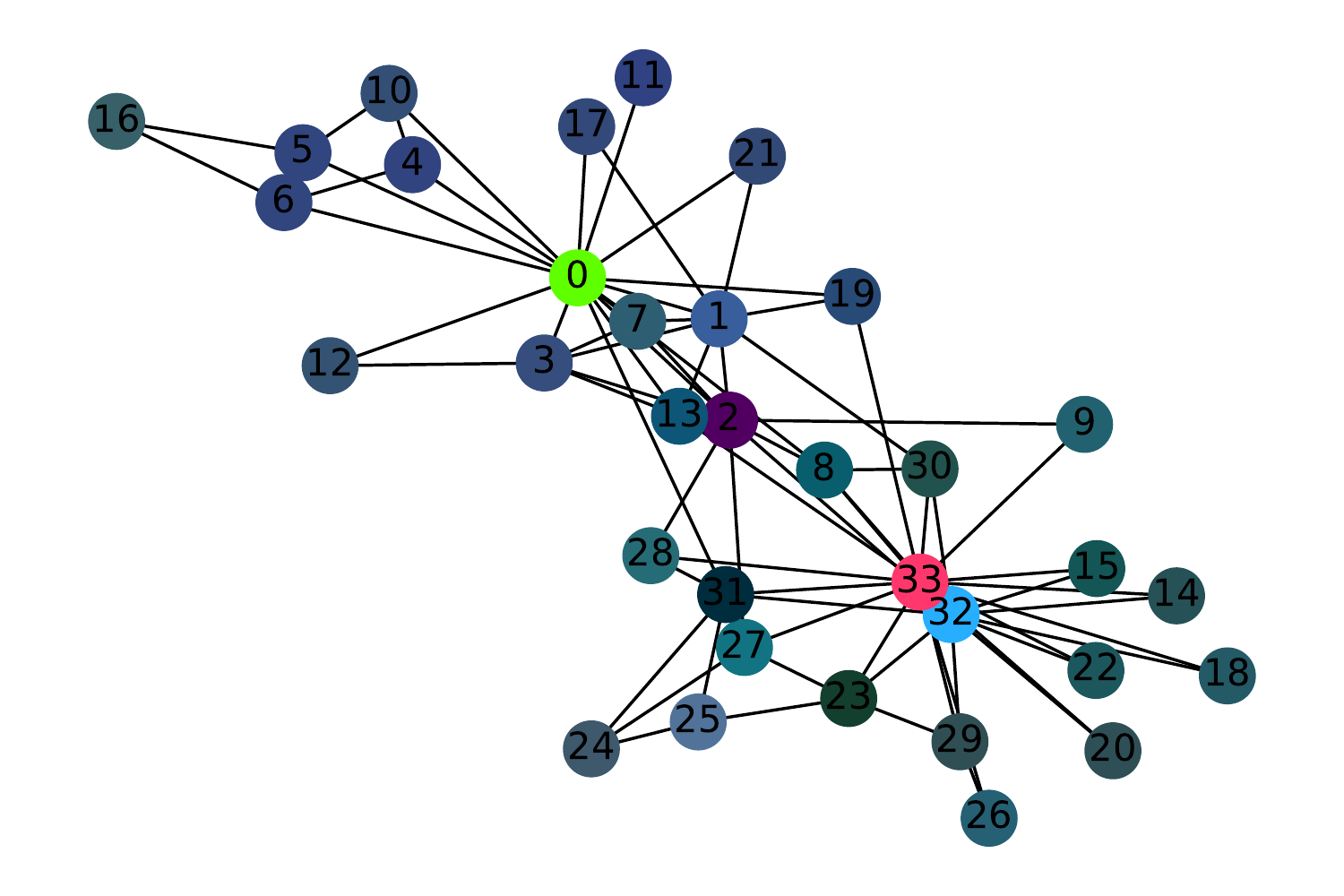}
        }
    \caption{Graph Laplacian based Position Encodings are too smoothing in real-world graphs. }
    \label{Fig.lap4}
\end{figure*}

\begin{figure*}
    \centering  
    \subfigure[2D EDEN \textbf{WITHOUT USING LABELS}.]{
        \label{Fig.edennb1}
        \includegraphics[width=0.265\textwidth]{figures/PE.pdf}
        }
    \subfigure[2D Laplacian PE with Minimal non-trivial eigenvalues]{
        \label{Fig.lapsb2}
        \includegraphics[width=0.3\textwidth]{figures/LAPEMIN.pdf}
        }
    \subfigure[2D Laplacian PE with Maximal eigenvalues]{
        \label{Fig.lapsb3}
        \includegraphics[width=0.29\textwidth]{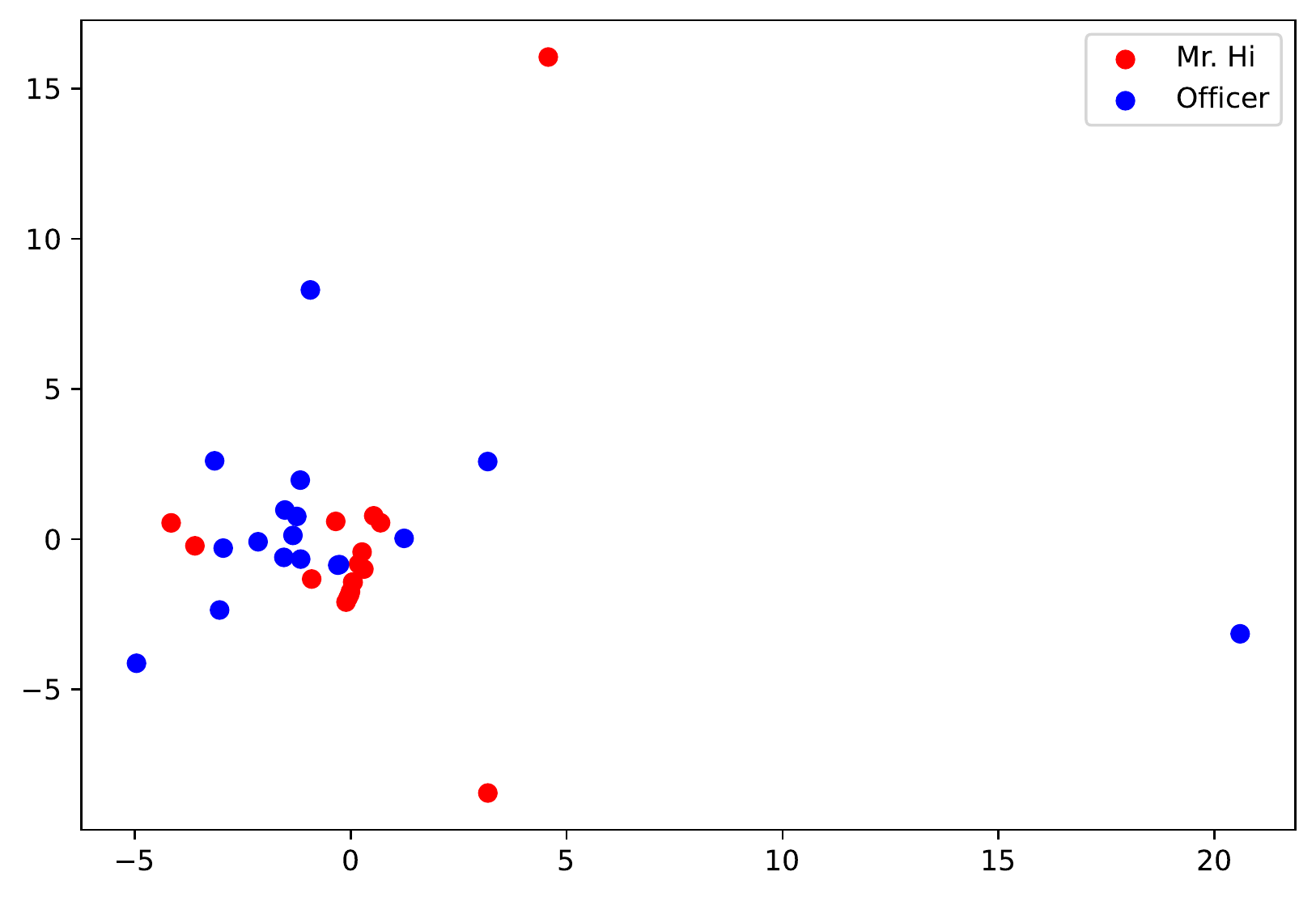}
        }
    \caption{Graph Laplacian based Position Encodings are less meaningful in real-world graphs. }
    \label{Fig.lap5}
\end{figure*}
\begin{table*}
 \centering
\begin{tabular}{cccccc}
\toprule
Cora      & Baseline        & Distance              & \begin{tabular}[c]{@{}c@{}}MinMax\\ Distance\end{tabular} & \begin{tabular}[c]{@{}c@{}}R-MinMax\\ Distance\end{tabular} & EDEN                     \\ \hline
GCN       & 0.864$\pm$0.019 & 0.869$\pm$0.021       & 0.866$\pm$0.020                                           & \textbf{0.879$\pm$0.019}                                    & 0.873$\pm$0.023          \\
GraphSAGE & 0.858$\pm$0.016 & 0.862$\pm$0.010       & 0.874$\pm$0.008                                           & 0.875$\pm$0.015                                             & \textbf{0.885$\pm$0.021} \\
GAT       & 0.859$\pm$0.028 & {\ul 0.855$\pm$0.013} & 0.863$\pm$0.009                                           & 0.863$\pm$0.009                                             & \textbf{0.870$\pm$0.019} \\
GIN       & 0.854$\pm$0.019 & 0.863$\pm$0.016       & 0.867$\pm$0.008                                           & 0.874$\pm$0.017                                             & \textbf{0.879$\pm$0.020} \\ \hline
Average   & 0.859           & 0.862                 & 0.868                                                     & 0.873                                                       & \textbf{0.877}                    \\ \bottomrule
\end{tabular}
\caption{Ablation study on Cora dataset, where the {\ul underline} denotes ``being even lower than the baseline'' and the \textbf{bold} means the best performance for a given MPNN (or the average of four MPNNs).}
\label{table:cora}
\end{table*}

\begin{table*}
 \centering
\begin{tabular}{cccccc}
\toprule
CiteSeer  & Baseline        & Distance              & \begin{tabular}[c]{@{}c@{}}MinMax\\ Distance\end{tabular} & \begin{tabular}[c]{@{}c@{}}R-MinMax\\ Distance\end{tabular} & EDEN                     \\ \hline
GCN       & 0.713$\pm$0.012 & 0.717$\pm$0.010       & 0.722$\pm$0.020                                           & 0.743$\pm$0.006                                             & \textbf{0.748$\pm$0.013} \\
GraphSAGE & 0.728$\pm$0.031 & 0.736$\pm$0.012       & 0.736$\pm$0.008                                           & 0.742$\pm$0.001                                             & \textbf{0.750$\pm$0.020} \\
GAT       & 0.721$\pm$0.052 & {\ul 0.719$\pm$0.011} & 0.721$\pm$0.009                                           & 0.733$\pm$0.008                                             & \textbf{0.734$\pm$0.020} \\
GIN       & 0.716$\pm$0.018 & {\ul 0.709$\pm$0.007} & 0.718$\pm$0.007                                           & \textbf{0.729$\pm$0.009}                                    & 0.722$\pm$0.023          \\ \hline
Average   & 0.720           & 0.720                 & 0.724                                                     & \textbf{0.740}                                       & 0.739                    \\ \bottomrule
\end{tabular}
\caption{Ablation study on CiteSeer dataset, where the {\ul underline} denotes ``being even lower than the baseline'' and the \textbf{bold} means the best performance for a given MPNN (or the average of four MPNNs).}
\end{table*}

\begin{table*}
 \centering
\begin{tabular}{cccccc}
\toprule
ENZYMES   & Baseline        & Distance        & \begin{tabular}[c]{@{}c@{}}MinMax\\ Distance\end{tabular} & \begin{tabular}[c]{@{}c@{}}R-MinMax\\ Distance\end{tabular} & EDEN                     \\ \hline
GCN       & 0.657$\pm$0.004 & 0.834$\pm$0.004 & 0.852$\pm$0.007                                           & 0.862$\pm$0.004                                             & \textbf{0.869$\pm$0.001} \\
GraphSAGE & 0.643$\pm$0.003 & 0.871$\pm$0.013 & 0.871$\pm$0.004                                           & 0.880$\pm$0.002                                             & \textbf{0.887$\pm$0.002} \\
GAT       & 0.613$\pm$0.008 & 0.862$\pm$0.009 & 0.868$\pm$0.004                                           & 0.868$\pm$0.004                                             & \textbf{0.876$\pm$0.002} \\
GIN       & 0.663$\pm$0.006 & 0.857$\pm$0.004 & 0.855$\pm$0.002                                           & 0.871$\pm$0.005                                             & \textbf{0.878$\pm$0.001} \\ \hline
Average   & 0.644           & 0.856           & 0.862                                                     & 0.871                                                       & \textbf{0.878}           \\ \bottomrule
\end{tabular}
\caption{Ablation study on ENZYMES dataset, where the {\ul underline} denotes ``being even lower than the baseline'' and the \textbf{bold} means the best performance for a given MPNN (or the average of four MPNNs).}
\end{table*}

\begin{table*}
 \centering
\begin{tabular}{cccccc}
\toprule
\begin{tabular}[c]{@{}c@{}}PROTEINS\\ Edge\end{tabular} & Baseline        & Distance                       & \begin{tabular}[c]{@{}c@{}}MinMax\\ Distance\end{tabular} & \begin{tabular}[c]{@{}c@{}}R-MinMax\\ Distance\end{tabular} & EDEN                     \\ \hline
GCN                                                     & 0.613$\pm$0.023 & 0.817$\pm$0.149                & 0.873$\pm$0.028                                           & 0.876$\pm$0.003                                             & \textbf{0.879$\pm$0.001} \\
GraphSAGE                                               & 0.658$\pm$0.002 & 0.790$\pm$0.172                & 0.882$\pm$0.003                                           & 0.884$\pm$0.003                                             & \textbf{0.917$\pm$0.001} \\
GAT                                                     & 0.610$\pm$0.002 & 0.783$\pm$0.179                & 0.876$\pm$0.003                                           & \textbf{0.886$\pm$0.004}                                    & 0.881$\pm$0.001          \\
GIN                                                     & 0.614$\pm$0.015 & 0.791$\pm$0.004 & 0.853$\pm$0.004                                           & 0.872$\pm$0.003                                             & \textbf{0.874$\pm$0.003} \\ \hline
Average                                                 & 0.624           & 0.795                          & 0.871                                                     & 0.880                                                       & \textbf{0.888}           \\ \bottomrule
\end{tabular}
\caption{Ablation study on PROTEINS dataset (link prediction task), where the {\ul underline} denotes ``being even lower than the baseline'' and the \textbf{bold} means the best performance for a given MPNN (or the average of four MPNNs).}
\end{table*}

\begin{table*}
 \centering
\begin{tabular}{cccccc}
\toprule
MUTAG     & Baseline        & Distance                       & \begin{tabular}[c]{@{}c@{}}MinMax\\ Distance\end{tabular} & \begin{tabular}[c]{@{}c@{}}R-MinMax\\ Distance\end{tabular} & EDEN                     \\ \hline
GCN       & 0.837$\pm$0.027 & 0.855$\pm$0.149                & 0.873$\pm$0.028                                           & 0.911$\pm$0.003                                             & \textbf{0.926$\pm$0.031} \\
GraphSAGE & 0.842$\pm$0.050 & 0.845$\pm$0.172                & 0.882$\pm$0.003                                           & 0.928$\pm$0.078                                             & \textbf{0.932$\pm$0.054} \\
GAT       & 0.811$\pm$0.031 & 0.837$\pm$0.179                & 0.876$\pm$0.003                                           & \textbf{0.911$\pm$0.051}                                    & \textbf{0.911$\pm$0.027} \\
GIN       & 0.847$\pm$0.087 & 0.866$\pm$0.028 & 0.853$\pm$0.004                                           & 0.915$\pm$0.055                                             & \textbf{0.937$\pm$0.086} \\ \hline
Average   & 0.834           & 0.851                          & 0.871                                                     & 0.916                                                       & \textbf{0.927}           \\ \bottomrule
\end{tabular}
\caption{Ablation study on MUTAG dataset, where the {\ul underline} denotes ``being even lower than the baseline'' and the \textbf{bold} means the best performance for given MPNN (or the average of four MPNNs).}
\end{table*}

\begin{table*}
 \centering
\begin{tabular}{cccccc}
\toprule
\begin{tabular}[c]{@{}c@{}}PROTEINS\\ Graph\end{tabular} & Baseline        & Distance                       & \begin{tabular}[c]{@{}c@{}}MinMax\\ Distance\end{tabular} & \begin{tabular}[c]{@{}c@{}}R-MinMax\\ Distance\end{tabular} & EDEN                     \\ \hline
GCN                                                      & 0.704$\pm$0.008 & 0.718$\pm$0.043                & 0.776$\pm$0.018                                           & \textbf{0.786$\pm$0.041}                                    & 0.781$\pm$0.014          \\
GraphSAGE                                                & 0.714$\pm$0.042 & 0.721$\pm$0.009                & 0.739$\pm$0.171                                           & 0.723$\pm$0.050                                             & \textbf{0.755$\pm$0.019} \\
GAT                                                      & 0.726$\pm$0.009 & {\ul 0.723$\pm$0.009}          & 0.757$\pm$0.022                                           & 0.775$\pm$0.011                                             & \textbf{0.788$\pm$0.013} \\
GIN                                                      & 0.718$\pm$0.012 & 0.720$\pm$0.018 & 0.735$\pm$0.012                                           & 0.723$\pm$0.024                                             & \textbf{0.744$\pm$0.076} \\ \hline
Average                                                  & 0.716           & 0.721                          & 0.752                                                     & 0.752                                                       & \textbf{0.767}           \\ \bottomrule
\end{tabular}
\caption{Ablation study on PROTEINS dataset (graph classification task), where the {\ul underline} denotes ``being even lower than the baseline'' and the \textbf{bold} means the best performance for a given MPNN (or the average of four MPNNs).}
\label{table:gproteins}
\end{table*}


\section{The Advantages of Cosine Phase Propagation}
\label{app:phase}
 EDEN is inspired by the sinusoidal positional encoding in Transformer \cite{transformer} and distance-based GNNs \cite{DEGNN, PGNN}. Transformer states that a set of trigonometric functions can directly characterize positions. The distance-based GNNs show that the distance can measure the relationship between two nodes, and the strength of the relationship should be inversely related to the distance. Therefore, the motivation of EDEN is straightforward. We want to use a set of trigonometric functions to represent the position of each node. This representation is based on the distance between each pair of nodes. The longer the distance, the weaker the relationship. The unreachable situation between two nodes in the graph should be considered, 
 so we manually set a fixed value for them.
 Taking all of the above into consideration, it ends up being EDEN.
 Furthermore, the input features are preferably normalized real vectors for GNNs. The unweighted graph distance matrix consists of integers, and its norm can be large, which is not conducive to model training. We propose three sets of distance-based node encodings 
 for the ablation study to test our viewpoint, namely, direct distance, Min-Max normalized distance, and Reversed Min-max normalized distance. 

We implement these three encodings in real-world datasets and compare them with EDEN. Results are shown in Tab. \ref{table:cora} to \ref{table:gproteins}.
The experimental results show that directly using the distance as an augmented feature has the least contribution to the MPNNs and may even have negative effects. Normalizing the distance matrix can improve the performance of the models, but it is often not as good as reversed normalization. On the other hand, reversed Min-Max normalized node encoding performs well on all tasks. In some cases, it can achieve better results than EDEN, but on average, EDEN with phase propagation achieves the best results.

\end{document}